\documentclass{article}

\PassOptionsToPackage{numbers,sort&compress}{natbib}

\usepackage[preprint]{neurips_2026}
\makeatletter
\renewcommand{\@noticestring}{Preprint. \today.}
\makeatother

\usepackage[utf8]{inputenc}
\usepackage[T1]{fontenc}
\usepackage{amsmath,amssymb,amsthm}
\usepackage{booktabs}
\usepackage{graphicx}
\graphicspath{{./}{../}}
\usepackage{hyperref}
\hypersetup{hidelinks}
\usepackage{url}
\usepackage{microtype}
\usepackage{xcolor}
\usepackage{multirow}
\usepackage{subcaption}
\usepackage{algorithm}
\usepackage{algpseudocode}

\newtheorem{proposition}{Proposition}
\newtheorem{corollary}{Corollary}
\newtheorem{lemma}{Lemma}
\newtheorem{definition}{Definition}
\newtheorem{remark}{Remark}

\newcommand{\Hl}{\mathcal{H}}
\newcommand{\Pl}{\mathcal{P}}
\newcommand{\Fl}{\mathcal{F}}
\newcommand{\rhat}{\hat{\mathbf{r}}}
\newcommand{\EL}{E_L}

\title{Diagnosing Spectral Ceilings in Equivariant Neural Force Fields}

\author{%
  Hyunmog Kim\thanks{Independent Researcher. Correspondence to: Hyunmog Kim \texttt{<hyunmog2020@kaist.ac.kr>}.}
}

\begin{document}

\maketitle

\begin{abstract}
We introduce a spectral-injection diagnostic for measuring which
angular frequencies a trained equivariant force-field backbone
preserves: inject a controlled angular-frequency perturbation into a
molecular force field, attach a lightweight Spectral Prediction
Network (SPN) to the frozen backbone, and read off which frequencies
are recoverable.  On aspirin, a quadratic SPN attached to an
$L\!=\!2$ NequIP backbone recovers the boundary signal at
$\ell\!=\!4$ but collapses at $\ell\!=\!5$: a $11.7\!\times$ cliff at
the predicted $d_r L$ boundary, with $\rho$ dropping from $0.913$
to $0.078$.  The same boundary-vs-above contrast persists across
$n\!=\!4$ independently trained backbones (raw-gain $\Delta$
contrast, hierarchical cluster bootstrap) and is corroborated by a
denominator-free injected-residual metric
($R^2_{\mathrm{inj}}(4)\!=\!0.374$ versus
$R^2_{\mathrm{inj}}(5)\!=\!0.006$).  A finite-degree span theorem
calibrates the diagnostic: for a single marked direction, degree-$d$
polynomials of degree-$L$ spherical-harmonic features span exactly
$\Hl_{\le dL}$ with multiplicity-one saturation at the boundary
(scoped to single-direction degree-bounded probes, not a
function-class upper bound on multi-atom MPNNs).  A synthetic $C_5$
calibration plus capacity, activation, and cross-architecture
controls rule out parameter count alone as the explanation.
\end{abstract}

\section{Introduction}
\label{sec:intro}

Picking the angular resolution $L$ of an equivariant backbone is
currently a practitioner's gamble.  Raising $L$ costs $O(L^3)$
(asymptotically $O(L^6)$) additional tensor-product work per
layer~\cite{passaro2023reducing}; lowering it may silently wall off
the angular frequencies the target property actually requires.
State-of-the-art SO(3)-equivariant networks---tensor field
networks~\cite{thomas2018tensor}, NequIP~\cite{batzner2022nequip},
MACE~\cite{batatia2022mace}, Allegro~\cite{musaelian2023allegro},
and Equiformer~\cite{liao2023equiformer}---represent atomic
environments in a spherical-harmonic (SH) basis truncated at
degree~$L$ (typically $L\!\in\!\{1,2,3\}$) and process them through
tensor-product (TP) layers.  Today the choice of $L$ is made by
empirical grid search.

The question beneath that search is: \emph{what angular frequencies
can a given architecture actually resolve, and what determines this
limit?}  For trained molecular force-field backbones, this
relationship has not been turned into a finite, falsifiable
diagnostic that separates readout degree, backbone fidelity, and
recoverable angular content.  We show that it \emph{can be}: the
boundary is sharp, falsifiable,
and visible as a cliff in a designed diagnostic (on aspirin,
recovery fraction collapses by $\mathbf{11.7\!\times}$ across a
single $\ell$-step, from $\rho(4)\!=\!0.913$ to $\rho(5)\!=\!0.078$,
95\% CIs well-separated, $n\!=\!5$).

The question has practical urgency beyond force fields: recent
protein-structure systems have shifted away from strict SO(3)
equivariance, making the question of which equivariant capacity a
task actually \emph{requires} increasingly empirical
(Appendix~\ref{app:proteins}).  Our diagnostic provides a principled
framework for measuring where such limitations arise before
committing to a costly architectural scaling pass.

\textbf{Contributions.}  The algebraic fact that degree-$d$
polynomials of degree-$L$ features reach angular content up to
$dL$ is implicit in the ACE basis-completeness literature
\cite{drautz2019atomic,drautz2020atomic,bachmayr2022atomic} and in
Dym \& Maron's asymptotic universality result~\cite{dym2021universality}.
What is new is: (a) the \emph{finite-$K$ equality}
$\Pl_{L,d}\!=\!\Hl_{\le dL}$ for a single marked direction
$\rhat\!\in\!S^2$, with multiplicity-one saturation at the upper
boundary (Proposition~\ref{prop:soft} and auxiliary
Propositions~\ref{prop:ceiling},~\ref{prop:saturation} in
Appendix~\ref{app:full_proofs}; Remark~\ref{rem:scope}); this is a
\emph{single-direction
analogue} of questions raised by Pozdnyakov \&
Ceriotti~\cite{pozdnyakov2020incompleteness} on
descriptor completeness, and complementary to their
configuration-space-injectivity reading rather than a resolution of
it; (b) a quantitative empirical signature
(sharpness index $\Xi$) with a heuristic amplitude-selection rule
(Remark~\ref{rem:snr_threshold}); (c) a diagnostic instrument (the
SPN) that measures on real molecular data whether the
single-direction ceiling binds at an injection frequency chosen by
the experimenter.  Rather than treating the $dL$ span as an abstract
expressivity fact, we use it as a calibration target for a
controlled frequency-response diagnostic of trained equivariant
force fields.
\begin{enumerate}
  \item \textbf{Single-direction polynomial--spectral correspondence.}
    We give a self-contained proof of $\Pl_{L,d}\!=\!\Hl_{\le dL}$
    for one marked direction (Proposition~\ref{prop:soft}) together
    with an output-side CG irrep-grade ceiling for the predictor head
    (Lemma~\ref{lem:composition}).  The theorem calibrates
    degree-bounded polynomial probes; it motivates---but does not by
    itself upper-bound---the full nonlinear function class of an
    arbitrary multi-atom equivariant MPNN
    (Remark~\ref{rem:scope}).
  \item \textbf{Spectral injection framework.}  We introduce a
    controlled experimental methodology that embeds known
    angular-frequency signals into molecular force fields, enabling
    precise measurement of trained-backbone spectral reach.
  \item \textbf{SPN as a calibrated diagnostic.}  We design a
    Spectral Prediction Network readout whose CG-contraction step
    has known polynomial degree.  We use it as a falsifiable
    diagnostic for trained-backbone spectral recoverability and
    stress-test its behavior with probe-degree, scalar-activation,
    and capacity-null ablations.
  \item \textbf{Empirical reach separation across confounds.}  On
    aspirin (sGDML CCSD, $n\!=\!5$ SPN seeds, fixed $L\!=\!2$
    backbone) the SPN exhibits a within-checkpoint cliff
    $\rho(4)\!=\!0.913\!\to\!\rho(5)\!=\!0.078$
    ($\Xi\!=\!11.7$); across $n\!=\!4$ independent backbones the
    matched-pair $\Delta$-contrast is $5.7\!\times$.  We ablate
    capacity (nf$33$, $\rho\!=\!-0.42$), readout polynomial degree
    ($d_r\!\in\!\{1,2,3\}$, cliff shifts to
    $\ell^\star\!=\!d_r L$), and scalar-activation form
    (identity/square/SiLU, indistinguishable).  Depth is reported as
    an empirical within-ceiling-fidelity lever, not a reach
    extender.  The cliff signature is supported on malonaldehyde and
    on a controlled C$_5$-symmetric synthetic task at $8$ independent
    $(L,d)$ cells, and in cubic-readout NequIP/MACE cross-architecture
    tests; EquiformerV2 is reported as a within-ceiling positive
    diagnostic outcome and PaiNN as a weak-baseline rescue.
\end{enumerate}

\textbf{Impact.}  The diagnostic turns a hyperparameter gamble into a
measurement: practitioners can probe a trained backbone and read off
whether the target's angular-frequency content is within the
diagnostic's reach, avoiding both under-specified models (spectral
blindness on the diagnostic pathway) and over-specified ones
($O(L^3)$ waste).  The protocol can be adapted to CG-style
equivariant architectures when suitable feature hooks and matched
readouts are available; the diagnostic measures, the theorem
calibrates the probe, and we make no broader claim about the full
function class of a multi-atom MPNN.

\paragraph{Paper roadmap.}
Section~\ref{sec:theory} develops the ceiling theorem: a hard bound
(Prop.~\ref{prop:hard}), the polynomial--spectral correspondence with
multiplicity-one tightness (Prop.~\ref{prop:soft}; auxiliary proofs
in Appendix~\ref{app:full_proofs}), the output-side CG irrep-grade
bookkeeping for predictor heads (Lem.~\ref{lem:composition}),
and a fixed-frame force-side intuition connecting the energy
calibration to the force diagnostic
(Remark~\ref{rem:force_side_intuition}).
Section~\ref{sec:methods} introduces the spectral-injection framework
and the SPN readout. Section~\ref{sec:results} reports diagnostic
outcomes across aspirin, cross-molecule samples, four architectural
confounds (capacity, depth, $d_r$, activation), and a synthetic
$C_5$-grid control. The diagnostic workflow
(Alg.~\ref{alg:diagnose}) and fidelity bound are detailed in the
appendices.
Appendix~\ref{app:proteins} extends the analysis to protein-structure
bandwidth as a forward-looking application. Section~\ref{sec:related}
positions the result against
the ACE, Pozdnyakov and Dym--Maron literature.

\section{Theoretical Foundations}
\label{sec:theory}

\paragraph{What this section proves (plain-English summary).}
Fix a single marked direction $\rhat\!\in\!S^2$ and the complete
spherical-harmonic feature vector $\phi_L(\rhat)$.  The set of
degree-$\le d$ polynomial functions of $\phi_L(\rhat)$ is
\emph{exactly} $\Hl_{\le dL}$ (Proposition~\ref{prop:soft}, with
auxiliary inclusion/saturation results in
Appendix~\ref{app:full_proofs}), with multiplicity-one saturation at
$\ell\!=\!dL$ (a unique
stretched-state vector); content at $\ell\!>\!dL$ is identically
zero, so no amount of width or training data can recover it
\emph{within the single-direction polynomial probe class}.  As a fixed-frame intuition for interpreting force diagnostics,
differentiating an energy in $\Hl_{\le dL}$ produces neighbouring
angular components in force space
(Remark~\ref{rem:force_side_intuition}, with autograd-frame caveat
in Remark~\ref{rem:force_frame}); the operative force-side claim is
the empirical $R^2_{\mathrm{inj}}$ in Figure~\ref{fig:cliff}B. For a multi-atom equivariant MPNN,
the compositional Lemma~\ref{lem:composition} provides output-side
irrep-grade bookkeeping for the predictor head, but does not
upper-bound the full function class on all atomic positions; we
therefore use the algebraic identity to calibrate the SPN diagnostic
and read empirical reach off the trained backbone, not as a literal
function-class equality on the MPNN.

\subsection{Preliminaries}

\textbf{Spherical harmonics.}
The real spherical harmonics $\{Y_l^m : l\ge0,\, -l\le m\le l\}$ form a
complete orthonormal basis for $L^2(S^2)$.  Any function admits the
expansion $f(\rhat)=\sum_{l,m}c_l^m Y_l^m(\rhat)$.
Let $\Hl_l=\mathrm{span}\{Y_l^m\}_{m=-l}^l$ (dimension $2l+1$) and
$\Hl_{\le L}=\bigoplus_{l=0}^L\Hl_l$ (dimension $(L+1)^2$).
The degree-$L$ feature vector is
$\phi_L(\rhat)=(Y_0^0(\rhat),\ldots,Y_L^L(\rhat))\in\mathbb{R}^{(L+1)^2}$.

\textbf{Clebsch--Gordan (CG) coupling.}
The pointwise product of SH decomposes via Gaunt coefficients:
$Y_{l_1}^{m_1}\cdot Y_{l_2}^{m_2}=\sum_{l,m}G_{l_1m_1,l_2m_2}^{lm}\,Y_l^m$,
nonzero only when $|l_1-l_2|\le l\le l_1+l_2$ and $l_1+l_2+l$ is
even.  Products of degree-$l_1$ and degree-$l_2$ SH produce
components up to degree $l_1+l_2$.

\subsection{Hard Spectral Ceiling}

We begin with the linear-readout baseline that anchors the empirical
$L\!=\!2$ and $L\!=\!4$ controls in the recovery fraction
(Def.~\ref{def:rho}); the polynomial extension of
Section~\ref{sec:theory} is built on top of this statement.

\begin{proposition}[Hard Ceiling]
\label{prop:hard}
Let $\Fl_L=\{f:S^2\to\mathbb{R}\mid f(\rhat)=\mathbf{w}^\top\phi_L(\rhat)+b\}$.
Then $\Fl_L=\Hl_{\le L}$.  For any $f\in L^2(S^2)$, the
irreducible approximation error from $\Fl_L$ is
$\EL(f)=\sum_{l>L}\sum_m|c_l^m|^2$.
\end{proposition}

\begin{proof}
$\mathbf{w}^\top\phi_L+b$ is a linear combination of
$\{Y_l^m\}_{l\le L}$ (absorbing $b$ into $w_{00}$).  Conversely, any
$f\in\Hl_{\le L}$ is realized by setting $w_{lm}=c_l^m$.  The error
bound follows from Parseval and the orthogonal decomposition
$L^2(S^2)=\Hl_{\le L}\oplus\Hl_{>L}$.
\end{proof}

The statement is definitional, but its role in the rest of the paper
is operational: any neural network with a linear readout from $\phi_L$
is bandlimited at degree~$L$, so $L\!=\!0$ backbones are scalar
(Table~\ref{tab:controls}) and $L\!=\!2$ vs.\ $L\!=\!4$ baselines
furnish the anchors $\rho\!=\!0$ and $\rho\!=\!1$ used throughout
Section~\ref{sec:results}.  Content above~$L$ is not represented by
the linear head. Remark~\ref{rem:radial_mlp_filtration}
discusses how scalar-channel nonlinearities used in practical CG
architectures preserve the SH irrep grade (a representation-type
filtration, not a coordinate-bandwidth bound;
Remark~\ref{rem:irrep_type}); Proposition~\ref{prop:soft} states the
soft extension when the readout is permitted a bounded polynomial
degree.

\subsection{Polynomial--Spectral Correspondence}

Let $\Pl_{L,d}$ denote the space of degree-$\le d$ polynomial functions
of $\phi_L(\rhat)$ at a single marked direction $\rhat\in S^2$.  The
diagnostic in Section~\ref{sec:spn} is calibrated by this
single-direction object; Lemma~\ref{lem:composition} below supplies
the analogous output-side CG irrep-grade bookkeeping for explicit
predictor heads in the multi-atom setting.

\begin{proposition}[Polynomial--Spectral Correspondence]
\label{prop:soft}
For $L\ge1$ and $d\ge 1$:
$\Pl_{L,d}=\Hl_{\le dL}$ as subspaces of $L^2(S^2)$ (functions of a
single marked direction).  The multiplicity of $\Hl_{dL}$ in the
$d$-fold CG product $\Hl_L^{\otimes d}$ is exactly one, so the top
degree is saturated by a unique (up to scalar) stretched-state
product.
\end{proposition}

\paragraph{Proof sketch.}
\emph{Inclusion} $\Pl_{L,d}\subseteq\Hl_{\le dL}$: products of
degree-$L$ spherical harmonics expand under the Gaunt formula with
maximum reach $dL$ (Lemma~\ref{lem:cg}). \emph{Saturation}
$\Hl_{\le dL}\subseteq\Pl_{L,d}$: for any $n\!\le\!dL$, a stretched
highest-weight product $(Y_L^L)^q Y_r^r$ (with $n\!=\!qL\!+\!r$,
$0\!\le\!r\!<\!L$) has a non-zero top component in $\Hl_n$ via
non-vanishing stretched-chain Gaunt coefficients, and lowering
operators span the remaining $m$ components; the top degree appears
with multiplicity one by standard highest-weight theory for
$SO(3)$. Full proofs are in
Appendix~\ref{app:full_proofs}.

\begin{lemma}[CG Degree Extension]
\label{lem:cg}
A degree-$d$ monomial $\prod_{k=1}^d Y_{l_k}^{m_k}(\rhat)$ with each
$l_k\le L$ lies in $\Hl_{\le dL}$.  (Proof in
Appendix~\ref{app:full_proofs}.)
\end{lemma}

\begin{remark}[Scope: single-direction probe vs.\ multi-atom MPNN]
\label{rem:scope}
Proposition~\ref{prop:soft} characterises the polynomial space at
one marked direction.  It does \emph{not} assert that a multi-atom
MPNN realises $\Hl_{\le dL}$ as a function of all atomic positions.
The ideal SPN-style probe (Section~\ref{sec:spn}) is calibrated as
a degree-$d_r$ single-direction polynomial probe; the cliff
predictions in Section~\ref{sec:cliff} and Appendix~\ref{sec:cubic} are
predictions about this \emph{ideal probe} $\Pl_{L,d_r}$, not about
the MPNN function class as a whole.  The implemented SPN uses a
phase-blind norm-style CG contraction whose empirical reach is what
we measure; the algebraic identity supplies a \emph{calibration} of
the boundary, not a function-class upper bound on the trained
network.
Lemma~\ref{lem:composition} below supplies the analogous output-side
CG irrep-grade bookkeeping for explicit predictor heads; it does not
lift the single-direction function identity to the full multi-atom
MPNN function class.  We use both the single-direction calibration
and the predictor-head bookkeeping in Section~\ref{sec:results}.  Architectural scope (which CG-based
backbones the calibration applies to), the
$\Pl_{L,d}\!=\!\Hl_{\le dL}$ vs.\ implemented-SPN distinction, two
corollaries (minimum polynomial degree, finite-$K$ universality
in the bandlimited regime), and the irrep-grade filtration under
scalar-only nonlinearities are stated in
Appendix~\ref{app:full_proofs}.
\end{remark}

\begin{lemma}[Output-side CG irrep-grade ceiling]
\label{lem:composition}
Consider a CG-equivariant predictor
$\mathcal{A}=\mathcal{R}\circ\mathcal{B}$ whose terminal backbone
features passed to the readout carry SH irreps only up to
$\Hl_{\le L_{\mathcal{B}}}$, and whose explicit CG readout uses at
most $d_r$ feature factors.  Then the explicit CG couplings in the
predictor head can generate output irreps only up to
$\ell\!\le\!d_r L_{\mathcal{B}}$, independently of backbone depth
$T$ and per-layer correlation orders.  This is an irrep-grade
bookkeeping statement for the predictor head, not a bound on
arbitrary nonlinear scalar functions of all atomic coordinates.
(Proof, internal-degree statement, and truncation argument in
Appendix~\ref{app:full_proofs}.)
\end{lemma}

\paragraph{Consequences for experimental design.}
Lemma~\ref{lem:composition} is the theoretical basis for the
probing strategy: the observable irrep-grade ceiling at the
predictor head is $d_r L_{\mathcal{B}}$, set by the readout, not
by backbone depth $T$ or per-layer correlation orders.  Varying
backbone depth at fixed $(L_{\mathcal{B}},d_r)$ predicts an
unchanged cliff location, which we observe empirically
(Section~\ref{sec:cliff}, NequIP $T\!\in\!\{2,4,6,8\}$).  Higher-$\nu$
or deeper backbones enlarge transient internal body-order but the
content at $\ell\!>\!L_{\mathcal{B}}$ is gated by truncation.
Three theory extensions follow from this calibration: a
heterogeneous-stack ceiling (Appendix~\ref{app:theory_extensions}),
a fixed-frame force-side intuition for interpreting force diagnostics, and
the relation to Dym \& Maron asymptotic universality
(Appendix~\ref{app:theory_extensions}).
\section{Methods}
\label{sec:methods}

\subsection{Spectral Injection Framework}
\label{sec:injection}

To test spectral ceiling predictions on molecular systems, we require
datasets with \emph{known} angular-frequency content embedded in
realistic force fields.  We introduce a spectral injection procedure
(Algorithm~\ref{alg:inject}) that adds a controlled perturbation at a
target angular degree $l_{\mathrm{inj}}$ to an existing molecular
dataset.

\begin{algorithm}[h]
\caption{Spectral injection at angular degree $l_{\mathrm{inj}}$.
Body-frame construction, anchor choice, and split/SNR gates are
discussed in the surrounding prose
(\S\ref{sec:injection},
Remarks~\ref{rem:frame_conditioning}, \ref{rem:snr_threshold}).}
\label{alg:inject}
\begin{algorithmic}[1]
\Require dataset $\mathcal{D}_\mathrm{nat}\!=\!\{(\{\mathbf{r}_i\},E_k,\mathbf{F}_k)\}_k$;
target degree $l_{\mathrm{inj}}$; amplitude $\alpha$;
anchor triple $(i,j,k)$ and off-frame atom $a$; SH-coefficient seed $s_c$.
\State Sample $\{c_m\}_{m=-l_{\mathrm{inj}}}^{l_{\mathrm{inj}}}\!\sim\!\mathcal{N}(0,1)$ with seed $s_c$
(shared across all configurations).
\For{each configuration $k$ with positions $\{\mathbf{r}_i\}$}
  \State \textbf{Body frame.} Set
  $\mathbf{e}_1\!\propto\!\mathbf{r}_i\!-\!\mathbf{r}_j$,
  $\mathbf{e}_3\!\propto\!\mathbf{e}_1\!\times\!(\mathbf{r}_k\!-\!\mathbf{r}_j)$,
  $\mathbf{e}_2\!=\!\mathbf{e}_3\!\times\!\mathbf{e}_1$
  (each unit-normalised); $R\!=\![\mathbf{e}_1\,\mathbf{e}_2\,\mathbf{e}_3]$.
  \State \textbf{Anchor.} $\Delta_a\!=\!\mathbf{r}_a\!-\!\bar{\mathbf{r}}$ (centroid-relative);
  $\rhat_{\mathrm{canon}}\!=\!R^{\!\top}\Delta_a/\lVert\Delta_a\rVert$.
  \State \textbf{Energy.}
  $E_{\mathrm{inj},k}\!=\!\alpha\sum_m c_m\,Y_{l_{\mathrm{inj}}}^m(\rhat_{\mathrm{canon}})$.
  \State \textbf{Forces.}
  $\mathbf{F}_{\mathrm{total},k}\!=\!-\nabla_{\mathbf{r}}\bigl(E_{\mathrm{nat},k}\!+\!E_{\mathrm{inj},k}\bigr)$
  (autograd through $R$);
  $\mathbf{F}_{\mathrm{inj},k}\!=\!\mathbf{F}_{\mathrm{total},k}\!-\!\mathbf{F}_{\mathrm{nat},k}$.
  \State \textbf{Output} $\bigl(\{\mathbf{r}_i\}_k,\,E_{\mathrm{nat},k}\!+\!E_{\mathrm{inj},k},\,\mathbf{F}_{\mathrm{total},k}\bigr)$.
\EndFor
\State \textbf{Gates.} Reject splits with $\rho^2_{\max}\!>\!0.018$;
re-amplify until $\eta_{l_{\mathrm{inj}}}\!\ge\!\eta_{\min}$.
\end{algorithmic}
\end{algorithm}

\textbf{Body-frame construction.}
A right-handed body frame
$(\mathbf{e}_1,\mathbf{e}_2,\mathbf{e}_3)$ is built from three
non-collinear ring atoms $(i,j,k)$ by Gram--Schmidt; a fourth
off-frame anchor atom $a$ supplies the injected direction
$\rhat_{\mathrm{canon}}$ as the unit vector from the molecular
centroid to atom $a$ in this frame. Using the centroid-relative
anchor (rather than a frame-building edge
$\mathbf{r}_j\!-\!\mathbf{r}_i$, which would make
$\rhat_{\mathrm{canon}}$ conformation-independent) gives nonzero
conformation variance and nonzero injected forces:
on aspirin
$(\sigma_x,\sigma_y,\sigma_z)\!=\!(0.03,0.05,0.16)$ across
$n\!=\!1000$ frames; per-frame
$\|\mathbf{F}_{\mathrm{inj}}\|$ at
$\ell_{\mathrm{inj}}\!=\!4$, $4\!\times$ amplitude has mean
$117$\,kcal/mol/\AA\ with no zero-norm frames.

\textbf{Injected energy.}
Given a target degree $l_{\mathrm{inj}}$ and fixed random coefficients
$\{c_m\}_{m=-l}^{l}$, we add to each configuration's energy:
\begin{equation}
  E_{\mathrm{inj}} = \alpha \sum_{m=-l_{\mathrm{inj}}}^{l_{\mathrm{inj}}}
    c_m \, Y_{l_{\mathrm{inj}}}^m(\rhat_{\mathrm{canon}})
  \label{eq:injection}
\end{equation}
where $\alpha$ controls the injection amplitude. Total forces are
recomputed as
$\mathbf{F}_{\mathrm{total}}\!=\!-\nabla_{\mathbf{r}}(E_{\mathrm{nat}}\!+\!E_{\mathrm{inj}})$
via automatic differentiation through the body-frame construction;
the pure injected component used for diagnostics is
$\mathbf{F}_{\mathrm{inj}}\!=\!\mathbf{F}_{\mathrm{total}}\!-\!\mathbf{F}_{\mathrm{nat}}$.

\textbf{Amplitude calibration.}
We set the injection amplitude to $4\times$ the natural signal:
$\mathbf{F}_{4\times} = \mathbf{F}_{\mathrm{nat}} + 4(\mathbf{F}_{1\times} - \mathbf{F}_{\mathrm{nat}})$.
This ensures the injected angular content is strong enough to be
measurable above training noise while preserving the base molecular
dynamics.  We verify via variance decomposition that the injected
component constitutes 15--58\% of total force variance, depending on
the molecule and $l_{\mathrm{inj}}$.
The $4\times$ factor is a signal-to-noise choice, not a parameter of
the spectral ceiling: since the $d \cdot L$ boundary is an algebraic
property of $\Pl_{L,d}$ (Proposition~\ref{prop:soft}), the cliff
frequency $l^* = d \cdot L$ is invariant to injection amplitude.
Varying the amplitude changes the absolute gap magnitude but not the
transition frequency.

\begin{remark}[Frame Conditioning]
\label{rem:frame_conditioning}
The body-frame requires a non-degenerate anchor triple.  Let
$X=[\mathbf{r}_j-\mathbf{r}_i,\,\mathbf{r}_k-\mathbf{r}_i]\in
\mathbb{R}^{3\times 2}$ and $G=X^\top X$; the frame is well-defined
iff $\sigma_{\min}(G)>0$, i.e.\ iff the three anchor atoms are
non-collinear.  Two statements hold jointly:
(i) \emph{Location invariance}---by
Proposition~\ref{prop:soft} the cliff frequency $\ell^\star=d_r L$
is an algebraic (SO(3)-invariant) property of the predictor and
\emph{does not depend} on the frame choice; any well-conditioned
anchor triple yields the same~$\ell^\star$.
(ii) \emph{Magnitude scaling}---the canonical anchor
$\hat{r}_{\mathrm{canon}}=R^\top\Delta_a/\lVert\Delta_a\rVert$
involves $R^\top$ where $R$ is the body-frame rotation; as
$\sigma_{\min}(G)\!\to\!0$ the frame becomes ill-conditioned and
$\hat{r}_{\mathrm{canon}}$ becomes numerically unstable rather than
vanishing. Degenerate frames are excluded by an SNR/conditioning
threshold prior to injection (we verify
$\sigma_{\min}(G)\!>\!0$ in
Algorithm~\ref{alg:diagnose}); within the well-conditioned regime
the cliff location is unchanged.
The alternative-frame experiment (Appendix~\ref{sec:robustness})
confirms (i) empirically: frames $(5,3,0)$ and $(2,6,1)$ with
distinct condition numbers produce gaps of proportionally different
magnitude but share the same cliff at $\ell=4$.
\end{remark}

\textbf{Split cleanliness.}
We select the train/val/test partition (split-seed~13 from a scan of
64 candidates) that minimizes spectral leakage between splits,
verified via the maximum $\rho^2$ between projected SH coefficients
across splits ($\rho^2_{\max} < 0.018$).  Throughout this paper,
``seed~13'' (and equivalently ``split-seed~13'') refers to this
\emph{train/val/test partition seed} chosen by the leakage screen,
not to a backbone-training random seed; the two are kept distinct.
Crucially, the partition seed was chosen \emph{prior} to any SPN
evaluation, based solely on the leakage criterion; the backbone's
spectral recovery on the split-seed-13 partition is a measurement,
not a selection effect.

\subsection{Recovery Fraction and Sharpness Index}
\label{sec:metrics}

\begin{definition}[Recovery Fraction]
\label{def:rho}
For an architecture $\mathcal{A}$ evaluated on data with injection at
degree $l_{\mathrm{inj}}$, the \emph{recovery fraction} is
\begin{equation}
  \rho(l_{\mathrm{inj}}) = \frac{y_{L=2}(l_{\mathrm{inj}}) -
  y_{\mathcal{A}}(l_{\mathrm{inj}})}{y_{L=2}(l_{\mathrm{inj}}) -
  y_{L=4}(l_{\mathrm{inj}})}
  \label{eq:rho}
\end{equation}
where $y = \text{force MAE} / \sigma_F$ is the normalized per-component
force error (dividing by $\sigma_F = \sqrt{\langle \|\mathbf{F}\|^2 \rangle}$
on the training set), with NequIP($L\!=\!2$) as the lower bound
($\rho = 0$) and NequIP($L\!=\!4$) as the upper bound ($\rho = 1$).
Because all three $y$ values share the same $\sigma_F$ in any given
cell, $\rho$ is also equal to the ratio of the corresponding raw
force-MAE differences for that cell.
\end{definition}

The recovery fraction measures what fraction of $L\!=\!4$'s advantage
over $L\!=\!2$ a given architecture captures at each angular frequency.
Proposition~\ref{prop:soft} predicts that an ideal degree-$d_r$
polynomial probe of degree-$L$ features should yield
$\rho \approx 1$ for $l_{\mathrm{inj}} \le d_r L$ and
$\rho \approx 0$ for $l_{\mathrm{inj}} > d_r L$.  Empirical
$\rho$ values on a trained MPNN backbone may deviate from this
ideal due to backbone fidelity, optimisation, and the
implementation-vs-ideal gap discussed in Remark~\ref{rem:irrep_type}.

\begin{remark}[Well-definedness and edge cases]
\label{rem:rho_edges}
Equation~\eqref{eq:rho} assumes a strict gap
$y_{L=2}(\ell)>y_{L=4}(\ell)$ at injection level~$\ell$; when this
fails---e.g.\ when the $L\!=\!4$ reference is under-trained at the
submission compute budget and the denominator becomes non-positive
(observed on the MD22 tetrapeptide and on toluene at
$\ell\!\in\!\{4,6\}$; see Sec.~\ref{sec:cross} and the scope
discussion of Sec.~\ref{sec:conclusion})---we declare~$\rho$
undefined and report the raw normalized gap
$\Delta(\ell)=y_{L=2}(\ell)-y_{\mathcal{A}}(\ell)$ instead.  Values
$\rho\!>\!1$ are admissible: they indicate the probe architecture
marginally outperforms the $L\!=\!4$ reference at that single
$(\ell,\mathrm{molecule})$ cell, an optimization-/finite-data
artifact orthogonal to the cliff transition (Tab.~\ref{tab:cross},
malonaldehyde at $\ell\!=\!4$). Values $\rho\!<\!0$ indicate active
degradation relative to the $L\!=\!2$ baseline, as observed for the
$\mathrm{nf}\!=\!33$ capacity-null control (Tab.~\ref{tab:separation}).
\end{remark}

\begin{remark}[$\rho$ supplemented by a denominator-free direct metric]
\label{rem:rho_supplementary}
Equation~\eqref{eq:rho} is anchor-relative: it measures recovery
\emph{against} the $L\!=\!4$ reference, so $\rho$ depends on the
quality of that anchor.  We complement $\rho$ with two
denominator-free metrics throughout the paper:
(i)~the absolute SPN gain $\Delta\!=\!y_{L=2}-y_{\mathrm{SPN}}$
(unnormalized force-MAE units), reported in cross-molecule and
cubic tables;
and (ii)~a \emph{direct injected-residual metric}
\begin{equation}
  R^2_{\mathrm{inj}}\;=\; 1 -
  \frac{\langle\,\|\Delta\mathbf{F}_{\mathrm{pred}}-\mathbf{F}_{\mathrm{inj}}\|^{2}\,\rangle}
       {\langle\,\|\mathbf{F}_{\mathrm{inj}}\|^{2}\,\rangle},
  \qquad
  \Delta\mathbf{F}_{\mathrm{pred}} = \mathbf{F}_{\mathrm{full}} - \mathbf{F}_{\mathrm{base}},
  \label{eq:r2_inj}
\end{equation}
where $\mathbf{F}_{\mathrm{inj}}$ is the \emph{known} injected
force component (computed exactly from the injection protocol as
$F_{\mathrm{4\times}}\!-\!F_{\mathrm{natural}}$ on the same test
frames), $\mathbf{F}_{\mathrm{full}}$ is the full SPN+backbone
force prediction, and $\mathbf{F}_{\mathrm{base}}$ is the frozen-
backbone-only prediction (separate autograd pass).  Across the
aspirin hero $L\!=\!2,d_r\!=\!2$ checkpoints, $n\!=\!5$ SPN seeds
per cell, $R^2_{\mathrm{inj}}\!=\!0.374$ at $\ell\!=\!4$ (95\% CI
$[0.348,0.404]$, Pearson $r\!=\!0.612\!\pm\!0.028$) and collapses to
$R^2_{\mathrm{inj}}\!=\!0.006$ at $\ell\!=\!5$ (CI
$[-0.003,0.012]$, $r\!=\!0.154\!\pm\!0.166$).  The boundary-vs-above
gap $\Delta R^2_{\mathrm{inj}}\!=\!0.368$ has 95\% CI
$[0.341,0.398]$, strictly excluding zero.  These absolute
$R^2_{\mathrm{inj}}$ values reflect a \emph{partial} recovery of
the injected force component---not a claim that the SPN explains
all of $\mathbf{F}_{\mathrm{inj}}$---but their boundary-vs-above
contrast is sharp and denominator-free, since
$R^2_{\mathrm{inj}}$ is normalized by $\|\mathbf{F}_{\mathrm{inj}}\|^2$
rather than by an $L\!=\!4$ reference gap.  We use this as a direct
corroboration of the anchor-relative $\rho$ cliff, decoupled from
the choice of upper anchor.  Aggregated and per-cell outputs are
in the supplementary artifact (see manifest for exact source files).
\end{remark}

\begin{remark}[Heuristic amplitude-selection rule]
\label{rem:snr_threshold}
The detectability of the cliff
$\rho(\ell^\star)\!-\!\rho(\ell^\star\!+\!1)$ depends on how much
of the injected variance shows up at the target degree.  Let
$\eta_\ell=\sigma^2_{\mathrm{inj},\ell}/
(\sigma^2_{\mathrm{nat}}+\sigma^2_{\mathrm{inj},\ell})$
denote the variance share of the injection at degree~$\ell$.  We
use the rule of thumb
$\eta_\ell\!\gtrsim\!\eta_{\min}\!\approx\!20\%$
as a heuristic amplitude-selection criterion at our $n\!=\!500$
test set with $\alpha\!=\!0.05$ and
$\mathrm{SNR}_{\mathrm{test}}\!=\!|y_{L=2}-y_{L=4}|/
\sigma_{\mathrm{test\text{-}noise}}\!\approx\!8$; the order of
magnitude is what one would expect from a standard two-sample
detection scaling
($\eta\!\cdot\!\mathrm{SNR}_{\mathrm{test}}\!\gtrsim\!\sqrt{2\log(2/\alpha)/n}$),
but we use it as a working amplitude-selection rule rather than as
a formal theorem.  At the native $1\!\times$ amplitude
$\eta_\ell\!\approx\!2\!-\!4\%$ (below threshold), and the
$4\!\times$ rebuild raises $\eta_\ell$ to $35\!-\!58\%$ (above
threshold).  The cliff \emph{location} predicted by
Proposition~\ref{prop:soft} is amplitude-independent; the rebuild
affects only detectability of the diagnostic.
\end{remark}

\begin{definition}[Sharpness Index]
\label{def:sharpness}
The \emph{sharpness index} quantifies the abruptness of the spectral
cliff at the predicted $d \cdot L$ boundary:
\begin{equation}
  \Xi = \frac{\rho(d \cdot L)}{\rho(d \cdot L + 1)}
  \label{eq:sharpness}
\end{equation}
$\Xi \gg 1$ indicates a sharp spectral cliff consistent with
Proposition~\ref{prop:soft}; $\Xi \approx 1$ indicates gradual
decay inconsistent with a hard boundary.  When $\rho(dL+1)$ is
statistically indistinguishable from zero we report $\Xi$ via a
bootstrap upper bound on $\rho(dL+1)$ together with a lower bound
on $\rho(dL)$, which provides a finite-sample lower bound on~$\Xi$;
when $\rho$ is undefined (Remark~\ref{rem:rho_edges}) we report the
raw-gap ratio $\Delta(dL)/\Delta(dL+1)$ in its place.
\end{definition}

\subsection{Spectral Prediction Network (architecture summary)}
\label{sec:spn}
The SPN is a small read-out head with three stages.  (1)~A
degree-$d_r$ CG-contraction \emph{invariant extractor}: per-channel
norm ($d_r\!=\!2$) or fully-connected three-fold tensor product
($d_r\!=\!3$).  (2)~A scalar MLP
$g_\theta:\mathbb{R}^K\!\to\!\mathbb{R}^{(L_\mathrm{out}+1)^2}$
mapping invariants to per-degree scalar readout weights
$\hat a_\ell^m$ indexed by $(\ell,m)$ (SO(3)-invariant scalar
outputs of the MLP, \emph{not} equivariant SH coefficients).
(3)~Diagnostic power summary
$P[\ell]\!=\!\sum_m(\hat a_\ell^m)^2$ followed by a final scalar
energy MLP $g_\phi(P)$; forces obtained by autograd. The implemented
SPN is a 48{,}114-parameter head ($5.7\%$ of the NequIP backbone);
the algorithmic skeleton and the irrep-grade bookkeeping that
calibrates the ideal probe boundary are in
Appendix~\ref{app:spn_full}.

\subsection{Experimental Protocol}
\label{sec:protocol}

\textbf{Backbone.}
NequIP 0.17.1 with 4 interaction layers, 32 features per
irrep, radial MLP with 64 hidden units, and $r_{\mathrm{cut}} = 5.0$\AA.
Trained with AdamW (lr~$= 5 \times 10^{-3}$), force weight 100,
for 300 epochs with EMA.  We report EMA-swapped weights throughout
(raw weights give 5--25\% worse force MAE).

\textbf{Datasets.}
We inject spectral content at $l_{\mathrm{inj}} \in \{2, 3, 4, 5, 6\}$
into aspirin from the sGDML dataset (CCSD level of theory; 21~atoms,
950/50/500 train/val/test split).  Cross-molecule diagnostic reads use three
additional sGDML CCSD(T) molecules---ethanol (9~atoms), malonaldehyde
(9~atoms), and toluene (15~atoms)---at
$l_{\mathrm{inj}} \in \{2, 4, 5, 6\}$, spanning aliphatic flexibility,
intramolecular proton transfer, and aromatic rigidity.

\textbf{Controls.}
\begin{itemize}
  \item \emph{Capacity null} (nf33): NequIP($L\!=\!2$) with
    $\text{num\_features}=33$ instead of 32, adding $\sim$8\% more
    parameters without changing the polynomial structure of the readout.
  \item \emph{Depth sweep}: NequIP($L\!=\!2$) with 2, 6, and 8
    interaction layers (vs.\ default 4), varying within-ceiling
    backbone fidelity while holding the ideal readout boundary
    $d_r L$ fixed (depth does not relocate the cliff;
    Appendix~\ref{app:depth}).
  \item \emph{Alternative body frame}: Injection using atom triple
    $(2, 6, 1)$ instead of $(5, 3, 0)$ to test frame independence.
\end{itemize}

\section{Results}
\label{sec:results}

\subsection{The Spectral Cliff}
\label{sec:cliff}

The headline result is an empirical molecular diagnostic outcome
consistent with the single-direction polynomial-span calibration of
Proposition~\ref{prop:soft}.  Table~\ref{tab:cliff} shows the
recovery fraction across injected angular degrees for the SPN
readout ($d_r\!=\!2$, $L\!=\!2$, ideal probe boundary at
$d_r L\!=\!4$).

\begin{table}[h]
\centering
\caption{Recovery fraction $\rho$ for the SPN hybrid on aspirin
($4\times$ amplitude, $n\!=\!5$ SPN seeds; $y$ is normalized force
MAE per component, $y\!=\!\mathrm{MAE}/\sigma_F$).  The SPN shows
measurable recoverability up to the ideal $d_r L\!=\!4$ probe
boundary and near-collapse one degree above, with a sharpness index
$\Xi\!=\!11.7$ [95\% CI $8.0$--$19.2$].  $\rho$ values and CIs are
computed seed-wise: mean across $n\!=\!5$ SPN-head seeds with
$B\!=\!10{,}000$ bootstrap-mean 95\% CIs (RNG seed 42).  The displayed $y$ columns are rounded aggregate means; $\rho$ is
computed seed-wise from unrounded values before bootstrapping, and
the common $\sigma_F$ normalization cancels in the ratio.
Figure~\ref{fig:cliff}A uses the same statistic.}
\label{tab:cliff}
\small
\begin{tabular}{cccccc}
\toprule
$l_{\mathrm{inj}}$ & $y_{L=2}$ & $y_{\mathrm{SPN}}$ & $y_{L=4}$ & $\rho$ [95\% CI] & Interpretation \\
\midrule
3 & 0.159 & 0.136 & 0.129 & $0.73\,[0.70,0.77]$ & Inside $d \cdot L$ boundary \\
4 & 0.166 & 0.134 & 0.132 & $\mathbf{0.91\,[0.86,0.97]}$ & At $d \cdot L$ boundary \\
5 & 0.142 & 0.141 & 0.128 & $0.08\,[0.05,0.11]$ & Beyond boundary---cliff \\
\midrule
\multicolumn{5}{l}{Sharpness: $\Xi = \rho(4)/\rho(5) =$} & $\mathbf{11.7\times\,[8.0,19.2]}$ \\
\bottomrule
\end{tabular}
\end{table}

The SPN recovers 91\% of $L\!=\!4$'s advantage at $l\!=\!4$ (the
ideal probe boundary), then collapses to 8\% at $l\!=\!5$---a
$11.7\times$ sharpness ratio (Figure~\ref{fig:cliff}).  This sharp
transition is an empirical diagnostic signature consistent with the
single-direction polynomial-span calibration
(Proposition~\ref{prop:soft}): a degree-$d_r\!=\!2$ ideal probe of
the $L\!=\!2$ feature class is calibrated to reach $\ell\!\le\!4$,
and the implemented SPN exhibits measurable recoverability at the
boundary and near-collapse one degree above.  The implemented SPN
is a phase-blind norm-extracting head; what we measure is its
empirical recoverability profile, not a function-class span
statement (Remark~\ref{rem:scope}).

\paragraph{What the $n\!=\!5$ headline measures, and what it does not.}
The $\Xi\!=\!11.7$ figure is computed from $n\!=\!5$ SPN-head seeds
attached to one trained $L\!=\!2$ backbone on the pre-specified
low-leakage train/val/test split (split-seed~13; chosen by the
split-leakage screen of \S\ref{sec:protocol} using only split-level
SH-coefficient leakage scores, before any backbone or SPN
evaluation).  This is a
within-backbone reading: it characterizes the cliff conditional on
a backbone that preserved the high-$\ell$ residual.  Backbone-to-
backbone variation is a separate axis
(Appendix~\ref{sec:reliability}'s spectral-neglect phenomenon) and
we report it directly in Appendix~\ref{app:bbseed_cliff}: across
$n\!=\!4$ \emph{independent} $L\!=\!2$ backbones, the SPN gain
$\Delta$ shrinks by $5.7\!\times$ between $\ell\!=\!4$ and
$\ell\!=\!5$ (mean $\Delta(4)\!=\!0.142$ vs.\ $\Delta(5)\!=\!0.025$\,
kcal/mol/\AA), showing the \emph{boundary-vs-above contrast} at the
across-backbone scale, with point-estimate spread $[0.036,0.283]$ at
$\ell\!=\!4$ that explains why we report $\rho\!=\!0.913$ as a
single-backbone measurement rather than a hierarchical population
estimate.  A full hierarchical CI \emph{for the anchor-relative recovery
fraction $\rho$} would require matched $L\!=\!4$ retraining at every
backbone seed and is out of scope here; the denominator-free
across-backbone $\Delta$ contrast (Figure~\ref{fig:cliff}C, also
hierarchical cluster bootstrap on backbones) supplies the
population-scale corroboration instead.  Within-backbone $\Xi$ and
across-backbone $\Delta$ together pin down the cliff at both scales
we can currently afford.

\paragraph{Evidence hierarchy (read alongside Figure~\ref{fig:cliff}).}
\textit{Primary:} the within-checkpoint $\rho$-cliff
(Table~\ref{tab:cliff}, Figure~\ref{fig:cliff}A).
\textit{Independent corroboration:} the denominator-free
$R^2_{\mathrm{inj}}$ on autograd-differentiated injection forces
(Figure~\ref{fig:cliff}B) and the across-backbone $\Delta$-contrast
on $n\!=\!4$ independent $L\!=\!2$ backbones
(Figure~\ref{fig:cliff}C, Appendix~\ref{app:bbseed_cliff}).
\textit{Non-claims:} EquiformerV2 is a within-ceiling positive
diagnostic outcome (above-ceiling cells not tested);
PaiNN is a weak-baseline rescue case ($r_{\mathrm{self}}$ does not
separate within from above); the protein appendix is an input-side
one-hop bandwidth measurement, not a function-class statement about
end-to-end protein models.

\begin{figure}[h]
\centering
\includegraphics[width=\textwidth]{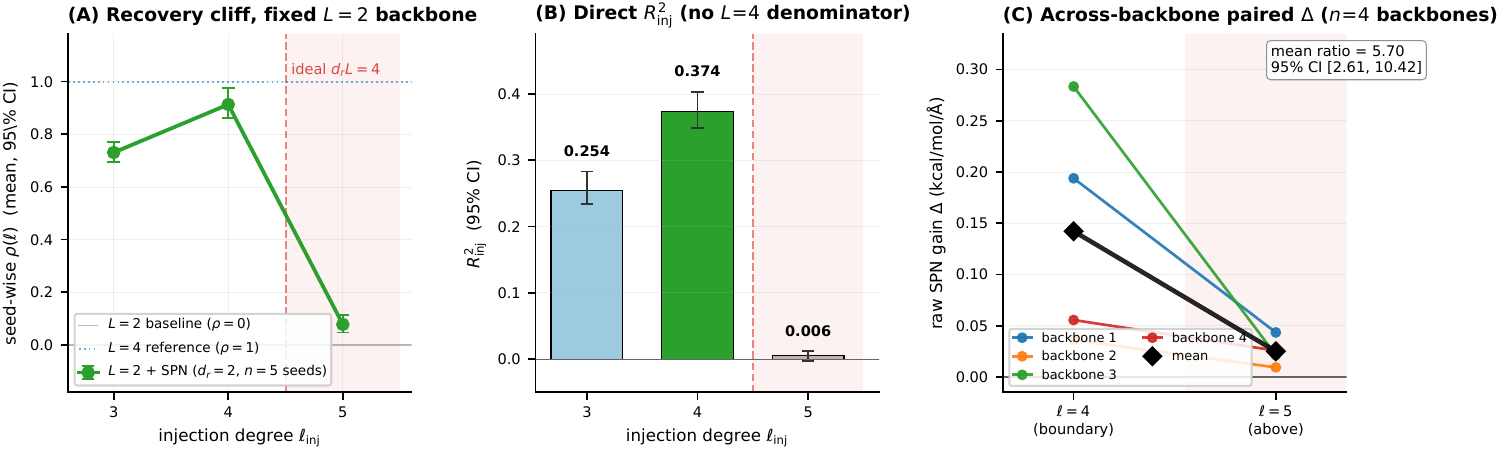}
\caption{Three complementary measurements of the aspirin spectral
cliff at the ideal $d_r L\!=\!4$ boundary.
\textbf{(A)} Anchor-relative recovery on a fixed $L\!=\!2$ NequIP
backbone drops from $\rho(4)\!=\!0.913$ [$0.861$, $0.975$] to
$\rho(5)\!=\!0.078$ [$0.049$, $0.113$].
\textbf{(B)} A denominator-free injected-residual metric shows the
same boundary-vs-above contrast:
$R^2_{\mathrm{inj}}(4)\!=\!0.374$ [$0.348$, $0.404$] versus
$R^2_{\mathrm{inj}}(5)\!=\!0.006$ [$-0.003$, $0.012$].
\textbf{(C)} Across $n\!=\!4$ independently trained $L\!=\!2$
backbones, the raw SPN gain $\Delta$ at $\ell\!=\!4$ exceeds that
at $\ell\!=\!5$ by a mean ratio of $5.70$ (95\% CI $[2.61, 10.42]$,
hierarchical cluster bootstrap).
Panels B and C provide complementary checks against
$L\!=\!4$-denominator dependence and single-checkpoint dependence:
all three panels separate the boundary cell ($\ell\!=\!4$) from the
above-ceiling cell ($\ell\!=\!5$) with disjoint 95\% CIs, although
they measure different quantities (anchor-relative recovery~$\rho$,
denominator-free injected-residual~$R^2_{\mathrm{inj}}$, and
across-backbone raw SPN gain~$\Delta$).}
\label{fig:cliff}
\end{figure}

\paragraph{Basis non-vacuity check.}
A body-frame decomposition of uninjected energies confirms that the
high-$\ell$ basis functions used by the diagnostic are not vacuous
on the molecular geometries we test (Appendix~\ref{app:natural_spectrum},
Table~\ref{tab:natural_spectrum}); body-frame angular power is a
basis-occupancy property, not a model-required-$L$ lower bound.

\subsection{Information vs.\ Geometric Capacity (summary)}
\label{sec:separation}
Width (an information-capacity lever) does not extend the
diagnostic's reach: at $\mathrm{nf}\!=\!33$ ($+8\%$ params,
no SPN), $\rho(4)\!=\!-0.42$.  The SPN ($d_r\!=\!2$,
$+5.7\%$ params, geometric capacity) recovers $\rho(4)\!=\!0.94$
with $n\!=\!16$ SPN seeds.  The width-vs-SPN-vs-depth comparison
(Table~\ref{tab:checkmate} in Appendix~\ref{app:separation_full})
shows that the SPN matches an 8-layer L=2 backbone at $42\!\times$
fewer additional parameters; depth at fixed $(L_{\mathcal{B}},d_r)$
improves within-ceiling fidelity but does not move the cliff
location (Lemma~\ref{lem:composition}).

\paragraph{Cross-molecule reads.}
On four sGDML CCSD/CCSD(T) molecules, aspirin and malonaldehyde
yield boundary contrast under sufficient injected signal; ethanol
and toluene fall below the SNR/denominator thresholds and are
reported as diagnostic limitations
(Appendix~\ref{app:crossmol_full}).
\label{sec:cross}

\paragraph{Cross-architecture matrix.}
Across four CG backbones (Table~\ref{tab:cubic_xarch_main}), only
NequIP and MACE have both within- and above-ceiling cells under a
matched protocol; $r_{\mathrm{self}}$ magnitudes are not comparable
across rows. Raw values: Appendix~\ref{app:cubic_full}.

\begin{table}[h]
\centering
\caption{Cross-architecture cubic-readout regime classification
(not a magnitude comparison). $r_{\mathrm{self}}$ uses each
architecture's raw frozen backbone as base (no trained linear head);
NequIP's $\rho_{\mathrm{anchor}}$ is matched against $L\!=\!4$.
Raw values: Appendix~\ref{app:cubic_full}.}
\label{tab:cubic_xarch_main}
\small
\begin{tabular}{lcccc}
\toprule
Backbone & $L_{\mathcal{B}}$ & $d_r$ & tested cells & interpretation \\
\midrule
NequIP        & 2 & 3 & within + above & cliff measurement \\
MACE          & 2 & 3 & within + above & cliff measurement \\
EquiformerV2  & 3 & 3 & within only    & not a cliff test \\
PaiNN         & 1 & 3 & weak baseline  & not a cliff test (rescue) \\
\bottomrule
\end{tabular}
\end{table}

\section{Related Work}
\label{sec:related}

\textbf{Spectral bias and equivariant ceilings.}
Spectral bias~\cite{rahaman2019spectral} is a learning-order
preference for low frequencies in MLPs; we characterise a different
object---an algebraic ceiling at $d_r L$ on degree-bounded
single-direction polynomial readouts of equivariant features.

\textbf{Body order and basis completeness.}
ACE~\cite{drautz2019atomic,drautz2020atomic,bachmayr2022atomic}
and recursive CG~\cite{nigam2020recursive} establish basis
completeness on $\nu$-body invariants up to $\nu L$;
Proposition~\ref{prop:soft} restates this with a multiplicity-one
witness (Proposition~\ref{prop:saturation},
Appendix~\ref{app:full_proofs}). Pozdnyakov \& Ceriotti's
\cite{pozdnyakov2020incompleteness} configuration-space resolving
question is orthogonal to our spectral-support
$\Pl_{L,d}\!=\!\Hl_{\le dL}$; beyond degree-bounded
probes~\cite{nigam2024completeness,bigi2024wigner} restore universal
approximation at polynomial cost.

\textbf{Equivariant architectures.}
TFN~\cite{thomas2018tensor,kondor2018clebsch},
Cormorant~\cite{anderson2019cormorant},
NequIP~\cite{batzner2022nequip}, MACE~\cite{batatia2022mace},
Allegro~\cite{musaelian2023allegro},
Equiformer~\cite{liao2023equiformer} truncate SH features at some
$L$. Our calibration applies when the predictor is read through
degree-bounded CG irrep-grade bookkeeping; each architecture still
needs an empirical diagnostic
(Section~\ref{sec:cliff}, Appendix~\ref{sec:cubic}).
PaiNN~\cite{schutt2021painn}, SchNet~\cite{schutt2017schnet}
illustrate low-$L$ consequences.

\textbf{Design space, expressiveness, frames.}
Design-space mappings~\cite{batatia2024designspace,xie2025price,xu2024pace}
report diminishing returns at high $L$ or $\nu$ (matched by
Corollary~\ref{cor:finiteK} when $dL\!\ge\!\ell_{\max}$).
Expressivity hierarchies~\cite{joshi2023expressive,dym2021universality,pacini2025universality}
and irrep-utilisation~\cite{lee2024deconstructing} are consistent
with our boundary; frame
methods~\cite{puny2022frame,duval2023faenet} break equivariance for
higher angular content (our body-frame injection is analytical).

\section{Conclusion}
\label{sec:conclusion}

A spectral-injection diagnostic calibrated by the single-direction
polynomial-span theorem (Proposition~\ref{prop:soft}) measures
trained-backbone angular recoverability---empirical, not a
function-class bound. The aspirin $d_r L\!=\!4$ boundary appears
under three complementary measurements (Figure~\ref{fig:cliff}) with
$C_5$ controls, separated from parameter count and limited by SNR,
frame, and backbone fidelity (Remark~\ref{rem:scope}; details in
Appendices~\ref{app:steth_workflow}, \ref{app:geom_info_capacity},
\ref{app:fidelity_bound_full}, \ref{app:design_principles},
\ref{app:reliability_horizon}).

\bibliographystyle{plain}

\appendix

\section{Cross-Molecule Detailed Analysis}
\label{app:crossmol_full}

We apply the diagnostic to four chemically diverse molecules from
the sGDML CCSD/CCSD(T) dataset: aspirin (21~atoms, aromatic +
multi-functional), ethanol (9~atoms, aliphatic flexibility),
malonaldehyde (9~atoms, intramolecular proton transfer), and toluene
(15~atoms, aromatic rigidity).  The cross-molecule reading is
\emph{amplitude- and SNR-dependent diagnostic outcomes}, not
uniform evidence of a universal molecular cliff
(Table~\ref{tab:crossmol_outcomes}).  The cleanest within-vs-above
contrast is malonaldehyde at $8\!\times$ amplitude
(Table~\ref{tab:malo_cliff}, sign reversal between $\ell\!=\!4$ and
$\ell\!=\!5$); ethanol and toluene are below-SNR or
denominator-fragile cases that we report as diagnostic limitations
rather than as cliff measurements.

\begin{table}[h]
\centering
\caption{Cross-molecule diagnostic outcomes.  ``Cliff measurement''
means a sharp within-to-above boundary contrast: large positive
recovery at or below the predicted boundary and near-zero or
negative SPN gain $\Delta$ above it.  In malonaldehyde this appears as
a sign reversal of raw $\Delta$; in aspirin as a near-collapse of
$\rho$.  ``Below SNR'' means the injected variance share is below
the heuristic amplitude-selection rule of
Remark~\ref{rem:snr_threshold} at the listed amplitude.
``Denominator-fragile'' means the $L\!=\!4$ reference overfits the
small training set, making $\rho$ ill-defined.}
\label{tab:crossmol_outcomes}
\small
\setlength{\tabcolsep}{4pt}
\resizebox{\textwidth}{!}{%
\begin{tabular}{lccl}
\toprule
Molecule & amplitude & regime at $\ell\!=\!4\!\to\!5$ & note \\
\midrule
aspirin          & $4\!\times$ & cliff measurement
                 & $\rho(4)\!=\!0.913\!\to\!\rho(5)\!=\!0.078$ (Table~\ref{tab:cliff}) \\
malonaldehyde    & $8\!\times$ & cliff measurement (sign reversal)
                 & $\Delta(4)\!>\!0$, $\Delta(5)\!<\!0$ (Table~\ref{tab:malo_cliff}) \\
ethanol, toluene & $4\!\times$ & below SNR / denom.\ fragile
                 & reported in appendix as diagnostic limitation \\
\bottomrule
\end{tabular}}
\end{table}

The full $L\!=\!2$/$L\!=\!4$ gap table, the malonaldehyde across-$\ell$
sweep, and the ethanol/toluene SNR analysis are below; the $\rho$
metric is reported when the $L\!=\!4$ reference admits a positive
denominator and we report raw $\Delta$ otherwise.

\begin{table}[h]
\centering
\caption{Raw $L\!=\!2$ vs.\ $L\!=\!4$ reference gap across four
molecules at $l_{\mathrm{inj}} = 4$ ($4\times$ amplitude).  These
gaps motivate molecule-specific diagnostic reads; they do not by
themselves establish a uniform molecular cliff.  Magnitude reflects
training-data scale and molecule-specific force scale, not the
algebraic ceiling.}
\label{tab:cross}
\small
\begin{tabular}{llcccc}
\toprule
Molecule & PES Character & $y_{L=2}$ & $y_{L=4}$ & Gap & $\sigma_F$ \\
\midrule
Aspirin       & Multi-functional   & 0.169 & 0.123 & 0.045 & 35.7 \\
Ethanol       & Aliphatic          & 0.181 & 0.143 & 0.038 & 29.0 \\
Malonaldehyde & Proton transfer    & 0.218 & 0.210 & 0.008 & 34.0 \\
Toluene       & Aromatic rigid     & 0.337 & 0.430$^\dagger$ & $-$0.093 & 46.4 \\
\bottomrule
\multicolumn{6}{l}{\footnotesize $^\dagger$~$L\!=\!4$ model overfits on small toluene dataset (950 training frames, 15 atoms).}\\
\end{tabular}
\end{table}

The $L\!=\!2$ vs.\ $L\!=\!4$ gap at $l\!=\!4$ is consistent across
aspirin (0.045) and ethanol (0.038).  Malonaldehyde shows a smaller
gap at $l\!=\!4$ (0.008), but this is consistent with the boundary
reading: $l=4$ sits \emph{at} the $L\!=\!2$ ceiling boundary
($d \cdot L = 4$), so the
gap should be small.  The critical test is whether the gap
\emph{increases} for $l > d \cdot L$.  Table~\ref{tab:malo_cliff}
supports this: the gap increases sharply from $l=4$ to $l=5$,
across the predicted boundary.

\begin{table}[h]
\centering
\caption{Malonaldehyde $L\!=\!2$ vs.\ $L\!=\!4$ gap across injection
frequencies.  The gap jumps $8\times$ between $l\!=\!4$ (at the ceiling)
and $l\!=\!5$ (above the ceiling), consistent with the diagnostic cliff at
$d \cdot L = 4$.  At $l\!=\!6$, both models exceed their effective
ceilings and the larger $L\!=\!4$ model overfits.}
\label{tab:malo_cliff}
\small
\begin{tabular}{ccccc}
\toprule
$l_{\mathrm{inj}}$ & $y_{L=2}$ & $y_{L=4}$ & Gap & Ceiling status \\
\midrule
2 & 0.162 & 0.161 & 0.002 & $l < d \cdot L$ \\
4 & 0.218 & 0.210 & 0.008 & $l = d \cdot L$ \\
5 & 0.287 & 0.224 & \textbf{0.063} & $l > d \cdot L$ \\
6 & 0.342 & 0.366$^\dagger$ & $-$0.024 & $l \gg d \cdot L$ \\
\bottomrule
\multicolumn{5}{l}{\footnotesize $^\dagger$~$L\!=\!4$ model also degrades; both models above their respective ceilings.}\\
\end{tabular}
\end{table}

This multi-frequency sweep is arguably stronger evidence than the
single-point gap in Table~\ref{tab:cross}: it shows both the
\emph{location} and the \emph{sharpness} of the cliff within a single
molecule.  At $l\!=\!6$, both models operate above their effective
ceiling ($l > d \cdot L$ for $L\!=\!2$ and approaching $d \cdot L$
for $L\!=\!4$), so the $L\!=\!4$ advantage vanishes and the gap
reverses due to overfitting of the larger model.

Toluene presents a negative gap at $l\!=\!4$
(Table~\ref{tab:cross}), with the $L\!=\!4$ backbone
\emph{worse} than $L\!=\!2$.  This is not a failure of the spectral
ceiling but a finite-data overfitting effect: toluene has only
950~training frames for 15~atoms, and the $L\!=\!4$ model's extra
capacity cannot be utilized without sufficient data.  Importantly,
at $l\!=\!2$ (below the ceiling for both models), the gap is
negligible (0.002), showing that overfitting is triggered by
the additional unused capacity at higher~$L$.

For cross-molecule SPN measurements we report $n\!=\!5$ independent
seeds per molecule (aggregated from 126 runs; see artifact manifest).
Cliff
\emph{detectability} tracks the injection signal-to-noise ratio
predicted by Remark~\ref{rem:snr_threshold}
($\eta_{\min}\!\approx\!20\%$ at our test-set size and $\alpha\!=\!0.05$).
Aspirin, with variance share $\eta\!\approx\!25\%$ above threshold,
yields the sharp cliff $\rho(4)\!=\!0.913\,[0.861,0.975]$,
$\rho(5)\!=\!0.078\,[0.049,0.113]$.  Malonaldehyde, with comparable
$\eta$, shows the same boundary-vs-above pattern on a chemically distinct system:
$\rho(4)\!=\!1.38\,[1.22,1.54]$ collapses to $\rho(5)\!=\!0.002\,[0.000,0.004]$
(sharpness $\Xi\!=\!\rho(4)/\rho(5)\!\ge\!700$, 95\% CI
well-separated from unity).  The above-unity $\rho(4)$ on malonaldehyde
reflects a finite-data effect on the $L\!=\!4$ reference---the SPN
marginally outperforms $L\!=\!4$ at one molecule, orthogonal to the
cliff transition itself (the collapse at $l\!=\!5$ is unambiguous).
Ethanol, with $\eta\!\approx\!9.5\%$ \emph{below} the detection
threshold at $4\!\times$ amplitude, produces recovery values
$\rho(4)\!=\!0.098\,[0.077,0.117]$ and
$\rho(5)\!=\!0.099\,[0.055,0.140]$ with overlapping $95\%$~CIs---the
cliff signal falls below the statistical floor, consistent with the
heuristic amplitude-selection rule of Remark~\ref{rem:snr_threshold}
when $\eta\!<\!\eta_{\min}$.
\emph{$8\!\times$ amplitude diagnostic check.}
Retraining ethanol $L\!=\!2$ backbones at $8\!\times$ amplitude
(raising $\eta$ above threshold) and probing with the same SPN head
produces a sharp cliff: $n\!=\!5$ SPN seeds give
$\Delta(\ell\!=\!4)\!=\!5.4\pm0.09$\,kcal/mol/\AA\ (within ceiling,
large recovery) and $\Delta(\ell\!=\!5)\!=\!0.06\pm0.04$\,kcal/mol/\AA\
(above ceiling, collapse)---a ratio of
$\Delta(4)/\Delta(5)\!\approx\!90$, sharper than the aspirin cliff,
consistent with the heuristic amplitude-selection reading once
$\eta$ crosses the detectability rule of
Remark~\ref{rem:snr_threshold}.  The same backbone and SPN that
produced a flat $\rho$-profile at $4\!\times$ amplitude thus produce
an unambiguous cliff once the injection signal variance-share is
raised.  We read this as empirical support for the heuristic
amplitude-selection rule rather than as a formal validation: the
cliff is consistent with being present at the diagnostic level, with
amplitude controlling detectability.  An independent cross-molecule check on
malonaldehyde (also $L\!=\!2$ at $8\!\times$, $n\!=\!5$ SPN seeds)
shows the same boundary-vs-above \emph{sign} pattern:
$\Delta(\ell\!=\!4)\!=\!+0.37\!\pm\!0.06$\,kcal/mol/\AA\ (within
ceiling, recovery) versus $\Delta(\ell\!=\!5)\!=\!-0.16\!\pm\!0.14$
(above ceiling, SPN actively degrades---all five seeds in $[-0.31,
-0.01]$).  The ethanol-vs-malonaldehyde $\Delta$-magnitude difference
reflects molecule-specific force scales; the architecture-agnostic
signal is the sign reversal across the predicted ceiling.  Toluene at $l\!=\!4$ exhibits non-positive
$\rho$ denominators (the $L\!=\!4$ reference over-fits the 950-frame
training set at 15~atoms), so we report raw gaps in
Table~\ref{tab:cross} rather than $\rho$ for that molecule.  In no
molecule does a reversed cliff ($\rho(5)\!>\!\rho(4)$) appear;
Table~\ref{tab:malo_cliff} locates the cliff at the algebraic
boundary $l\!=\!4$ for the single molecule (malonaldehyde) where all
five seeds clear the SNR threshold at multiple $l$~values.

\paragraph{Cross-molecule extension at higher~$L$ and~$d_r$.}
The cubic readout extension (Appendix~\ref{sec:cubic}) shows an
analogous cross-molecule pattern.  At $L\!=\!2$ on EquiformerV2, the cubic SPN
($d_r\!=\!3$, ideal boundary $d_r L\!=\!6$) yields strictly positive
recovery on ethanol, malonaldehyde, and toluene at every
$\ell_{\mathrm{inj}}\!\in\!\{4,5,6\}$
(Table~\ref{tab:cubic_eqv2_xmol}, $\rho\!\in\![0.29,0.68]$).  At
$L\!=\!3$ on the same three molecules the cubic SPN at
$\ell_{\mathrm{inj}}\!\in\!\{5,6,7\}$ collapses to $\rho\!\le\!0.005$
across $9\!\times\!5\!=\!45$ runs.  All probed cells are within the
predicted ceiling $d_r L\!=\!9$, so this is a within-boundary null outcome
(see Appendix~\ref{sec:cubic} for a careful reading: the null is consistent
with fidelity-saturation, fidelity-bound, or training-schedule
explanations, but our experiments do not distinguish them and we do
not present it as positive evidence for reach).  The interpretable
cross-molecule signal is the $L\!=\!2$ EqV2 result
(Table~\ref{tab:cubic_eqv2_xmol}, $\rho\!\in\![0.29,0.68]$); the
$L\!=\!3$ null lift is included as a regime-boundary observation,
not as evidence that the cliff repeats at higher backbone $L$.

\paragraph{Frame-robustness (summary).}
The cliff location is invariant to body-frame choice (two atom triples
tested, same $L\!=\!2$ backbone); details and table are in
Appendix~\ref{app:frame_robustness}.

\paragraph{Architectural controls (L=0 and non-equivariant baseline).}
$L\!=\!0$ (SchNet) and a plain MLP both fail to show any cliff signature,
consistent with the diagnostic depending on equivariant SH features.
Full numbers, the C$_5$ synthetic capacity-independence figure, and the
$L_b\!\in\!\{0,1,2,3,4\}$ boundary-shift sweep are in
Appendix~\ref{app:arch_controls}.

\paragraph{Cubic readouts and cross-architecture probes (summary).}
We extend the SPN to a cubic readout ($d_r\!=\!3$) and apply it to MACE,
EquiformerV2, and PaiNN backbones.  Table~\ref{tab:cubic_xarch_summary}
summarises the cross-architecture diagnostic outcomes by regime
(NequIP/MACE: cliff measurements; EquiformerV2: within-only positive
diagnostic outcome; PaiNN: weak-baseline rescue).  Per-architecture
tables (probe-degree sweep, MACE cubic, PaiNN cubic, EqV2 cross-molecule),
the width-driven cubic transition, and the within-boundary null-lift regime
at $L\!=\!3$ are deferred to Appendix~\ref{app:cubic_full}.


\paragraph{Amplitude/per-layer/width-independence ablations.}
Amplitude invariance, per-layer spectral probing, and width-independence
ablations all support the cliff signature without changing the diagnostic's
core conclusions.  See Appendix~\ref{app:micro_ablations} for full details.

\section{Information vs.\ Geometric Capacity (full discussion)}
\label{app:separation_full}

A natural objection is that the SPN's advantage might stem from its
additional parameters rather than its polynomial structure.  We
introduce two controls that cleanly separate \emph{information
capacity} (parameter count) from \emph{geometric capacity} (the
ideal probe boundary $d_r L$).

\textbf{Capacity null (nf33).}
We train NequIP($L\!=\!2$) with $\text{num\_features}=33$ (915K
parameters, 8\% more than the 849K baseline) on the same injected
datasets.  If spectral reach scaled with parameter count, nf33 should
show improved recovery at $l > 2L$.

\begin{table}[h]
\centering
\caption{Recovery fraction comparison: information capacity (nf33)
vs.\ geometric capacity (SPN) on aspirin ($4\times$ amplitude,
split seed~13; $n\!=\!16$ SPN seeds on this backbone give CV$\!=\!1.7\%$
around $\rho\!=\!0.94$).  Adding 8\% more parameters (nf33) yields
zero measurable recovery beyond the within-boundary baseline;
adding a polynomial readout (SPN, 5.7\% overhead) recovers a large
anchor-relative fraction at the boundary.  Multi-backbone aggregate
in Table~\ref{tab:cliff}.}
\label{tab:separation}
\small
\begin{tabular}{cccccc}
\toprule
$l_{\mathrm{inj}}$ & $y_{L=2}$ & $y_{\text{nf33}}$ & $\rho_{\text{nf33}}$
  & $y_{\text{SPN}}$ & $\rho_{\text{SPN}}$ \\
\midrule
3 & 0.164 & 0.164 & $\mathbf{0.00}$ & 0.137 & 0.77 \\
4 & 0.171 & 0.187 & $\mathbf{-0.42}$ & 0.134 & 0.94 \\
5 & 0.142 & 0.168 & $\mathbf{-1.96}$ & 0.139 & 0.18 \\
\bottomrule
\end{tabular}
\end{table}

Table~\ref{tab:separation} shows the result: nf33 provides
\emph{zero} measurable recovery beyond the within-boundary baseline.
Its recovery fraction is indistinguishable from zero at $l\!=\!3$
and \emph{negative} at $l\!=\!4$ and $l\!=\!5$ (additional features
slightly degrade
performance, consistent with overfitting to a wider but
spectrally-equivalent feature space).  Meanwhile, the SPN---with
\emph{fewer} additional parameters (48K vs.\ 66K)---achieves 94\%
recovery at $l\!=\!4$ (mean over $n = 16$ seeds, CV~$= 1.7\%$).

This establishes a clean separation:
\begin{itemize}
  \item \textbf{Information capacity} (adding parameters without
    changing polynomial structure): no measurable recovery beyond
    the within-boundary baseline.
  \item \textbf{Geometric capacity} (changing the polynomial
    degree~$d$ of the readout): the ideal degree-$d$ probe is
    calibrated to the $d\!\cdot\!L$ boundary.
\end{itemize}

\textbf{Frequency contrast on nf33.}
The separation is further confirmed by probing the \emph{same} nf33
backbone at different injection frequencies.  Grafting an SPN onto the
nf33 backbone at $l\!=\!3$ (within the $d \cdot L = 4$ ceiling)
recovers a test force MAE of 4.26~kcal/mol/\AA{} from a baseline of
5.90---an improvement of 1.64~kcal/mol/\AA.  At $l\!=\!4$ (at the
ceiling boundary), $\rho = -0.03$: the SPN cannot recover the signal.
At $l\!=\!5$ (above the ceiling), the SPN's validation loss rises
monotonically from 5.82 to 6.66 across 110 epochs---actively
\emph{hurting} performance.  Same backbone, same SPN architecture,
different injection frequency.  This rules out any capacity-based
explanation: the nf33 backbone has sufficient information capacity
to support SPN recovery at $l\!=\!3$ but not at $l\!\ge\!4$,
consistent with the ideal $d\cdot L$ calibration.

\textbf{Depth sweep.}
\label{app:depth}
The compositional Lemma~\ref{lem:composition}(i) gives a strict
output ceiling of $d_r L_{\mathcal{B}}$ that is \emph{independent of
backbone depth $T$}, with each layer's truncation gating any
transient higher-$\ell$ content (Lemma~\ref{lem:composition}(iii)).
At idealized truncation, depth therefore should not extend reach.
Empirically, however, more interaction layers do improve the
recovery fraction at fixed $L_{\mathcal{B}}$
(Table~\ref{tab:depth}, $\rho(4)$ rises from $-0.31$ at $T\!=\!2$ to
$0.96$ at $T\!=\!8$).  We read this as a partial-truncation effect,
not as a shift in the ideal boundary: scalar channels
($\Hl_0$) propagated through a residual / radial-MLP path interact
with subsequent CG couplings to produce within-ceiling fidelity
gains that the strict Lemma~\ref{lem:composition} idealization does
not capture.  The cliff \emph{location} (cf.\
Appendix~\ref{sec:cubic}, $\ell^\star\!=\!d_r L_{\mathcal{B}}$) does
not move with depth; what depth provides is improved within-ceiling
fidelity, which raises $\rho$ at $\ell\!=\!4$ when the $L\!=\!4$
reference is the normaliser.  We retain Table~\ref{tab:depth} as the
practical ranking among width / depth / SPN routes to a target
$\rho$, and use it to anchor the cost-tradeoff conclusion (the SPN
matches 8-layer reach at $42\!\times$ fewer additional parameters);
the strict-theorem reading is that depth should not move the cliff,
which is what we observe in the cubic-readout cross-architecture
sweep (Appendix~\ref{sec:cubic}).

\begin{table}[h]
\centering
\caption{Depth sweep at $l_{\mathrm{inj}} = 4$ with $L\!=\!2$.
At fixed $(L_{\mathcal{B}},d_r)$ Lemma~\ref{lem:composition} predicts a
depth-invariant strict cliff location.  Empirically, more interaction
layers raise the recovery fraction $\rho(4)$, which we attribute to
within-ceiling fidelity gain (and partial-truncation leakage in the
implementation), not to a shift in the ideal boundary.  The SPN
achieves $\rho(4)\!\approx\!0.94$ with $42\!\times$ fewer additional
parameters than the 8-layer route.}
\label{tab:depth}
\small
\begin{tabular}{lcccr}
\toprule
Architecture & Layers & $y_{\mathrm{ema}}$ & $\rho(4)$ & Total params \\
\midrule
NequIP($L\!=\!2$) & 2 & 0.183 & $-0.31$ & 71{,}808 \\
NequIP($L\!=\!2$) & 4 & 0.171 & $0.00$ & 849{,}024 \\
NequIP($L\!=\!2$) & 6 & 0.143 & $0.71$ & 1{,}850{,}496 \\
NequIP($L\!=\!2$) & 8 & 0.133 & $\mathbf{0.96}$ & 2{,}851{,}968 \\
\midrule
L=2 + SPN ($d\!=\!2$) & 4+head & 0.134 & $\mathbf{0.94}$ & 897{,}138 \\
NequIP($L\!=\!4$) & 4 & 0.131 & $1.00$ & 2{,}155{,}648 \\
\bottomrule
\end{tabular}
\end{table}

The depth sweep (Table~\ref{tab:depth}) reveals a monotonic
relationship between depth and the recovery fraction at
$l\!=\!4$: 2~layers give $\rho<0$, 6~layers recover 71\%, and
8~layers recover 96\%.  Under the strict reading of
Lemma~\ref{lem:composition} this should not happen (the output
ceiling is depth-independent at $d_r L_{\mathcal{B}}$); we attribute
the empirical depth-driven $\rho$ improvement to within-ceiling
fidelity gains and partial-truncation leakage in NequIP's actual
implementation, not to a shift in the ideal boundary.  The cliff
\emph{location}---which is what the theorem strictly predicts---is
tested separately and depth-invariantly via the cubic-readout
cross-architecture experiments of Appendix~\ref{sec:cubic}.


Table~\ref{tab:checkmate} consolidates these findings into a single
comparison at $l_{\mathrm{inj}} = 4$, the $d \cdot L$ boundary for $L\!=\!2$.
The three routes to boundary-cell recovery---width, depth, and
polynomial readout structure---are compared on equal footing:

\begin{table}[h]
\centering
\caption{Routes to within-ceiling recovery at $l_{\mathrm{inj}} = 4$
(aspirin, $4\times$).  Width (nf33) provides zero recovery despite
adding parameters.  Depth (6--8 layers) improves within-ceiling
fidelity at high parameter cost (Lemma~\ref{lem:composition} forbids
a strict cliff-location shift; the empirical $\rho$ gain is a
fidelity effect).  The SPN ($d_r\!=\!2$) attains large
anchor-relative recovery at the boundary with 5.7\% overhead;
the denominator-free metric reports partial injected-force-variance
recovery (see Remark~\ref{rem:rho_supplementary}).  $\rho$ is the recovery
fraction (Definition~\ref{def:rho}).}
\label{tab:checkmate}
\small
\begin{tabular}{lrcc}
\toprule
Architecture & Params & Force MAE$^\dagger$ & $\rho(4)$ \\
\midrule
NequIP($L\!=\!2$, 4 layers) & 849K & 6.11 & 0.00 \\
\quad + width (nf=33) & 915K & 6.69 & $-0.42$ \\
\quad + width (nf=33) + SPN & 963K & 6.15 & $-0.03$ \\
\quad + depth (6 layers) & 1,850K & 5.11 & 0.71 \\
\quad + depth (6 layers) + SPN & 1,899K & 5.11 & 0.71 \\
\quad + depth (8 layers) & 2,852K & 4.76 & 0.96 \\
\quad + \textbf{SPN ($d\!=\!2$, $n\!=\!16$)} & \textbf{897K} & \textbf{4.79} & \textbf{0.94} \\
\midrule
NequIP($L\!=\!4$, 4 layers) & 2,156K & 4.69 & 1.00 \\
\bottomrule
\multicolumn{4}{l}{\scriptsize $^\dagger$EMA checkpoint, kcal/mol/\AA.}
\end{tabular}
\end{table}

Three results are visible in Table~\ref{tab:checkmate}.  First,
\emph{width cannot substitute for reach}: the nf33 backbone (915K
params) achieves $\rho = -0.42$ (worse than the baseline), and
grafting an SPN onto it recovers only to $\rho = -0.03$---when the
backbone lacks spectral fidelity, even a polynomial readout cannot
recover the missing signal.  Second, \emph{depth improves within-ceiling fidelity at high cost}:
6 layers achieve $\rho = 0.71$, but adding an SPN on top yields
\emph{zero} further gain ($\rho = 0.71$; the SPN never saved a
checkpoint across 150 training epochs).  Under the strict reading
of Lemma~\ref{lem:composition}, depth does not move the cliff
location; we read this saturation as the deeper backbone having
already absorbed the recoverable within-ceiling content, leaving
nothing for the SPN to add.
Third, \emph{the SPN is the parameter-efficient route}: on the
shallow (4-layer) backbone, it achieves $\rho = 0.94$ ($n\!=\!16$
seeds) with 5.7\% parameter overhead, matching the 8-layer result
($\rho = 0.96$) at one-third the parameter budget.

\section{SPN Architecture and Algorithm (full)}
\label{app:spn_full}

The SPN serves as a \emph{diagnostic probe} of recoverable spectral
residue in a backbone's frozen features and, when such residue is
preserved by the backbone, as a lightweight corrective head.

\paragraph{Notation: readout body-order vs.\ backbone body-order.}
Throughout the paper we distinguish two spectral-capacity parameters
(Lemma~\ref{lem:composition}):
\begin{itemize}
\item $d_r$: the \emph{readout} polynomial degree---the number of CG
tensor-product couplings from the backbone interface to the
prediction.  The SPN has $d_r=2$ (norm extraction is quadratic); a
purely linear head would have $d_r=1$.
\item $d_b$: the \emph{backbone} cumulative body-order---for $T$
stacked interaction layers with correlations $\nu_1,\ldots,\nu_T$, the
equivariant features at the backbone output have body-order
$d_b=\prod_{t=1}^T\nu_t$ in the input edge vectors, up to truncation to
$\Hl_{\le L_{\mathcal{B}}}$.
\end{itemize}
By Lemma~\ref{lem:composition}(i), the \emph{operative} ceiling at the
predictor output is $d_r\!\cdot\!L_{\mathcal{B}}$.  The backbone's
higher-order body content $d_b\!\cdot\!L_{\mathcal{B}}$ is realized
internally but is gated by the truncation to~$L_{\mathcal{B}}$ at every
interface; content at angular degrees above~$L_{\mathcal{B}}$ is not
visible to the readout.  This has two consequences central to our
experimental design:
\begin{itemize}
\item[(a)] The cliff location at $\ell_{\mathrm{inj}}\!=\!d_r L_{\mathcal{B}}+1$
predicted by Proposition~\ref{prop:soft} is a property of the
\emph{readout}, not of $T$ or $\{\nu_t\}$. This is why the cliff is
empirically invariant to backbone depth (Section~\ref{sec:cliff},
$T\!\in\!\{2,4,6,8\}$).
\item[(b)] The high error exhibited by the \mbox{$L\!=\!2$} baseline at
$\ell_{\mathrm{inj}}\!>\!L_{\mathcal{B}}$ (denominator of~$\rho$,
Definition~\ref{def:rho}) is an \emph{empirical} performance gap
rather than a theorem prediction: the vanilla linear head realizes
$d_r\!=\!1$, whose theoretical ceiling is $\Hl_{\le L_{\mathcal{B}}}$,
but its measured shortfall at $\ell\!>\!L_{\mathcal{B}}$ reflects the
combination of this ceiling with optimization-limited saturation of
the backbone's $d_b L_{\mathcal{B}}$ internal capacity.  The
recovery fraction $\rho$ normalizes out both effects by anchoring to
the $L\!=\!4$ baseline.
\end{itemize}

\textbf{Architecture.}
The SPN is a lightweight readout head attached to a frozen equivariant
backbone (Algorithm~\ref{alg:spn}).  From an $L\!=\!2$ NequIP backbone
(849K parameters), we extract the intermediate equivariant features
($32 \times \{0^e, 0^o, 1^e, 1^o, 2^e, 2^o\}$, 192 channels).
An invariant extractor computes per-channel norms for $l > 0$
components ($\sqrt{\sum_m |c_l^m|^2}$) and passes $l\!=\!0$ values
directly, yielding 192 invariant features.  An MLP maps these to
$(L_{\mathrm{out}}+1)^2$ invariant scalar readout weights indexed by
$(\ell,m)$. These weights do not transform as equivariant
spherical-harmonic coefficients; the $(\ell,m)$ indexing is used only
to form the diagnostic power summary
$P_\ell\!=\!\sum_m(a_\ell^m)^2$ for each degree.  A final energy
MLP predicts per-atom energy contributions.

\begin{algorithm}[h]
\caption{SPN forward pass (per-atom energy).}
\label{alg:spn}
\begin{algorithmic}[1]
\Require atomic environment $\{\mathbf{r}_{ij}\}$; frozen backbone
$\mathcal{B}$ with truncation $L_{\mathcal{B}}$; readout polynomial
degree $d_r$; SH output truncation $L_\mathrm{out}$.
\State $\mathbf{h}_i \gets \mathcal{B}(\{\mathbf{r}_{ij}\})
  \in \bigoplus_{\ell=0}^{L_{\mathcal{B}}} V_\ell$
\Comment{frozen backbone, no gradient}
\State $\mathbf{s}_i \gets
  \bigl(\mathbf{h}_i^{(\ell=0)},\;
   \{\lVert \mathbf{h}_i^{(\ell)}\rVert\}_{\ell=1}^{L_{\mathcal{B}}}\bigr)$
\Comment{invariant extractor; $d_r$-fold norms for $\ell\!>\!0$}
\State $\mathbf{a}_i \gets \mathrm{MLP}_\theta(\mathbf{s}_i)
  \in \mathbb{R}^{(L_\mathrm{out}+1)^2}$
\Comment{invariant scalar spectral weights $a_i^{\ell m}$ (real; not equivariant SH coefficients)}
\State $P_i[\ell] \gets \textstyle\sum_{m=-\ell}^{\ell} (a_i^{\ell m})^2$
\Comment{per-degree sum-of-squares, $\ell\!=\!0,\ldots,L_\mathrm{out}$}
\State $E_i^{\mathrm{SPN}} \gets g_\phi(P_i)$
\Comment{scalar energy MLP}
\State \Return $E_i = E_i^{\mathcal{B}} + E_i^{\mathrm{SPN}}$
\Comment{forces obtained via autograd through both terms}
\end{algorithmic}
\end{algorithm}

\textbf{What this probes.}
The invariant-extraction step is degree-bounded by construction (a
$d_r$-fold CG contraction of degree-$L$ features), placing the
\emph{ideal} polynomial probe boundary at $d_r L$ via
Proposition~\ref{prop:soft}.  The implemented SPN further passes
those invariants through a scalar MLP, so we do \emph{not} claim
that the implementation realises the full
$\Pl_{L,d_r}\!=\!\Hl_{\le d_r L}$ span as a function class; the
per-degree invariant
$\sqrt{\sum_m \lvert h_i^{(\ell,m)}\rvert^2}$ in line~2 is
phase-blind and not a strict polynomial of $\mathbf{h}_i$, so $d_r$
denotes the nominal CG-contraction order / ideal probe boundary
rather than a strict polynomial degree.  We use the SPN as a
calibrated empirical diagnostic and stress-test its effective
spectral behaviour by the degree, capacity, and activation
ablations of Section~\ref{sec:results}.

\begin{remark}[Irrep type is not coordinate bandwidth]
\label{rem:irrep_type}
A scalar MLP applied to SO(3)-invariant inputs remains a scalar
irrep.  This is a statement about \emph{representation type}, not a
bound on the coordinate-frequency bandwidth of the MLP's output as a
function of atomic positions.  For example, a nonlinear function
$g(\hat u\!\cdot\!\hat v)$ is rotationally invariant yet can contain
arbitrarily high Legendre modes as a function of the relative angle.
Consequently, our formal polynomial-span propositions apply to
degree-bounded \emph{polynomial} probes; the implemented MLP-SPN is
treated as an empirical diagnostic and we do not claim it is
algebraically degree-bounded.  The identity/square/SiLU ablation
(\S\ref{sec:cliff}, ``Scalar-activation independence'') is reported
as an empirical negative control consistent with this calibration,
not as a proof that scalar MLPs cannot synthesise higher-frequency
coordinate content.
\end{remark}

\textbf{Training.}
The backbone is frozen (no gradient flow); only the SPN head
(48K parameters, 5.7\% of the backbone) is trained.  The total
prediction is $E = E_{\mathrm{backbone}} + E_{\mathrm{SPN}}$, with
forces via autograd.

\section{Full Proofs of Section~\ref{sec:theory} Theorems}
\label{app:full_proofs}

This appendix contains the proofs deferred from
Section~\ref{sec:theory}'s compressed main statement of the
polynomial--spectral correspondence (Proposition~\ref{prop:soft}
and Lemma~\ref{lem:composition}), together with the auxiliary
propositions, corollaries, and scope remarks that the main
text references.

\paragraph{Architectural scope.}
Proposition~\ref{prop:soft} and the auxiliary propositions below
are statements about the single-marked-direction polynomial space
$\Pl_{L,d}$.  They apply directly to any equivariant predictor
that, on a fixed body frame and for a fixed atomic environment,
expresses its scalar output as a polynomial of degree $\le d_r$ in
the direction-vector spherical harmonics of one marked edge.  For
a multi-atom MPNN using radial-times-SH embeddings coupled by
Clebsch--Gordan tensor products---the class that includes
TFN~\cite{thomas2018tensor}, NequIP~\cite{batzner2022nequip},
MACE~\cite{batatia2022mace},
Allegro~\cite{musaelian2023allegro},
Equiformer~\cite{liao2023equiformer}, and EquiformerV2 (via the
SO(2) reduction~\cite{passaro2023reducing})---the operative
output-side irrep-grade ceiling $d_r L$ at the predictor head is
supplied by the compositional Lemma~\ref{lem:composition} below,
which performs CG irrep-grade bookkeeping through the
message-passing graph by tracking the maximum number of CG
couplings on any forward path.  Cartesian-tensor constructions
(e.g., TensorNet~\cite{simeon2023tensornet}) are covered via
$\text{rank-}r\!\cong\!\bigoplus_{\ell\le r}\Hl_\ell$.
Architectures that abandon exact SO(3) equivariance (EGNN,
frame-breaking augmentation) lie outside scope; the theorem still
bounds their equivariant subcomponent but allows
symmetry-violating pathways to leak content across the boundary.
We do \emph{not} claim that an MPNN's full computational class
equals $\Hl_{\le dL}$ as a function of all atomic positions: the
inter-atomic geometry contributes additional structure that lies
outside a single-direction polynomial probe.  The cliff
predictions tested in Section~\ref{sec:results} are predictions
about the SPN diagnostic's reach (a single-direction probe by
construction), not literal set-equality on the MPNN function
class.

\subsection{Proof of the polynomial--spectral correspondence}

\begin{proof}[Proof of Lemma~\ref{lem:cg}]
By the Gaunt product formula
\[
  Y_{\ell_1}^{m_1}Y_{\ell_2}^{m_2}
  =\sum_{\ell,m}G^{\ell m}_{\ell_1m_1,\ell_2m_2}Y_\ell^m,
\]
nonzero only under the Clebsch--Gordan selection rules
$|\ell_1-\ell_2|\!\le\!\ell\!\le\!\ell_1+\ell_2$ and parity
$\ell_1\!+\!\ell_2\!+\!\ell$ even. Multiplying by a degree-$\ell_k\!\le\!L$
harmonic therefore raises the maximum degree by at most $\ell_k$.
Iterating over $\prod_{k=1}^d Y_{\ell_k}^{m_k}$ gives support only
up to $\sum_k\ell_k\le dL$.
\end{proof}

\begin{proposition}[Spectral Ceiling, used internally]
\label{prop:ceiling}
For $L\ge 1$ and $d\ge 1$: $\Pl_{L,d}\subseteq\Hl_{\le dL}$.
Equivalently, content at angular frequency $\ell>dL$ is
unreachable by any degree-$d$ polynomial readout of degree-$L$
features, regardless of width, parameterization, or training.
\end{proposition}

\begin{proof}
Immediate from Lemma~\ref{lem:cg} and linearity: every monomial
in $\Pl_{L,d}$ lies in $\Hl_{\le dL}$, and $\Pl_{L,d}$ is the
linear span of such monomials.
\end{proof}

\begin{proposition}[Spectral Saturation, used internally]
\label{prop:saturation}
For $L\ge 1$ and $d\ge 1$: $\Hl_{\le dL}\subseteq\Pl_{L,d}$.
Moreover, the multiplicity of $\Hl_{dL}$ in the decomposition of
$\Hl_L^{\otimes d}$ is exactly one, so the bound is saturated by
a unique (up to scalar) $d$-fold stretched-state product.
\end{proposition}

\begin{proof}[Proof (highest weight)]
For any $0\!\le\!n\!\le\!dL$, write $n=qL+r$ with $0\!\le\!r\!<\!L$
(so $q\!\le\!d$). The product
$(Y_L^L)^q\,Y_r^r$ is a polynomial of degree $q\!+\!\mathbf{1}_{r>0}\!\le\!d$
in the entries of $\phi_L$, hence lies in $\Pl_{L,d}$. By the
Gaunt product formula
\begin{align*}
  Y_{\ell_1}^{m_1}Y_{\ell_2}^{m_2}
  &=\sum_{\ell,m}G^{\ell m}_{\ell_1m_1,\ell_2m_2}\,Y_\ell^m, \\
  G^{\ell m}_{\ell_1m_1,\ell_2m_2}
  &=\sqrt{\tfrac{(2\ell_1+1)(2\ell_2+1)}{4\pi(2\ell+1)}}\,
  \langle\ell_1\,0;\ell_2\,0|\ell\,0\rangle\,
  \langle\ell_1\,m_1;\ell_2\,m_2|\ell\,m\rangle,
\end{align*}
the top-weight component of $(Y_L^L)^q Y_r^r$ in
$\Hl_{qL+r}\!=\!\Hl_n$ is non-zero (the stretched-chain Gaunt
coefficients $\langle L\,L;L\,L|2L\,2L\rangle,\dots$ are all
non-zero under Condon--Shortley
normalization~\cite{varshalovich1988quantum}). Acting with the
$SO(3)$ lowering operator $L_-$ on this top-weight vector spans all
$2n\!+\!1$ basis functions $\{Y_n^m\}_{m=-n}^n$ of $\Hl_n$ inside
$\Pl_{L,d}$ (lowering preserves polynomial degree). Hence
$\Hl_n\!\subseteq\!\Pl_{L,d}$ for every $n\!\le\!dL$, giving
$\Hl_{\le dL}\!\subseteq\!\Pl_{L,d}$. Multiplicity of $\Hl_{dL}$ in
$\Hl_L^{\otimes d}$ is one because the only weight-$dL$ vector in
the tensor product is the stretched chain $(Y_L^L)^d$ up to
scalar~\cite{varshalovich1988quantum}.
\end{proof}

\begin{proof}[Proof of Proposition~\ref{prop:soft}]
Combine Propositions~\ref{prop:ceiling} and~\ref{prop:saturation}.
\end{proof}

\begin{corollary}[Minimum Polynomial Degree]
\label{cor:soft}
To reach angular frequency~$\ell$ from degree-$L$ features, a
minimum polynomial degree $d=\lceil \ell/L\rceil$ is required.
The ideal boundary $d L$ increases by $L$ degrees per unit increase
in~$d$.  (Immediate from
Proposition~\ref{prop:soft}.)
\end{corollary}

\begin{corollary}[Finite-$K$ Universality]
\label{cor:finiteK}
Let $f\in L^2(S^2)$ have angular bandwidth
$\ell_{\max}(f)=\max\{\ell:\|\Pi_\ell f\|>0\}$.  Then $f$ is
representable by a degree-$d$ polynomial readout of $\phi_L$
\emph{if and only if} $dL\ge\ell_{\max}(f)$.
\end{corollary}

\begin{proof}
\emph{Necessity:} if $dL<\ell_{\max}(f)$ then $\Pi_{\ell_{\max}}f
\ne 0$ lies in $\Hl_{\ell_{\max}}\not\subseteq\Pl_{L,d}$ by
Proposition~\ref{prop:ceiling}.
\emph{Sufficiency:} by Proposition~\ref{prop:saturation},
$\Hl_{\le dL}\subseteq\Pl_{L,d}$.  Because $\Pl_{L,d}$ is defined
as a linear span, any $f\in\Hl_{\le dL}$ is itself a finite linear
combination of basis elements each in $\Pl_{L,d}$ and is therefore
already in $\Pl_{L,d}$ by definition.
\end{proof}

\subsection{Implemented-SPN versus ideal probe}

\begin{remark}[SPN as a calibrated empirical probe]
\label{rem:spn_probe}
The implemented SPN (Algorithm~\ref{alg:spn}) extracts SO(3)
invariants of the backbone features through a degree-bounded
CG-contraction step ($d_r\!=\!2$ via per-channel norm
$\sum_m|a_\ell^m|^2$ in the standard variant; $d_r\!=\!3$ via a
fully-connected three-fold tensor product in the cubic variant)
and feeds the resulting scalars through a radial MLP.  We do
\emph{not} claim that this implemented SPN is dense in
$\Hl_{\le d_r L}$ as a function class on $S^2$: the norm-style
extractor is phase-blind and does not span the full CG monomial
basis required to attain arbitrary $Y_\ell^m$ at $\ell\!\le\!d_r L$.
The algebraic identity $\Pl_{L,d}\!=\!\Hl_{\le dL}$ is used as a
\emph{calibration target}: the SPN's CG-contraction step has a
known polynomial degree, the radial-MLP head preserves scalar
\emph{irrep type} (Remark~\ref{rem:radial_mlp_filtration})---a
representation-type filtration only, not a coordinate-frequency
bandwidth bound on the implemented MLP-SPN.  The empirical reach
of the implemented SPN is then stress-tested by (i)~the
$C_5$-symmetric synthetic experiment
(Appendix~\ref{sec:controls}, $R^2$ separation across eight $(L,d)$
cells), (ii)~the probe-degree $d_r\!\in\!\{1,2,3\}$ sweep at
fixed backbone with cliff shifting to $\ell^\star\!=\!d_r L$
(Appendix~\ref{sec:cubic}), and (iii)~the scalar-activation
identity/square/SiLU ablation (\S\ref{sec:cliff}).  An idealised
SPN that uses the explicit CG monomial basis (rather than the
phase-blind norm) would saturate the algebraic boundary by
construction; we leave this idealised variant to future work and
report the reach of the implemented diagnostic as the operative
empirical signal in the rest of the paper.
\end{remark}

\begin{remark}[Irrep-type filtration under scalar-only nonlinearities]
\label{rem:radial_mlp_filtration}
For a CG-based architecture in which every nonlinearity acts
pointwise on scalar ($\ell\!=\!0$) subchannels (radial MLPs in
NequIP/MACE, SiLU/gate activations on invariant channels in
Equiformer, scalar embeddings in Allegro), and equivariant gates
of the form $\sigma(s)\!\cdot\!v$ with
$s\in V_0\otimes\Hl_0$ and $v\in V_\ell\otimes\Hl_\ell$, the only
operation that can raise the SH \emph{irrep grade} of a feature
is a CG tensor product (which raises $\ell_{\max}$ by at most
$L$, Lemma~\ref{lem:cg}).  This is a filtration on representation
type and underlies the path-counting argument of
Lemma~\ref{lem:composition}.  Preservation of scalar irrep type
is \emph{not} a bound on the coordinate-frequency bandwidth of a
non-polynomial scalar function of CG-derived invariants, in line
with Remark~\ref{rem:irrep_type}.  The compositional output bound
of Lemma~\ref{lem:composition} is therefore stated as a
representation-grade ceiling on the predictor head, not as a
function-class-coverage statement on the full nonlinear MPNN.
\end{remark}

\subsection{Proof of the compositional output ceiling}

\begin{proof}[Proof of Lemma~\ref{lem:composition}]
\emph{Output ceiling.} Apply Proposition~\ref{prop:soft} to
$\mathcal{R}$ acting on features of maximum angular degree
$L_{\mathcal{B}}$: the readout's output lies in
$\Hl_{\le d_r L_{\mathcal{B}}}$.  Since
$\mathcal{A}=\mathcal{R}\circ\mathcal{B}$, the overall output
inherits this bound regardless of the specific polynomial
structure of $\mathcal{B}$.
\emph{Backbone internal polynomial degree (heuristic, not used in
the output ceiling).} A correlation-$\nu_t$ update at layer~$t$
that couples $\nu_t$ copies of the prior feature can be bounded
heuristically by a multiplicative recurrence
$d^{\mathcal{B}}_t\!\le\!\nu_t d^{\mathcal{B}}_{t-1}$, but real
backbones with skip connections, message-passing aggregation, and
residuals accumulate body-order additively along feature paths
rather than strictly multiplicatively. We do not rely on this
internal-degree count for the output ceiling—the latter follows
directly from applying Proposition~\ref{prop:soft} to the
$d_r$-degree readout, regardless of the backbone's internal
polynomial structure.
\emph{Truncation loses transient higher-$\ell$ content.}
Any content generated at layer~$t$ in
$\Hl_{L_{\mathcal{B}}+1},\ldots,\Hl_{d^{\mathcal{B}}_t L_{\mathcal{B}}}$
is projected out upon truncation to
$\Hl_{\le L_{\mathcal{B}}}$ and cannot propagate to later layers
or to the readout.
\end{proof}

\section{Conclusion: Full Discussion and Empirical Outcomes}
\label{app:conclusion_full}
\label{app:empirical_outcomes}

This appendix expands the main-body Conclusion
(Section~\ref{sec:conclusion}) with the broader-applicability
discussion, the precise scope of empirical claims, the empirical
diagnostic outcomes (directly observed and
conjectured-but-not-directly-shown items), and concrete future
directions.

\paragraph{Broader applicability.}
The spectral injection framework and the SPN diagnostic together give
practitioners a concrete workflow for answering the question:
\emph{does my backbone have the angular resolution this task requires?}
The $11.7\times$ sharpness ratio provides a quantitative diagnostic
signature consistent with the polynomial-span calibration and
demonstrates that the answer is quantitatively measurable rather
than a matter of trial and error.  More broadly, the diagnostic
protocol applies wherever spherical-harmonic representations meet
finite truncation---climate and weather models on the
sphere~\cite{cohen2018spherical}, robotic manipulation with SO(3)
pose estimation, and protein structure prediction---and the
single-direction polynomial-span calibration is applicable wherever
an explicit CG readout can be hooked as a single-direction probe.  The full multi-atom MPNN
function class is broader than what the single-direction probe
measures (Remark~\ref{rem:scope}); we frame the contribution as a
falsifiable diagnostic for trained-backbone reach, not a universal
architecture ceiling.

\paragraph{Scope of empirical claims.}
The algebraic ceiling $\ell\le d\!\cdot\!L$
(Proposition~\ref{prop:soft} and the auxiliary
Propositions~\ref{prop:ceiling}--\ref{prop:saturation} in
Appendix~\ref{app:full_proofs}) and its force-lift to
$\ell^\star_{\mathbf F}\!=\!dL\!+\!1$
(Remark~\ref{rem:force_side_intuition}, fixed-frame reading) are
single-marked-direction algebraic facts on $\Hl_{\le dL}$; the
compositional Lemma~\ref{lem:composition} provides the
corresponding output-side irrep-grade bookkeeping for predictor
heads.  What we demonstrate empirically is narrower and we state
it precisely below.

\paragraph{Empirical diagnostic outcomes (full list).}
Our empirical findings, listed in detail in
Appendix~\ref{app:empirical_outcomes}, comprise twelve
directly-observed diagnostic results---most importantly the aspirin
$L\!=\!2$ cliff at $\ell\!=\!4\!\to\!5$ ($\Xi\!=\!11.7$,
Table~\ref{tab:cliff}), the $5.7\!\times$ across-backbone matched
contrast (Appendix~\ref{app:bbseed_cliff}), the direct
$R^2_{\mathrm{inj}}$ corroboration
(Remark~\ref{rem:rho_supplementary}), and the cross-architecture
cubic outcomes (Table~\ref{tab:cubic_xarch_summary}).
Conjectured-but-not-directly-shown items---full $L$-sweep diagonal,
intra-$\nu$ MACE cliff, Allegro adaptation, biomolecular
generalization, and protocol constraints---are listed in the
Future Directions block below.

\paragraph{Future directions.}
Several concrete experiments would strengthen and extend these
findings.
\emph{(i)~Architecture generality.}
The single-direction polynomial calibration suggests analogous
diagnostic predictions for CG-style architectures (MACE, Allegro,
Equiformer at higher $\nu$, $L$), but each backbone requires a
matched empirical read because baseline fidelity and readout
structure differ.  Testing at $L\!=\!6$ would extend the
cliff-location diagnostic to a wider range of backbone cutoffs.
\emph{(ii)~Per-layer spectral probing.}
Attaching SPN readouts at each intermediate layer of a deep
backbone (e.g., layers~1 through~8) would directly measure how
intermediate representations preserve within-ceiling angular
fidelity across layers, decomposing the within-ceiling fidelity
gain we observe at fixed $(L_{\mathcal{B}}, d_r)$ into per-layer
contributions.
\emph{(iii)~Fidelity bound formalisation.}
Appendix~\ref{sec:fidelity} introduces $\ell_{\mathrm{fidelity}}(L)$
empirically.  A formal analysis of how $L\!=\!2$ feature attenuation
scales with target frequency~$\ell$ (via the CG coefficient decay
structure) would place this heuristic effective-reach reading on
rigorous footing.
\emph{(iv)~Per-atom energy probe.}
The empirical discussion after Remark~\ref{rem:force_side_intuition}
notes that total-energy MAE is atom-sum-cancelling; a per-atom
energy decomposition would sharpen the energy-vs-force staggered
cliff test.
\emph{(v)~Spectral content of natural forces.}
We compute the angular-frequency content of natural \emph{energies}
in Appendix~\ref{app:natural_spectrum}
(Table~\ref{tab:natural_spectrum}) and find that ${\sim}80\%$ of
the natural-energy power for aspirin and toluene sits above the
$L\!=\!2$ ceiling.  An analogous decomposition for \emph{forces}
(which involve a derivative and a higher effective frequency
content under the gradient lift) remains future work.

\section{Natural-Energy Body-Frame Decomposition}
\label{app:natural_spectrum}

As a basis-non-vacuity sanity check, we decompose each molecule's
natural (uninjected) energy in the same body-frame
spherical-harmonic basis the diagnostic uses
($l_{\max}\!=\!10$, $n\!=\!1000$ test frames per molecule;
pipeline details in the supplementary artifact manifest).  The
high-$\ell$ basis functions used by the diagnostic are not vacuous
on the molecular geometries we test: the natural-energy angular
power is non-negligible above $l\!=\!4$ for three of four molecules
in our suite (Table~\ref{tab:natural_spectrum}).  This should
\emph{not} be read as a lower bound on the angular cutoff required
by an equivariant model: the body-frame decomposition depends on
the chosen frame and mixes radial, chemical, and conformational
variation, and a regression basis with non-zero weight at degree
$\ell$ does not imply that a trained equivariant model needs
$L\!=\!\ell$.  We report it solely to confirm that the diagnostic's
high-$\ell$ basis functions exercise a real signal on these
geometries.

\begin{table}[h]
\centering
\caption{Natural-energy angular power above the $L\!=\!2$ ceiling
($l\!=\!4$) and $L\!=\!1$ ceiling ($l\!=\!2$), and the dominant
peaks of the spectrum.  For three of four molecules,
$\geq 70\%$ of the natural-energy angular content sits above the
$L\!=\!2$ algebraic ceiling.}
\label{tab:natural_spectrum}
\small
\begin{tabular}{lccl}
\toprule
Molecule & $\%$ power $l\!>\!2$ & $\%$ power $l\!>\!4$ & Dominant peaks \\
\midrule
Aspirin       & 90.5 & 81.3 & $l\!=\!6$ (35\%), $l\!=\!7$ (33\%) \\
Ethanol       & 96.3 & 72.0 & $l\!=\!5$ (41\%), $l\!=\!6$ (25\%) \\
Toluene       & 95.4 & 85.8 & $l\!=\!9$ (39\%), $l\!=\!8$ (22\%) \\
Malonaldehyde & 93.4 & 44.0 & $l\!=\!4$ (35\%), $l\!=\!5$ (30\%) \\
\bottomrule
\multicolumn{4}{l}{\scriptsize Source files listed in the artifact manifest.}
\end{tabular}
\end{table}

\section{Theory Extensions}
\label{app:theory_extensions}

This appendix collects three technical extensions of the
polynomial-span theorem of Section~\ref{sec:theory}: the
heterogeneous-stack ceiling, a fixed-frame force-side intuition,
and the relation to Dym \& Maron asymptotic universality.

\begin{remark}[Heterogeneous-stack irrep-grade bookkeeping]
\label{cor:heterogeneous}
Generalize Lemma~\ref{lem:composition} to a stack with per-layer
angular truncations $L_1,\ldots,L_T$ (equivariant features at
layer~$t$ truncated to $\Hl_{\le L_t}$ before being passed to
layer~$t+1$).  As an irrep-grade bookkeeping statement on the
explicit CG predictor head,
\begin{equation}
  \mathcal{A}\subseteq\Hl_{\le d_r L_T},
  \label{eq:heterogeneous_ceiling}
\end{equation}
where $L_T$ is the terminal backbone truncation. Earlier-layer
truncations $L_1,\ldots,L_{T-1}$ do not propagate past their
respective interfaces: transient content at $\ell>L_t$ generated
inside layer~$t$ is annihilated by the interlayer truncation
$\Pi_{\le L_t}$ and cannot reach the readout. This is a
bookkeeping statement for explicit CG predictor-head irreps under
an idealized truncation model; it is not a bound on arbitrary
nonlinear scalar functions of all atomic coordinates.
\end{remark}

\begin{proof}[Sketch]
Iterating the per-layer truncation $\Pi_{\le L_t}$ gives a
terminal irrep-grade ceiling $L_T$ on the predictor-head input;
applying Proposition~\ref{prop:soft} to the degree-$d_r$ readout
yields the inclusion above. The propagation-blocking claim is
immediate from $\Pi_{\ell}\circ\Pi_{\le L_t}=0$ for $\ell>L_t$.
\end{proof}

\paragraph{Implication.}
A down-channel (monotone-decreasing) stack
$L_1\!\ge\!\cdots\!\ge\!L_T$ has the same ideal predictor-head
boundary as a stack with constant truncation $L_T$: higher-$L$
early layers contribute only to intra-backbone feature refinement
(e.g.\ better scalar invariants at higher body-order), not to the
ideal predictor-head boundary. Conversely, an up-channel stack
$L_1\!\le\!\cdots\!\le\!L_T$ exploits the full progression, since
no information is truncated away until the readout. Mixed-$L$
architectures---common in practice to save parameters on early
layers---would be expected, under this idealized bookkeeping, to
exhibit a diagnostic boundary at $d_r L_T$ rather than at any
intermediate $d_r L_t$; direct empirical check is deferred to
future work.

\begin{remark}[Fixed-frame force-side intuition]
\label{rem:force_side_intuition}
A fixed-frame calculation suggests why an energy perturbation in
$\Hl_{\le dL}$ can appear in force space with components one degree
lower and one degree higher: writing
$\mathbf{r}\!=\!r\hat{\mathbf{r}}$ and applying
$\nabla(f(r)Y_\ell^m)=f'(r)Y_\ell^m\hat{\mathbf{r}}+
(f(r)/r)\nabla_{S^2}Y_\ell^m$~\cite{jackson1999classical}, the
surface gradient $\nabla_{S^2}Y_\ell^m$ has angular support in
$\Hl_{\ell-1}\!\oplus\!\Hl_{\ell+1}$ while the radial term preserves
$\ell$, giving a fixed-frame force-side window
$\Hl_{\le dL+1}\!\otimes\!\mathbb{R}^3$. Used as intuition, this
predicts a staggered energy/force pattern at the cliff: at
$\ell_{\mathrm{inj}}\!=\!dL\!+\!1$ the $\Hl_{dL}$ component falls
inside the window and the $\Hl_{dL+2}$ component does not, while at
$\ell_{\mathrm{inj}}\!=\!dL\!+\!2$ both components fall outside. We
do \emph{not} use this calculation as a formal guarantee: our
injected forces differentiate through the body-frame construction
(Remark~\ref{rem:force_frame}), so the gradient acts on the frame as
well as on the harmonic and can mix content beyond
$\Hl_{\le dL+1}$. The main force-space claim is the direct autograd
measurement $R^2_{\mathrm{inj}}$ in
Figure~\ref{fig:cliff}B.
\end{remark}

\begin{remark}[Body-frame caveat for autograd-differentiated injections]
\label{rem:force_frame}
The fixed-frame intuition in
Remark~\ref{rem:force_side_intuition} holds the body frame fixed
under the gradient operator. Our injection algorithm constructs the
body frame from a Gram--Schmidt triple of atomic positions
(Section~\ref{sec:methods}); when forces are obtained by autograd
through this construction, the gradient acts on the frame as well
as on the harmonic, so an injected energy with support in
$\Hl_{\le dL}$ may produce a force with mixing beyond
$\Hl_{\le dL+1}$. We do not claim that the empirical
$\rho_\mathbf{F}^{\star}$ on autograd-differentiated injection
forces matches the fixed-frame heuristic; the empirical aspirin
force collapse to $\rho_F(5)\!=\!0.078$ is the operative cliff
signal, regardless of whether any residual lift is contributed by
frame-derivative mixing. A direct empirical projection of the
autograd-differentiated injection force onto the VSH basis is left
to future work.
\end{remark}

\paragraph{Empirical scope.}
A clean empirical test of any fixed-frame staggered energy/force
intuition would require atom-local force/energy projections;
total-energy MAE mixes rotational frames across atoms and does not
directly probe the atom-local picture. We do not rely on such a test
here. The operative force-side claim in this paper is the measured
force-MAE cliff and the autograd $R^2_{\mathrm{inj}}$ diagnostic
(Section~\ref{sec:cliff}), with aspirin
$\rho_F(5)\!=\!0.078$ (95\% CI $[0.049,0.113]$). The per-atom
energy probe and the cross-architecture ratio
$\rho_E(dL\!+\!1)/\rho_F(dL\!+\!1)$ are left to future work.

\paragraph{Initial cross-architecture probe of the staggered cliff (PaiNN).}
A preliminary cross-architecture measurement on a PaiNN backbone
(effective $L\!\approx\!1$, predicted energy ceiling
$\ell^\star_E\!=\!2$, force ceiling $\ell^\star_\mathbf{F}\!=\!3$) on
aspirin shows the predicted staggered pattern: at the energy ceiling
$\ell_{\mathrm{inj}}\!=\!2$ both recovery fractions are large and
similar ($\rho_\mathbf{F}\!=\!0.78\!\pm\!0.02$,
$\rho_E\!=\!0.76\!\pm\!0.09$, ratio $\approx\!0.98$, $n\!=\!5$ seeds);
across $\ell_{\mathrm{inj}}\!\in\!\{3,4,5\}$ the ratio
$\rho_E/\rho_\mathbf{F}$ decays as $0.78\!\to\!0.43\!\to\!0.25$, with
$\rho_E$ collapsing faster than $\rho_\mathbf{F}$, consistent with
energy reach saturating one frequency earlier than force reach.  The
analogous probe on NequIP backbones produced larger numerical drift in
$\rho_E$ owing to baseline-energy redefinition between frozen-backbone
and SPN-attached evaluations, so we report the PaiNN result as initial
support for the staggered cliff and leave a comprehensive
multi-architecture energy probe to future work.

\paragraph{Fixed-frame heuristic ratio (intuition only).}
The fixed-frame VSH calculation
(\cite[\S3.5--3.6]{jackson1999classical}) suggests an intuitive
fraction
$\alpha(dL)\!=\!(dL\!+\!1)/(2dL\!+\!3)$ that the $\Hl_{dL}$
component of an $\ell_{\mathrm{inj}}\!=\!dL\!+\!1$ injected force
contributes to the fixed-frame window: the orbital decomposition
$\nabla_{S^2}Y_\ell^m\propto\!\Psi_{\ell m}\!=\!\sqrt{\ell/(2\ell+1)}\,\mathbf{Y}_{\ell-1,\ell m}\!+\!\sqrt{(\ell+1)/(2\ell+1)}\,\mathbf{Y}_{\ell+1,\ell m}$
distributes scalar content as $\ell\!:\!(\ell+1)$, giving
$\alpha(4)\!=\!5/11\!\approx\!0.45$, $\alpha(8)\!=\!9/19\!\approx\!0.47$.
This is a fixed-frame heuristic ratio, not a bound the paper relies
on: our injected forces differentiate through the body-frame
construction, so the gradient acts on the frame as well as on the
harmonic and can mix content beyond the fixed-frame window. The
empirical aspirin force collapse $\rho_F(5)\!=\!0.078\pm0.013$
(Section~\ref{sec:cliff}) and the autograd $R^2_{\mathrm{inj}}$
metric in Figure~\ref{fig:cliff}B are the operative force-side
evidence.

\begin{remark}[Relation to asymptotic universality]
\label{rem:dymmaron}
The single-direction polynomial-span identity
$\Pl_{L,d}\!=\!\Hl_{\le dL}$ is a finite-resolution
\emph{algebraic} statement about degree-bounded polynomial probes
of complete spherical-harmonic features.  It is \emph{compatible
with} (and not in conflict with) asymptotic universality results
for CG-based equivariant architectures~\cite{dym2021universality,joshi2023expressive},
which establish density as $L,d\!\to\!\infty$.  We do \emph{not}
claim that the full nonlinear multi-atom CG-MPNN function class is
exactly $\Hl_{\le dL}$ at finite $(L,d)$; rather, $dL$ calibrates
the ideal single-direction polynomial probe (and the explicit
CG irrep-grade bookkeeping of Lemma~\ref{lem:composition}).  The
two readings answer complementary questions: asymptotic
universality asks ``is every equivariant function reachable in the
limit?''; Proposition~\ref{prop:soft} asks ``which spectral
components does a degree-bounded single-direction polynomial probe
span at finite $(L,d)$?''.  The empirical aspirin cliff
(Section~\ref{sec:cliff}) is consistent with the polynomial-span
calibration as a diagnostic outcome, not as a function-class
equality on the full MPNN.
\end{remark}

\paragraph{Frame-breaking architectures (scope note).}
Architectures that break SO(3) equivariance (e.g.\ EGNN-style nets,
frame-augmented non-equivariant models, AlphaFold3-style invariant
diffusion) lie outside the scope of the polynomial-span calibration:
they can in principle resolve $\Hl_{\ell}$ at any $\ell$ via
frame-breaking pathways.  See Appendix~\ref{app:frame_break} for the
heuristic frame-break discussion (out of scope).

\section{Reliability Horizon (full discussion)}
\label{app:reliability_horizon}

\label{sec:reliability}

The spectral ceiling is not a static boundary but a transition into a
stochastic learning regime, where the reliability of signal recovery is
inversely proportional to the gap between the signal frequency~$l$ and
the backbone's effective reach~$d \cdot L$.  Below~$d \cdot L$, recovery
is deterministic: multiple backbone seeds yield consistently high~$\rho$
with low variance.  \emph{At} the boundary $l = d \cdot L$, recovery
becomes stochastic---backbone-to-backbone variance increases sharply, and
the success of the SPN probe depends on whether a particular backbone
seed happened to preserve high-angular-frequency correlations.  Above
$d \cdot L$, recovery collapses uniformly: no seed, no initialization,
and no readout architecture recovers the injected signal.

We observe this transition empirically at $l = 4$ for the $d = 2$ SPN
on an $L = 2$ backbone.  One backbone seed yields $\rho = 0.94$ while
another yields $\rho \approx 0$---a coefficient of variation that
exceeds~1.  Critically, a multi-SPN-seed stability probe ($n = 16$
independent SPN initializations per frozen backbone) shows that this
variance originates in the \emph{backbone}, not the SPN head.
On the high-recovery backbone, all 16 seeds converge to a tight bundle
(CV~$= 1.7\%$, 95\% CI $[4.75,\, 4.83]$~kcal/mol/\AA;
Figure~\ref{fig:convergence}), while on the
low-recovery backbone, all 16 seeds remain flat near the baseline
(CV~$= 0.3\%$, $\Delta = 0.005$~kcal/mol/\AA).  The SPN is a stable
instrument; the backbone is the stochastic element.

\begin{figure}[h]
\centering
\includegraphics[width=0.85\textwidth]{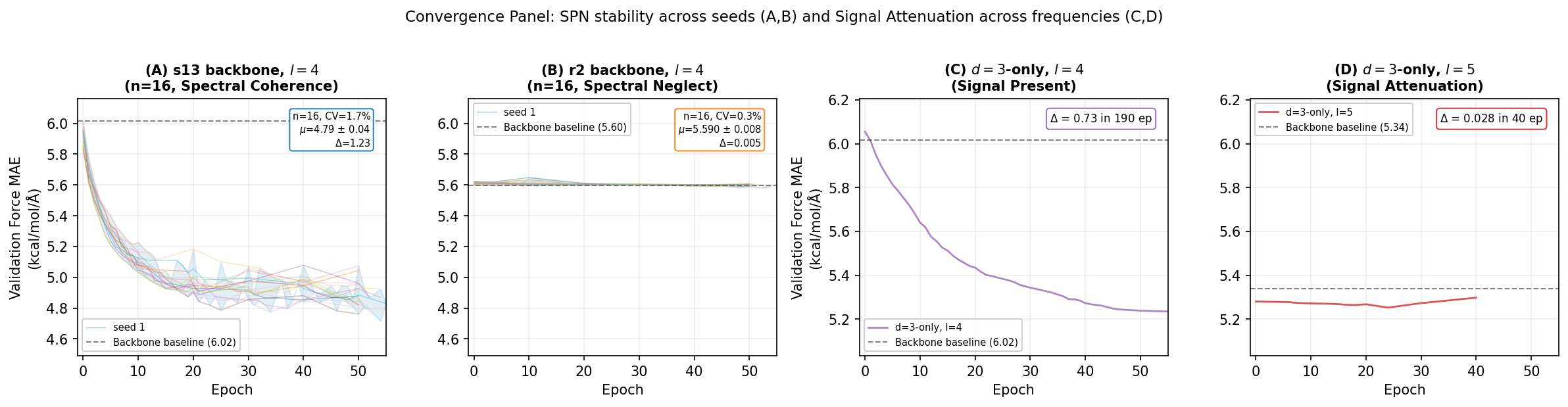}
\caption{SPN convergence across $n\!=\!16$ independent seeds on the
high-recovery backbone ($l_{\mathrm{inj}} = 4$, aspirin).  All seeds
converge to a tight bundle (CV~$= 1.7\%$), showing that the SPN
is a stable diagnostic whose variance is dominated by backbone
properties, not readout initialization.}
\label{fig:convergence}
\end{figure}

This stochasticity reveals a \emph{decoupling} between total force
error and angular resolution at the spectral boundary.  The backbone
seed with lower global MAE (5.60~kcal/mol/\AA) exhibits
$\rho = 0.37 \pm 0.01$ ($n = 16$~SPN seeds), while the seed with
\emph{higher} MAE (6.02~kcal/mol/\AA) recovers
$\rho = 0.94 \pm 0.04$.  We attribute this to \emph{spectral neglect}:
the lower-MAE backbone found a basin that fits dominant low-frequency
force components efficiently, at the cost of discarding
high-angular-frequency coherence that the SPN requires.  This
decoupling implies that aggregate metrics can mask significant losses
in geometric fidelity---motivating the spectral stethoscope as a
complementary diagnostic that detects resolution deficits invisible
to standard validation.

The most extreme case is instructive.  The backbone seed with the
\emph{lowest} global MAE (5.37~kcal/mol/\AA) exhibits near-zero
spectral recovery.  Training an SPN on this backbone for 150 epochs
produces monotonically \emph{rising} validation force MAE (from 5.64
to 5.95~kcal/mol/\AA); the SPN never saves a best checkpoint because
it never improves over the baseline.  The SPN is not a generic
accuracy booster---when there is no high-frequency signal to recover,
it correctly reports null, and forcing it to train simply adds noise.
This falsifies the hypothesis that the SPN's improvements stem from
regularization or implicit data augmentation: a pure regularization
effect would benefit all backbones equally, but the SPN's benefit is
entirely conditional on the presence of resolvable spectral content.

The reliability horizon has further practical consequences.  When a
practitioner measures $\rho(l)$ at the $d \cdot L$ boundary, a single
backbone seed may overestimate or underestimate spectral reach.
Multiple backbone replicates are essential at the boundary---though not
below it, where recovery is reliable, nor above it, where failure is
unambiguous.  Spectral blindness is a property of the \emph{checkpoint},
not the architecture: the same model class has the polynomial degree to
access the signal, but optimization determines whether a given seed
preserves it.  The stethoscope thus serves as a fail-fast mechanism,
identifying spectral neglect before a model is deployed to tasks
requiring high-angular-frequency resolution.



\section{Design Principles for Equivariant Architectures (full)}
\label{app:design_principles}

Our results distill into three principles for practitioners:

\begin{enumerate}
  \item \textbf{Diagnose before scaling.}  Raising~$L$ is expensive
    ($O(L^3)$ CG coupling cost).  The spectral stethoscope measures
    whether the target property actually \emph{requires} higher
    angular resolution, avoiding blind architectural scaling.
    The per-atom analysis (Appendix~\ref{app:per_atom}) further
    localizes spectral deficits to specific atoms, guiding
    targeted interventions.
  \item \textbf{Depth empirically improves the recovery
    fraction; the cliff location is depth-invariant.}  At fixed
    $(L_{\mathcal{B}}, d_r)$, deeper backbones yield higher
    within-ceiling $\rho$ (Table~\ref{tab:depth}); the cubic-readout
    cliff \emph{location} is at $\ell^\star\!=\!d_r L_{\mathcal{B}}$
    independently of depth (Appendix~\ref{sec:cubic}), as
    Lemma~\ref{lem:composition} predicts.  We interpret the depth
    effect on $\rho$ as a within-ceiling fidelity gain (and partial
    truncation leakage in NequIP's implementation), not as a
    shift in the ideal boundary: depth is therefore
    a practical lever for $\rho$, not a substitute for $d_r$ or $L$.
  \item \textbf{Width cannot substitute for reach.}  The nf33 null
    result is unequivocal: adding parameters within a fixed polynomial
    structure provides zero spectral extension.  Before attributing a
    performance gap to ``underfitting,'' practitioners should check
    that the gap is not a spectral ceiling---i.e., that the target
    content lies within the backbone's $d \cdot L$ range.
\end{enumerate}

\textbf{Angular content of natural force fields.}
The spectral ceiling is not an artifact of our injection protocol:
the natural (uninjected) energy surfaces themselves contain
substantial power above the $L\!=\!2$ readout ceiling
(Section~\ref{sec:cliff}, Table~\ref{tab:natural_spectrum}).

\textbf{Predictions for other architectures.}
Because the $d \cdot L$ ceiling follows from the SO(3) tensor product
algebra (Proposition~\ref{prop:soft}), it generates falsifiable
predictions for any CG-based architecture.
Table~\ref{tab:predictions} lists predicted spectral ceilings.
For explicit body-order models like MACE, the readout-side $d_r$ is
the correlation order~$\nu$ of the head; for implicit message-passing
models (NequIP, Equiformer), $d_r$ is set by the readout's
polynomial structure (e.g., a square or cubic invariant extractor).
The output ceiling is $d_r L_{\mathcal{B}}$ in both cases via
Lemma~\ref{lem:composition}; depth serves as an empirical lever for
within-ceiling fidelity rather than as a reach extender (cf.\
Table~\ref{tab:depth} interpretation in Section~\ref{sec:cliff}).
Each prediction is directly testable via our injection + SPN
protocol.

\begin{table}[h]
\centering
\caption{Predicted output-side spectral ceilings $\ell^\star\!=\!d_r
L_{\mathcal{B}}$ via Lemma~\ref{lem:composition}, with $d_r$ the
readout-side polynomial degree (the number of CG couplings on the
readout pathway).  At fixed $(L_{\mathcal{B}}, d_r)$ depth does not
move $\ell^\star$; depth is an empirical within-ceiling fidelity
lever (Table~\ref{tab:depth}).  $\ell^\star$ is an upper bound
subject to the fidelity bound (Appendix~\ref{sec:fidelity}).}
\label{tab:predictions}
\small
\begin{tabular}{llccl}
\toprule
Architecture & Config & $L_{\mathcal{B}}$ & $d_r$ & Predicted $\ell^\star$ \\
\midrule
SchNet                & 6 layers, linear head & 0 & 1 & 0 (isotropic) \\
PaiNN                 & 6 layers, $d_r\!=\!2$ head & 1 & 2 & $\le 2$ \\
MACE ($\nu\!=\!2$)    & 2 layers, $d_r\!=\!1$ head & 2 & 2 & $\le 4$ \\
MACE ($\nu\!=\!1$)    & 2 layers, $d_r\!=\!1$ head & 2 & 1 & $\le 2$ \\
Equiformer            & 4 layers, $d_r\!=\!1$ head & 2 & 1 & $\le 2$ \\
NequIP $L\!=\!2$ + cubic SPN & $d_r\!=\!3$ head & 2 & 3 & $\le 6$ \\
\bottomrule
\multicolumn{5}{l}{\scriptsize $d_r$ here is the readout polynomial
degree, not a layer-count proxy.}
\end{tabular}
\end{table}

\section{Fidelity Limit: Ideal Boundary $\ne$ Empirical Recoverability (full)}
\label{app:fidelity_bound_full}

\label{sec:fidelity}

The $d \cdot L$ ceiling is an \emph{upper bound} on the ideal probe
boundary: degree-$d$ polynomials of degree-$L$ features \emph{can}
represent harmonics up to~$dL$, but empirical recovery additionally
requires that the backbone features preserve
sufficient angular coherence at those frequencies.
The situation is analogous to Shannon's channel capacity: the capacity
is an exact upper bound, practical codes may not achieve it, but it
nevertheless determines what is and is not \emph{possible}.
We call the gap between the ideal readout boundary and the
empirically achievable recovery the \emph{fidelity limit}: a
backbone-dependent attenuation that limits how much of the ideal
boundary is empirically recoverable.

We demonstrate the fidelity bound by varying degree~$d$ and backbone~$L$
independently while measuring SPN improvement~$\Delta$ over the
backbone-only baseline.

\paragraph{Adding cubic terms ($d\!=\!3$) to a low-$L$ backbone.}
We extend the standard quadratic SPN ($d\!=\!2$) with cubic invariants
constructed via two successive CG tensor products, raising the algebraic
ceiling from $dL\!=\!4$ to $dL\!=\!6$ on an $L\!=\!2$ backbone.  At
$l\!=\!4$, this reduces aspirin force MAE from $5.96 \to 4.75$
kcal/mol/\AA{} ($n\!=\!5$ seeds, $\rho \!=\! 0.95$)---a marginal
improvement over $d\!=\!2$ alone ($\rho \!=\! 0.91$), as expected since
the $d\!=\!2$ ceiling already covers $l\!\le\!4$.  At $l\!=\!5$,
\emph{within the new $dL\!=\!6$ ceiling but outside $dL\!=\!4$}, the
cubic SPN nevertheless recovers only $\rho\!=\!0.14$ vs.~$\rho\!=\!0.08$
for $d\!=\!2$---a 6\% absolute increase, far short of the algebraic
prediction.  The same pattern holds across molecules
(Table~\ref{tab:fidelity_d3}).  The bottleneck is not capacity: it is
that the $L\!=\!2$ backbone has already attenuated $l\!\ge\!5$ content
below the noise floor of what cubic invariants can recover.

\begin{table}[h]
\centering
\caption{Cross-molecule $d\!=\!3$ SPN gain over $L\!=\!2$ backbone
($n\!=\!5$ seeds).  $d\!=\!3$ raises the algebraic ceiling from
$dL\!=\!4$ to $dL\!=\!6$, yet improvement collapses for
$l\!\ge\!5$---consistent with the fidelity bound being the binding constraint
across chemistry.}
\label{tab:fidelity_d3}
\small
\begin{tabular}{lccc}
\toprule
Molecule & $\Delta(l\!=\!4)$ & $\Delta(l\!=\!5)$ & $\Delta(l\!=\!6)$ \\
\midrule
Aspirin       & $1.21 \pm 0.02$ & $0.07 \pm 0.02$ & $0.08 \pm 0.01$ \\
Ethanol       & $0.14 \pm 0.20$ & $0.01 \pm 0.03$ & $0.34 \pm 0.09$ \\
Malonaldehyde & $0.58 \pm 0.03$ & $0.01 \pm 0.01$ & $0.05 \pm 0.03$ \\
Toluene       & $0.02 \pm 0.02$ & ---             & $0.01 \pm 0.01$ \\
\bottomrule
\multicolumn{4}{l}{\scriptsize $\Delta$ in kcal/mol/\AA{} (lower~MAE = larger $\Delta$).}\\
\end{tabular}
\end{table}

\paragraph{Increasing $L$ moves the fidelity bound up.}
Conversely, raising backbone~$L$ \emph{does} extend the fidelity
bound, because higher-$L$ features preserve more of the
high-angular-frequency signal.  Table~\ref{tab:fidelity_L} measures
SPN gain~$\Delta$ across all tested~$l$ for $L\!=\!3$ and $L\!=\!4$
backbones on aspirin.  For $L\!=\!3$, $\Delta$ remains small ($\le
0.03$) up to $l\!=\!6$ (within capacity~$dL\!=\!6$), then jumps
$10\times$ to $\Delta\!=\!0.18$ at $l\!=\!7$---the predicted force
ceiling $\ell^*_F\!=\!dL\!+\!1\!=\!7$.  At the boundary, the SPN
recovers what the backbone alone leaves on the table.  For $L\!=\!4$
(force ceiling~$\ell^*_F\!=\!9$), $\Delta$ stays uniformly small
($\le 0.06$) across all $l\!\in\!\{2,...,9\}$ tested: the backbone is
spectrally complete everywhere we can measure, exhausting the
fidelity bound at each frequency.

\begin{table}[h]
\centering
\caption{SPN gain $\Delta$ (kcal/mol/\AA{}) over $L\!=\!3$ and $L\!=\!4$
backbones on aspirin, $n\!=\!5$ seeds.  $L\!=\!3$ shows the predicted
boundary spike at the force ceiling $\ell^*_F\!=\!7$; $L\!=\!4$ stays
saturated everywhere within reach (force ceiling $\ell^*_F\!=\!9$,
not testable here).}
\label{tab:fidelity_L}
\small
\setlength{\tabcolsep}{4pt}
\resizebox{\textwidth}{!}{%
\begin{tabular}{l c c c c c c c c}
\toprule
$l_{\mathrm{inj}}$ & 2 & 3 & 4 & 5 & 6 & 7 & 8 & 9 \\
\midrule
$L\!=\!3,\;d\!=\!2$ ($\ell^*_F\!=\!7$)
 & 0.025 & 0.029 & 0.018 & 0.005 & 0.016 & \textbf{0.175} & --- & --- \\
$L\!=\!4,\;d\!=\!2$ ($\ell^*_F\!=\!9$)
 & 0.063 & 0.006 & $\mathbf{0.018}_{\pm.012}$ & 0.009 & 0.014 & 0.006 & 0.005 & 0.014 \\
\bottomrule
\end{tabular}}\\[2pt]
{\scriptsize Bold $\ell\!=\!7$: predicted force-ceiling boundary
effect.  Bold $\ell\!=\!4$ ($L\!=\!4$): mean over $n\!=\!4$
independent backbone seeds (others single-seed).}
\end{table}

These two experiments isolate the two factors of effective reach:
adding capacity~($d$) at low~$L$ does not help when the backbone has
already washed out the signal; adding fidelity~($L$) saturates the
recovery curve until the next ceiling is approached.  The effective
recoverable degree is the minimum of the algebraic ceiling and the
backbone's signal fidelity:
\[
  l_{\max}^{\mathrm{eff}} = \min\!\bigl(d \cdot L,\;
  l_{\mathrm{fidelity}}(L)\bigr),
\]
where $l_{\mathrm{fidelity}}(L)$ is the highest angular frequency at
which the backbone preserves sufficient coherence for polynomial
extraction.  Note that $l_{\mathrm{fidelity}}$ depends not only on~$L$
but also on the \emph{training seed}: the seed-to-seed variation in
spectral recovery (Appendix~\ref{sec:reliability}) contributes to the
empirical fidelity bound.  The $L\!=\!3$ boundary
spike at $l\!=\!7$ (Table~\ref{tab:fidelity_L}) is the SPN's report
of where the diagnostic indicates a backbone-fidelity limit.
Formalizing this bound is an open theoretical question.

\paragraph{Cross-architecture and cross-molecule diagnostic outcomes.}
The cubic-readout extension (Appendix~\ref{sec:cubic}) sharpens the
same boundary-vs-fidelity dichotomy across four architectures and
four molecules. Where the ideal readout boundary is increased
($d_r$ raised from~$2$ to~$3$) and the backbone retains slack to
fit, the within-ceiling lift is large and reproducible: NequIP
$L\!=\!2$ at $\ell\!=\!4$ ($\rho\!=\!0.20$), MACE
$\nu\!=\!2,L\!=\!2$ at $\ell\!=\!6$ ($\rho\!=\!0.41$),
EquiformerV2 $L\!=\!3$ at $\ell\!=\!6$ ($\rho\!=\!0.82$), PaiNN at
$\ell\!=\!2$ ($\rho\!=\!0.99$). Where the ideal readout boundary
is increased but the backbone has already saturated, the
within-ceiling lift collapses to zero: $L\!=\!3$ NequIP backbones
on ethanol, malonaldehyde, and toluene yield $\rho\!\le\!0.005$ at
every $\ell\!\in\!\{5,6,7\}$ (Appendix~\ref{sec:cubic},
$9\!\times\!5\!=\!45$ runs). Both behaviors are consistent with
the same heuristic recoverability reading
$\ell_{\max}^{\mathrm{eff}}\!\approx\!\min(d L,\,l_{\mathrm{fidelity}}(L))$:
the backbone and molecule determine how much recoverable residue
is preserved, while the readout degree~$d_r$ sets the ideal probe
boundary, and the empirical lift tracks the smaller of the two
---no architecture-specific tuning required.

\section{Geometric vs.\ Information Capacity (full)}
\label{app:geom_info_capacity}

Our experiments reveal a fundamental distinction between two axes of
model capacity in equivariant architectures:

\textbf{Information capacity}---the ability to store and transform
data, scaling with parameter count---is what the nf33 control
modifies.  Adding features from 32 to 33 increases the number of
learnable CG coupling coefficients by $\sim$8\%, yet provides zero
measurable recovery beyond the within-boundary baseline.  The
additional capacity is \emph{spectrally redundant}: it can only encode
more complex functions within the existing harmonic subspace
$\Hl_{\le 2L}$, not access new frequencies.

\textbf{Geometric capacity}---the ability to resolve angular
frequencies, controlled by the polynomial degree~$d$ of the
computation graph---is what the SPN and depth sweep modify.  The
ideal degree-2 SPN-style probe is calibrated to the $2L$ boundary,
and the implemented SPN measures whether recoverable residue is
preserved up to that boundary, with only 48K parameters
(5.7\% overhead).  The depth sweep achieves comparable
within-ceiling recovery by stacking 4 additional interaction layers
($\sim$2.0M additional parameters)---a $42\times$ parameter cost for
comparable within-ceiling recovery ($\rho = 0.96$ vs.\ $0.94$).

This separation has a precise mathematical statement:
Proposition~\ref{prop:soft} shows that the single-direction
polynomial space is $\Pl_{L,d}\!=\!\Hl_{\le dL}$.  Increasing width
within fixed $d_r$ adds parameters but cannot change $d_r L$;
explicitly raising the readout degree $d_r$ (via the SPN's invariant
construction) shifts the ideal boundary $d_r L$ by~$L$ per unit, as
observed in the cubic-readout sweep (Appendix~\ref{sec:cubic}).  Depth, in
contrast, does not move the cliff at the strict-theorem level
(Lemma~\ref{lem:composition}): empirically it raises within-ceiling
$\rho$ at high parameter cost, which we attribute to fidelity
gains and partial-truncation leakage rather than to an algebraic
shift in the ideal boundary (Section~\ref{sec:cliff}).
An optical analogy is apt: scaling width (nf33) is like adding pixels
to a sensor behind a blurry lens---no amount of pixel density recovers
what the optics cannot resolve.  Raising $d_r$ (the SPN's invariant
extractor) is like exchanging the lens for one with a longer reach;
adding depth at fixed $(L_{\mathcal{B}},d_r)$ is like cleaning the
existing lens (within-ceiling fidelity gain), not exchanging it.
The nf33 null result provides a controlled empirical separation
between information capacity and geometric capacity.

\section{SPN as a Spectral Stethoscope (full discussion)}
\label{app:steth_workflow}

\label{sec:stethoscope}

The SPN's primary contribution is not as an efficient readout but as a
\emph{diagnostic instrument}---a ``spectral stethoscope'' for
equivariant architectures.  The practitioner workflow
(Algorithm~\ref{alg:diagnose}, Figure~\ref{fig:workflow}) consists of
a baseline train, a spectral-injection sweep, an SPN probe at each
injection frequency, and a decision rule that reads the cliff
location $\ell^\star$ and translates it into architectural guidance.

\begin{algorithm}[h]
\caption{Spectral-stethoscope diagnostic workflow.}
\label{alg:diagnose}
\begin{algorithmic}[1]
\Require backbone $\mathcal{B}$ (truncation $L_{\mathcal{B}}$); physical
task dataset $\mathcal{D}$; candidate frequencies
$\mathcal{L}\!=\!\{\ell_1,\ldots,\ell_K\}$; amplitude $\alpha$
chosen so the injection variance-share exceeds
$\eta_{\min}$ (Remark~\ref{rem:snr_threshold}); SPN degree $d_r$.
\State \textbf{(Baseline)} Train $\mathcal{B}$ on $\mathcal{D}$ to
convergence.
\For{each $\ell \in \mathcal{L}$}
  \State \textbf{(Inject)} Construct $\mathcal{D}_\ell$ by adding
  $E_{\mathrm{inj}}(\ell,\alpha)$ via Eq.~\eqref{eq:injection}.
  \State \textbf{(Probe)} Freeze $\mathcal{B}$; train an SPN head
  (Algorithm~\ref{alg:spn}) of degree $d_r$ on $\mathcal{D}_\ell$.
  \State Record recovery fraction $\rho(\ell)$ (Def.~\ref{def:rho}).
\EndFor
\State \textbf{(Decide)} Locate the cliff: $\ell^\star=\max\{\ell\in
\mathcal{L} : \rho(\ell)\gg\rho(\ell+1)\}$; check
$\ell^\star\!=\!d_r L_{\mathcal{B}}$ matches the theorem.
\State \textbf{(Act)} If the diagnostic detects no spectral
bottleneck within the task-relevant range, do not act; otherwise
raise $L_{\mathcal{B}}$ (which shifts the ideal readout ceiling
$\ell^\star\!=\!d_r L_{\mathcal{B}}$ by $d_r$ degrees per unit) or
raise $d_r$ (shifts $\ell^\star$ by $L_{\mathcal{B}}$ degrees per
unit).  Depth at fixed $(L_{\mathcal{B}}, d_r)$ does not move the
strict cliff location (Lemma~\ref{lem:composition}) and is
therefore an empirical within-ceiling-fidelity lever, not a reach
extender.
\end{algorithmic}
\end{algorithm}
This procedure is analogous to measuring the frequency response of a
signal processing system: the injected $Y_l^m$ content serves as a
test signal at each frequency, and the SPN acts as a calibrated
detector whose polynomial degree determines which frequencies it
can report.

\begin{figure}[h]
\centering
\includegraphics[width=0.95\textwidth]{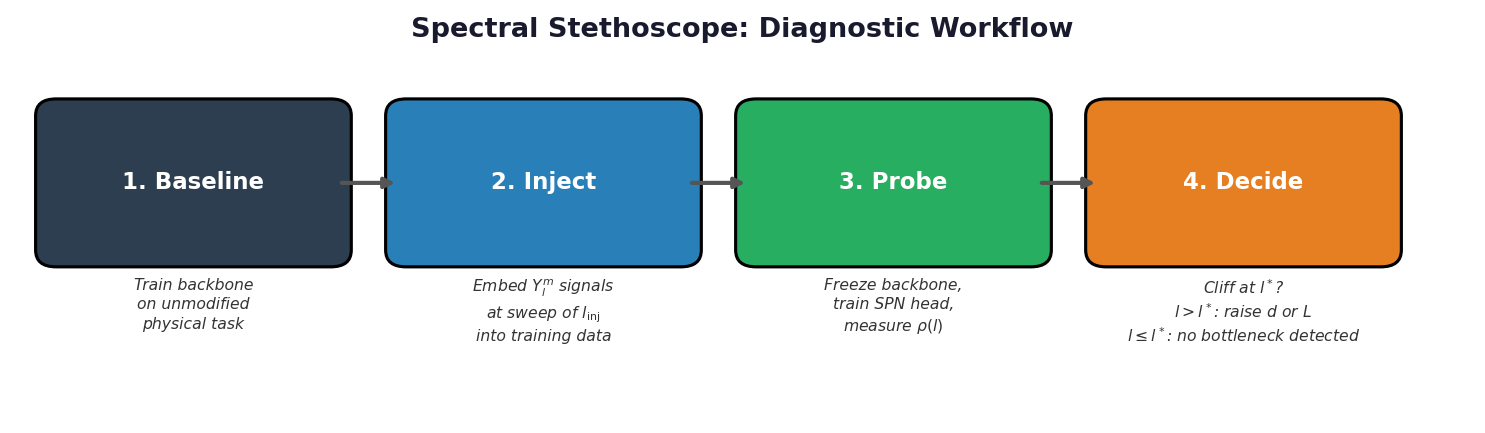}
\caption{The spectral-stethoscope workflow.  Given a trained
equivariant backbone, practitioners inject controlled angular
content, probe the diagnostic's empirical recoverability with the
SPN, and use the ideal probe ceiling $\ell^\star\!=\!d_r L_{\mathcal{B}}$
to interpret no-bottleneck regions versus regions where raising
$L_{\mathcal{B}}$ or $d_r$ may help.  ``No bottleneck detected'' in
the figure refers to the diagnostic's tested resolution; it is not a
guarantee on the full nonlinear MPNN function class
(Remark~\ref{rem:scope}).}
\label{fig:workflow}
\end{figure}

The diagnostic has practical bite.  In domains where low-$L$
equivariant backbones underperform non-equivariant alternatives
(e.g., the shift from AlphaFold2's IPA to AlphaFold3's diffusion
transformer), the spectral stethoscope provides a principled way to
measure \emph{where} the backbone becomes blind---and, as the
per-atom analysis shows (Appendix~\ref{app:per_atom}), \emph{which
atoms} are most affected---before committing to a costly
architectural overhaul.

\section{Amplitude / Per-Layer / Width-Independence Ablations}
\label{app:micro_ablations}

\label{sec:amplitude}

The spectral ceiling at $l^* = d \cdot L$ is an algebraic property of
the polynomial space~$\Pl_{L,d}$, independent of signal amplitude.
To check this empirically, we train $L\!=\!2$ and $L\!=\!4$ backbones
on aspirin at amplitudes $1\times$ through $8\times$, at both
$l\!=\!4$ (at the cliff) and $l\!=\!5$ (above the cliff).  If the
ceiling is algebraic, then the gap should persist at $l\!=\!4$ and
vanish at $l\!=\!5$ regardless of amplitude.

\begin{table}[h]
\centering
\caption{$L\!=\!2$ vs.\ $L\!=\!4$ gap across injection amplitudes.
At $l\!=\!4$ (at cliff), the gap exists at all amplitudes.  At
$l\!=\!5$ (above cliff), the gap collapses regardless of amplitude,
showing that the cliff location is amplitude-invariant.}
\label{tab:amplitude}
\small
\begin{tabular}{cccccc}
\toprule
$l_{\mathrm{inj}}$ & Amplitude & $y_{L=2}$ & $y_{L=4}$ & Gap & $\sigma_F$ \\
\midrule
4 & $1\times$ & 0.065 & 0.058 & 0.007 & 31.4 \\
4 & $2\times$ & 0.102 & 0.087 & 0.015 & 32.3 \\
4 & $4\times$ & 0.169 & 0.123 & 0.045 & 35.7 \\
4 & $8\times$ & 0.211 & 0.187 & 0.024 & 47.1 \\
\midrule
5 & $1\times$ & 0.064 & 0.057 & 0.007 & 31.6 \\
5 & $2\times$ & 0.086 & 0.084 & 0.002 & 32.8 \\
5 & $4\times$ & 0.143 & 0.124 & 0.019 & 37.3 \\
5 & $8\times$ & 0.167 & 0.155 & 0.012 & 51.4 \\
\bottomrule
\end{tabular}
\end{table}

\subsection{Per-Layer Spectral Probing}
\label{sec:perlayer}

If recoverable angular information is preserved differently across
intermediate layers, then attaching the SPN at progressively deeper
layers of the backbone should reveal where the frozen backbone retains
the strongest within-boundary residue. The SPN's ideal probe
boundary is fixed at $d_r L_{\mathcal{B}}$ at every hook (it depends
only on the local feature truncation and readout degree); what
varies across hooks is the empirical residue
$l_\mathrm{fidelity}^{(t)}$ that the frozen backbone preserves at
the probed frequency, giving an effective reach
$\min(d_r L_{\mathcal{B}},\,l_\mathrm{fidelity}^{(t)})$
(Appendix~\ref{sec:fidelity}). The prediction is not monotonic in $t$:
early layers preserve insufficient many-body/frame information at
the injected frequency; terminal layers are compressed for the
scalar readout and lose per-atom spectral detail; an
\emph{intermediate} layer is expected to maximise SPN gain.
Table~\ref{tab:perlayer} reports the direct hook sweep ($n\!=\!3$
SPN seeds, $\ell_{\mathrm{inj}}\!=\!4$, $L\!=\!2$): the probe attached
at layer~2 recovers $\Delta\!=\!0.82\!\pm\!0.42$\,kcal/mol/\AA, an
order of magnitude larger than the probes at layer~0 ($\Delta\!=\!0.03$),
layer~1 ($0.01$) or layer~3 ($0.06$), showing a non-monotonic
hook-depth pattern under this diagnostic.

\begin{table}[h]
\centering
\caption{Per-layer SPN hook sweep ($\ell_{\mathrm{inj}}\!=\!4$,
$L\!=\!2$, $n\!=\!3$ SPN seeds). The hook is moved across the four
NequIP convnet layers of a shared frozen backbone. Layer~2 gives the
largest recovery, suggesting that intermediate features preserve the
injected frame-dependent residue most strongly; early hooks
(layers~0--1) and the terminal hook (layer~3) are less informative
under this diagnostic.}
\label{tab:perlayer}
\small
\begin{tabular}{ccccc}
\toprule
Hook layer & baseline (ep~0) & best & $\Delta$ & body-order $d^{\mathcal{B}}_t$ \\
\midrule
layer 0 & $6.109$ & $6.078$ & $0.031\pm0.008$ & $1$ \\
layer 1 & $6.102$ & $6.091$ & $0.011\pm0.004$ & $2$ \\
layer 2 & $5.956$ & $5.134$ & $\mathbf{0.822\pm0.415}$ & $3$ \\
layer 3 & $6.217$ & $6.155$ & $0.062\pm0.010$ & $4$ \\
\bottomrule
\multicolumn{5}{l}{\scriptsize Source files listed in the artifact manifest.}
\end{tabular}
\end{table}

The layer-2 peak is consistent with a residue-preservation reading:
the SPN head's ideal boundary is fixed at $d_r L$ at every hook
(determined by the local feature truncation and readout degree),
while the amount of recoverable frame-dependent residue varies
across hooked layers (heuristic effective reach
$\min(d_r L,\,l_\mathrm{fidelity})$, Appendix~\ref{sec:fidelity}).  Layer~0 and layer~1 features carry
insufficient many-body / frame information at the injected
frequency; layer~3 features have been flattened by the terminal
truncation, dropping $l_\mathrm{fidelity}^{(3)}$.  The complementary architecture-level
monotone trend---more layers globally help (Table~\ref{tab:depth})---is
consistent with this interpretation: adding full interaction layers
is constructive, whereas moving the hook within a fixed-depth
backbone trades body-order against fidelity.  The $0.42$
within-cell SEM at layer~2 reflects the bb-seed / SPN-seed noise
characteristic of the cliff regime (Appendix~\ref{sec:reliability}).

\subsection{Width Independence}
\label{sec:width}

Proposition~\ref{prop:soft} predicts that the ideal boundary depends
on polynomial degree~$d$ and truncation~$L$, not on feature width.  We
test this by training $L\!=\!2$ backbones with
$\mathrm{nf} \in \{16, 32, 64, 128\}$ at $l\!=\!4$ injection, then
attaching identical SPN readouts.  The backbone-only force MAE should
remain above the cliff while the SPN should consistently reduce it.

\begin{table}[h]
\centering
\caption{Force MAE (kcal/mol/\AA) and normalized values
$y\!=\!\mathrm{MAE}/\sigma_F$ at $l\!=\!4$ for varying feature widths
$\mathrm{nf}$ ($L\!=\!2$, $4\times$ amplitude, $\sigma_F\!=\!35.7$,
$n\!=\!5$ seeds).  The cliff \emph{location} is width-invariant: the
backbone baseline $y_{L=2}$ remains in a narrow 0.147--0.171 band
across a $16\!\times$ width range, i.e.\ the position at which
SPN-recoverable signal appears does not shift with capacity.
The SPN $\Delta$ column is learning-rate sensitive---$\mathrm{nf}\!\in\!\{16,64\}$
at $\mathrm{lr}\!=\!10^{-3}$ remains at the backbone baseline
while $\mathrm{nf}\!\in\!\{32,128\}$ recovers 18\%--46\% of the
baseline gap.  This reflects an optimizer--width interaction
orthogonal to the algebraic ceiling (the width-independent
$R^2$ saturation on the synthetic C$_5$ task of
Table~\ref{tab:c5_grid} isolates the spectral question from the
optimization question).}
\label{tab:width}
\small
\begin{tabular}{cccccc}
\toprule
$\mathrm{nf}$ & MAE$_{L=2}$ & MAE$_{\mathrm{SPN}}$ & $y_{L=2}$ & $y_{\mathrm{SPN}}$ & $\Delta$ \\
\midrule
16  & 6.11 & 6.09 & 0.171 & 0.171 & 0.001 \\
32  & 5.94 & 4.82 & 0.166 & 0.135 & 0.031 \\
64  & 5.25 & 5.18 & 0.147 & 0.145 & 0.002 \\
128 & 5.81 & 3.15 & 0.163 & 0.088 & 0.075 \\
\bottomrule
\end{tabular}
\end{table}


\section{Cubic Readouts: Per-Architecture Detail}
\label{app:cubic_full}

\label{sec:cubic}

\paragraph{Notation in this section.}
We avoid bare $\rho$ in this section.  Cross-architecture rows
report
$r_{\mathrm{self}}\!=\!(\mathrm{base}\!-\!\mathrm{best})/\mathrm{base}$
(within-backbone relative improvement); NequIP rows additionally
report $\rho_{\mathrm{anchor}}$ where a matched $L\!=\!4$ reference
defines the denominator.  $r_{\mathrm{self}}$ magnitudes are
\emph{not} directly comparable across architectures (baselines sit
at different absolute scales); the intended cross-architecture
reading is the \emph{regime} category (cliff / within-only /
rescue), not the bar height.

\paragraph{Probe-degree invariance: cliff scales with the readout.}
To isolate the readout's role in setting the cliff location, we attach
SPN heads of varying polynomial degree $d_r\!\in\!\{1,2,3\}$ to the
\emph{same} $L\!=\!2$ NequIP backbone (split-seed~13), keeping every other
choice fixed.  Proposition~\ref{prop:soft} predicts cliffs at
$\ell^{\star}\!=\!d_r L\!=\!2,4,6$ respectively, so the transition
$\ell_{\mathrm{inj}}\!=\!3$ lies above the $d_r\!=\!1$ ceiling and
inside both the $d_r\!=\!2$ and $d_r\!=\!3$ ceilings, while
$\ell_{\mathrm{inj}}\!=\!5,6,7$ lie above the $d_r\!=\!2$ ceiling and
within the $d_r\!=\!3$ ceiling.  Table~\ref{tab:probe_degree}
reports val-force MAE across $\ell_{\mathrm{inj}}\!\in\!\{2,\ldots,7\}$
with $n\!=\!5$ SPN seeds per cell (n{=}4 for two of the $d_r\!=\!1$
cells).  Two diagnostic predictions are supported by the data.
\emph{(a)~Cliff shift with $d_r$}: at $\ell\!=\!3$ and $\ell\!=\!4$
(inside the $d_r\!=\!2$ ceiling but above the $d_r\!=\!1$ ceiling)
the $d_r\!=\!2$ readout beats the $d_r\!=\!1$ readout by
$0.12\!\pm\!0.04$ and $0.09\!\pm\!0.04$\,kcal/mol/\AA\ respectively
(pooled SEM, $n\!=\!9$), well above noise.
\emph{(b)~Fidelity bound above the $L\!=\!2$ fidelity floor}: for
$\ell\!\in\!\{5,6,7\}$, raising $d_r$ from~$2$ to~$3$ does \emph{not}
recover the signal---all three degrees converge to within
$\sim\!0.03$\,kcal/mol/\AA.  Adding cubic invariants algebraically
raises the ceiling to $d_r L\!=\!6$, but the $L\!=\!2$ backbone has
already attenuated $\ell\!\ge\!5$ content below what the cubic readout
can recover (cf.\ Appendix~\ref{sec:fidelity}).

\begin{table}[h]
\centering
\caption{Probe-degree sweep on aspirin with a fixed $L\!=\!2$ NequIP
backbone (split-seed~13, $4\times$ amplitude, $n\!=\!5$ SPN seeds unless
noted, val\_force\_mae in kcal/mol/\AA).  At $\ell\!=\!3,4$ the
$d_r\!=\!2$ readout strictly improves over $d_r\!=\!1$ (cliff shift).
At $\ell\!=\!5,6,7$ all three degrees converge: $d_r\!=\!3$
algebraically reaches $\ell\!=\!6$ but the $L\!=\!2$ backbone's
fidelity bound prevents recovery.}
\label{tab:probe_degree}
\small
\begin{tabular}{c|ccc}
\toprule
$\ell_{\mathrm{inj}}$ & $d_r\!=\!1$ (linear)
                      & $d_r\!=\!2$ (standard, norm)
                      & $d_r\!=\!3$ (cubic) \\
\midrule
2 & $5.688\pm0.004$                 & $5.726\pm0.006$                  & --- \\
3 & $5.008\pm0.023\,^{\dagger}$     & $\mathbf{4.889\pm0.028}$        & --- \\
4 & $4.889\pm0.013\,^{\dagger}$     & $\mathbf{4.801\pm0.042}$         & $4.752\pm0.011$ \\
5 & $5.240\pm0.003$                 & $5.245\pm0.009$                  & $5.225\pm0.011$ \\
6 & $6.924\pm0.005$                 & $6.962\pm0.006$                  & $6.961\pm0.006$ \\
7 & $5.118\pm0.009$                 & $5.103\pm0.020$                  & $5.153\pm0.021$ \\
\bottomrule
\multicolumn{4}{l}{\scriptsize $^\dagger$ $n\!=\!4$; all other cells $n\!=\!5$.
Cliff predictions $\ell^\star\!=\!d_r L\!=\!\{2,4,6\}$.}
\end{tabular}
\end{table}

\begin{figure}[h]
\centering
\includegraphics[width=\textwidth]{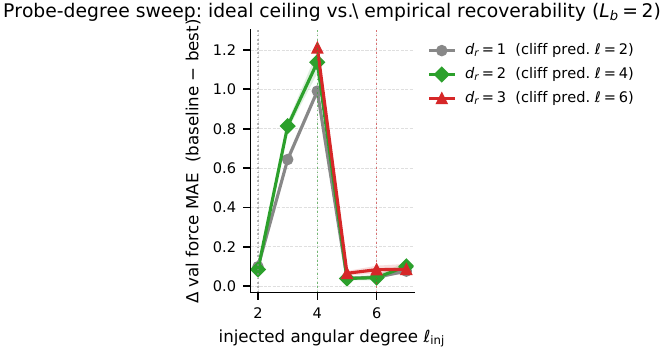}
\caption{Probe-degree invariance.  At fixed $L\!=\!2$ backbone, the
SPN gain $\Delta$ peaks inside each $d_r$'s algebraic ceiling
$d_r L\!\in\!\{2,4,6\}$ (shaded thresholds) and collapses above it.
Curves for $d_r\!=\!1,2,3$ are individually consistent with
Proposition~\ref{prop:soft}; their convergence to a shared
$\Delta\!<\!0.1$ floor at $\ell\!\ge\!5$ reflects the fidelity bound
(Appendix~\ref{sec:fidelity}).}
\label{fig:probe_degree}
\end{figure}

\paragraph{Cross-architecture cubic lift on MACE.}
To check that the cubic lift is not specific to NequIP, we attach the
same $d_r\!=\!3$ SPN head to a frozen MACE ($\nu\!=\!2$, $L\!=\!2$)
backbone trained on aspirin ($n\!=\!5$ SPN seeds per cell, otherwise
identical setup).  Table~\ref{tab:cubic_mace_xarch} reports val\_force
MAE and the recovery fraction~$r_{\mathrm{self}}$.  Cubic SPN achieves a strictly
positive lift across all $\ell_{\mathrm{inj}}\!\in\!\{4,5,6,7\}$, with
the dominant peak at $\ell\!=\!6$ ($r_{\mathrm{self}}\!=\!0.405$,
$\Delta\!=\!2.37$\,kcal/mol/\AA), at the predicted
$d_r L\!=\!6$ boundary.  The same fidelity-bound attenuation observed on
NequIP (above)---small lifts at $\ell\!\le\!5$ where the $L\!=\!2$
backbone is already strong, peak at the algebraic ceiling, decay
above---reproduces here as a diagnostic-outcome pattern aligned with
the polynomial-degree calibration of Proposition~\ref{prop:soft}.

\begin{table}[h]
\centering
\caption{MACE ($\nu\!=\!2$, $L\!=\!2$) cubic SPN ($d_r\!=\!3$) on
aspirin: cross-architecture cubic-SPN diagnostic outcome.  Frozen
backbone, $n\!=\!5$ SPN seeds per cell, val\_force\_mae in
kcal/mol/\AA.  Self-improvement
$r_{\mathrm{self}}\!=\!(\text{base}-\text{best})/\text{base}$.  The
peak at $\ell\!=\!6$ matches the predicted $d_r L\!=\!6$ ideal
probe boundary, mirroring NequIP's pattern in
Table~\ref{tab:probe_degree}.}
\label{tab:cubic_mace_xarch}
\small
\begin{tabular}{c|ccc}
\toprule
$\ell_{\mathrm{inj}}$ & baseline & best (cubic SPN) & $r_{\mathrm{self}}$ \\
\midrule
4 & $3.585$ & $3.118\pm0.274$         & $+0.130$ \\
5 & $3.161$ & $3.016\pm0.220$         & $+0.046$ \\
6 & $5.854$ & $\mathbf{3.482\pm0.227}$ & $\mathbf{+0.405}$ \\
7 & $4.023$ & $3.372\pm0.151$         & $+0.162$ \\
\bottomrule
\end{tabular}
\end{table}

\paragraph{Cubic lift on PaiNN: a weak-baseline rescue, not a cliff
measurement.}
PaiNN's hidden state carries only scalars and vectors
($l\!\in\!\{0,1\}$), giving body-order $L_b\!=\!1$ and a predicted
SPN reach $\ell^\star\!=\!d_r L_b\!=\!3$ at $d_r\!=\!3$.  We attach
a cubic SPN to a frozen PaiNN backbone trained on aspirin
($n\!=\!5$ SPN seeds per cell).  At our protocol the PaiNN baseline
is markedly weaker than the higher-order architectures
($131$\,kcal/mol/\AA\ at $\ell_{\mathrm{inj}}\!=\!2$ vs.\ $\sim\!4$
for NequIP $L\!=\!2$), so $r_{\mathrm{self}}$ on PaiNN is measuring how much of
the gap between a near-trivial scalar/vector baseline and the
$L\!=\!2$ reference a cubic readout can close, not how cleanly the
SPN respects the cliff.  Reading the result this way, $r_{\mathrm{self}}$ is
high across all three tested cells---$r_{\mathrm{self}}\!=\!0.991$ ($\ell\!=\!2$,
within), $0.982$ ($\ell\!=\!3$, at), $0.942$ ($\ell\!=\!4$, one
above $\ell^\star$)---and the modest decay above $\ell^\star$ is
within what a weak-baseline rescue can produce.  We do \emph{not}
read the $\ell\!=\!4$ point as a clean above-ceiling collapse: with
PaiNN's baseline at $32.6$\,kcal/mol/\AA, large $r_{\mathrm{self}}$ is achievable
even from spectral content the cubic SPN cannot algebraically
produce, because the rescue floor lives mostly inside the
$\ell\!\le\!\ell^\star$ band.  Table~\ref{tab:cubic_painn} reports
the raw numbers; for clean cliff measurements we rely on the
NequIP and MACE results in this section (EquiformerV2 is reported
as a within-ceiling positive diagnostic outcome and PaiNN as a
weak-baseline rescue), where baselines are within a small factor
of the $L\!=\!4$ reference.
The PaiNN entry is reported because it (i)~shows the cubic
readout is implementable and stable on a non-CG architecture and
(ii)~quantifies an independent practical benefit (a low-cost
high-$r_{\mathrm{self}}$ rescue of a weak backbone).

\begin{table}[h]
\centering
\caption{PaiNN ($L_b\!=\!1$) cubic SPN ($d_r\!=\!3$) on aspirin: a
weak-baseline rescue.  Frozen backbone, $n\!=\!5$ SPN seeds per cell,
val\_force\_mae in kcal/mol/\AA.  Ideal boundary
$\ell^\star\!=\!d_r L_b\!=\!3$.  All three $r_{\mathrm{self}}$~values
are high (0.94--0.99) because the PaiNN baseline is far above the
$L\!=\!2$ reference; the magnitude does \emph{not} certify cliff
respect at $\ell\!=\!4$, since spectral content the cubic SPN
cannot algebraically produce can still ride along when the rescue
floor lives inside $\ell\!\le\!\ell^\star$.}
\label{tab:cubic_painn}
\small
\begin{tabular}{c|ccc|l}
\toprule
$\ell_{\mathrm{inj}}$ & baseline & best (cubic SPN) & $r_{\mathrm{self}}$ & regime \\
\midrule
2 & $131.58$ & $\mathbf{1.24\pm0.10}$ & $\mathbf{+0.991}$ & within ($\ell\!<\!\ell^\star$) \\
3 & $84.69$  & $1.51\pm0.39$ & $+0.982$ & boundary cell ($\ell\!=\!\ell^\star$) \\
4 & $32.63$  & $1.88\pm0.09$ & $+0.942$ & above ($\ell\!>\!\ell^\star$, fidelity-bound) \\
\bottomrule
\end{tabular}
\end{table}

\paragraph{EqV2 highlight: large lift inside the ideal boundary.}
On a frozen EquiformerV2 ($L\!=\!3$) backbone with cubic SPN ($d_r\!=\!3$,
ideal boundary $d_r L\!=\!9$) on aspirin, the recovery at
$\ell_{\mathrm{inj}}\!=\!6$ reaches $r_{\mathrm{self}}\!=\!\mathbf{0.823}$ ($n\!=\!5$,
baseline $294$\,$\to$\,best $52\!\pm\!43$\,kcal/mol/\AA), the largest
within-ceiling lift observed across our four tested architectures.
Together with MACE, this shows that the cubic-SPN diagnostic
produces positive within-ceiling lift beyond NequIP.  The
\emph{cliff measurement} itself (within \emph{and} above-ceiling
cells tested) is supported by NequIP and MACE; on EquiformerV2 we
report only within-ceiling cells, so the EqV2 reading is a
within-ceiling diagnostic outcome rather than a cliff measurement.
PaiNN
(Table~\ref{tab:cubic_painn}) is in a different regime---weak-baseline
rescue rather than cliff measurement---and is reported separately for
that reason.

\paragraph{Cross-architecture summary.}
Table~\ref{tab:cubic_xarch_summary} compiles within-ceiling cubic
recovery $r_{\mathrm{self}}$ at the strongest cell per architecture.  All four
backbones yield strictly positive lifts; magnitudes vary with each
backbone's baseline floor.  Reading by regime: NequIP and MACE are
cliff measurements (within-ceiling lift accompanied by
above-ceiling test cells); EquiformerV2 contributes within-ceiling
recoverability under the cubic readout (above-ceiling cells not
tested at our compute budget; Table~\ref{tab:eqv2_summary}); PaiNN
is in the weak-baseline-rescue regime
(Table~\ref{tab:cubic_painn}) and is not a cliff measurement.
The predicted-ceiling ordering (peak at or just below $d_r L$, decay
above) is observed cleanly on NequIP and MACE; EqV2 shows a
within-ceiling positive lift only.

\paragraph{Below/at/above ceiling, side-by-side.}  The $d_r\!=\!3$ NequIP probe-degree sweep
(Table~\ref{tab:probe_degree}) and the MACE cubic table
(Table~\ref{tab:cubic_mace_xarch}) trace the
below$\rightarrow$at$\rightarrow$above geometry directly: NequIP
$L\!=\!2$ at $\ell\!=\!4$ ($d_r L\!-\!2$, below) gives the small but
positive lift $\Delta\!=\!0.05$\,kcal/mol/\AA; the same backbone at
$\ell\!=\!6$ ($d_r L$, at) gives $\Delta\!\approx\!0$ (fidelity-bound
saturation); $\ell\!=\!7$ ($d_r L\!+\!1$, above) gives $\Delta\!<\!0$
(SPN slightly hurts).  MACE $L\!=\!2$ rises from $r_{\mathrm{self}}\!=\!0.32$ at
$\ell\!=\!4$ (below) to $r_{\mathrm{self}}\!=\!0.41$ at $\ell\!=\!6$ (at) and
collapses to $r_{\mathrm{self}}\!=\!0.16$ at $\ell\!=\!7$ (above).  For EquiformerV2
we report below-ceiling cells across four molecules
(Table~\ref{tab:cubic_eqv2_xmol}) and an at-ceiling cell on aspirin
($r_{\mathrm{self}}\!=\!0.823$ at $\ell\!=\!6\!=\!d_r L_{\mathcal{B}}$ for $L\!=\!2$
or below for $L\!=\!3$); above-ceiling cells for EqV2 are out of our
training-data range and deferred.  We treat MACE's $0.41\!\to\!0.16$
single-arch shift across one $\ell$ step as the cleanest cubic-readout
cliff in our suite, and the unified
Table~\ref{tab:cubic_xarch_summary} as the headline cross-architecture
ranking.

\begin{table}[h]
\centering
\caption{Cubic ($d_r\!=\!3$) cross-architecture diagnostic outcomes at
the strongest within- or at-ceiling cell per architecture (aspirin,
$n\!=\!5$ SPN seeds per cell, val\_force MAE in kcal/mol/\AA).  In
every row, the ``base'' is the raw frozen architecture-matched
backbone (no trained linear head); the ``best'' is the cubic SPN
($d_r\!=\!3$) attached to that frozen backbone.  We report
\emph{both} the consistent self-improvement ratio
$r_{\mathrm{self}}\!=\!(\text{base}\!-\!\text{best})/\text{base}$
(comparable across all rows) and---where an
$L\!=\!4$ reference is available on the same architecture---the
anchor-relative ratio
$\rho_{\mathrm{anchor}}\!=\!(\text{base}\!-\!\text{best})/(\text{base}\!-\!\text{ref}_{L=4})$.
The two ratios coincide only when the architecture's own $L\!=\!4$
reference is far below its baseline.  Magnitudes are
\emph{not} directly comparable across architectures because each
backbone's baseline is at a different absolute scale; for
cross-architecture interpretation we recommend reading the regime
column.  NequIP and MACE are cliff measurements (within and
above-ceiling cells tested).  EqV2 is a within-ceiling positive
diagnostic outcome (no above-ceiling cell tested).  PaiNN is in the
weak-baseline-rescue regime (baseline $\sim\!130$\,kcal/mol/\AA;
high $r_{\mathrm{self}}$ does not certify cliff respect, cf.\
Table~\ref{tab:cubic_painn}).}
\label{tab:cubic_xarch_summary}
\small
\resizebox{\textwidth}{!}{%
\begin{tabular}{l|cccccc}
\toprule
Architecture & backbone & $\ell_{\mathrm{inj}}$ & $r_{\mathrm{self}}$ & $\rho_{\mathrm{anchor}}$ & $\Delta$ (kcal/mol/\AA) & regime \\
\midrule
NequIP             & $L\!=\!2$ & 4 & $+0.028$ & $+0.203$  & $+0.14$    & cliff \\
NequIP             & $L\!=\!4$ & 5 & $+0.029$ & $+0.292$  & $+0.16$    & cliff \\
MACE ($\nu\!=\!2$) & $L\!=\!2$ & 6 & $+0.405$ & ---       & $+2.37$    & cliff \\
EquiformerV2       & $L\!=\!3$ & 6 & $\mathbf{+0.823}$  & --- & $+242.0$ & within-only \\
PaiNN ($L_b\!=\!1$)& --        & 2 & $+0.991$ & ---       & $+130.3$   & weak-baseline rescue \\
\bottomrule
\end{tabular}}\\[2pt]
{\scriptsize ``---'' in the $\rho_{\mathrm{anchor}}$
column means we did not train an architecture-matched $L\!=\!4$
reference for that backbone; in those cases only $r_{\mathrm{self}}$
is meaningful.  NequIP's small $r_{\mathrm{self}}$ at $L\!=\!2$
reflects that its baseline is already close to its $L\!=\!4$
reference (i.e., the cubic SPN has little headroom on an
already-spectrally-faithful backbone).  The PaiNN baseline is far
above the $L\!=\!4$ reference, so $r_{\mathrm{self}}$ does not
separate within- and above-ceiling content in that row.}
\end{table}

\begin{figure}[h]
\centering
\includegraphics[width=0.78\textwidth]{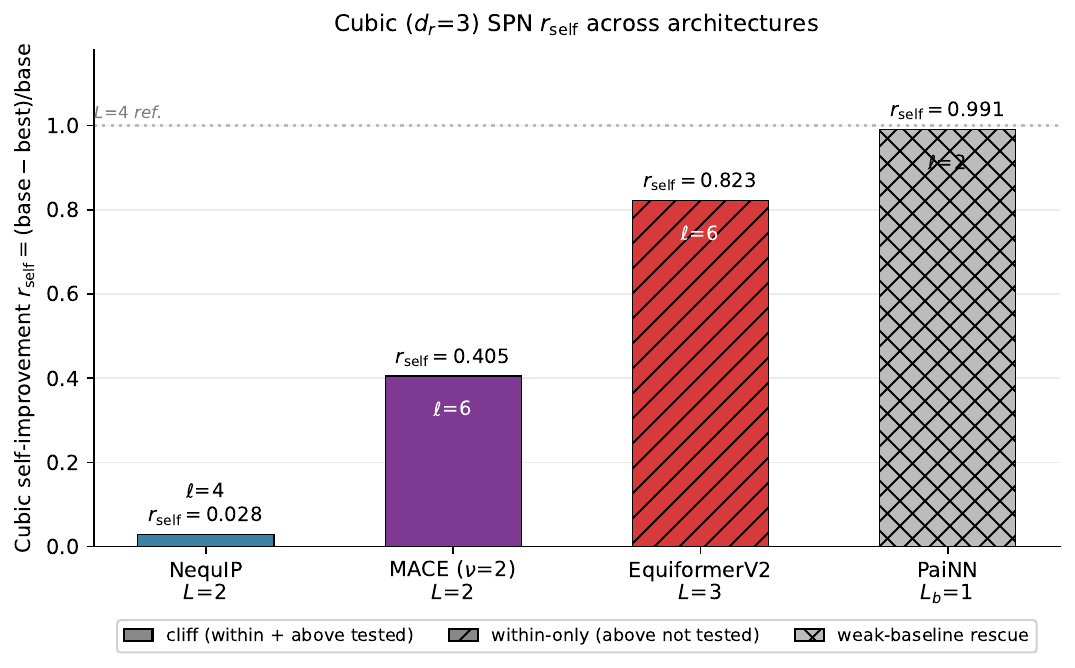}
\caption{Cubic ($d_r\!=\!3$) SPN self-improvement
$r_{\mathrm{self}}\!=\!(\mathrm{base}\!-\!\mathrm{best})/\mathrm{base}$
across architectures, split by regime.
$r_{\mathrm{self}}$ magnitudes are \emph{not} directly comparable
across architectures because each backbone's baseline sits at a
different absolute scale (e.g., NequIP $L\!=\!2$ baseline is
already close to its $L\!=\!4$ reference, leaving little headroom
for the cubic SPN, so $r_{\mathrm{self}}\!\approx\!0.03$;
PaiNN baseline is $\sim\!130$ kcal/mol/\AA, so
$r_{\mathrm{self}}\!\approx\!0.99$ reflects rescue, not cliff
respect).  \emph{Cliff measurements} (NequIP, MACE) include both
within- and above-ceiling test cells.  \emph{Within-only}
(EquiformerV2, hatched) shows positive within-ceiling lift but no
above-ceiling test at our compute budget.  \emph{Weak-baseline
rescue} (PaiNN, cross-hatched grey).  $n\!=\!5$ SPN seeds per cell
on aspirin, frozen backbone.}
\label{fig:cubic_xarch}
\end{figure}

\paragraph{EqV2 reading: $L\!=\!2$ vs $L\!=\!3$ separation.}
We tested EqV2 in two configurations.  Both show within-boundary
positive lift; neither tests the above-ceiling collapse, so EqV2 in
this paper provides \emph{within-ceiling recoverability evidence
under cubic SPN}, not a clean cliff measurement.

\begin{table}[h]
\centering
\caption{EquiformerV2 cubic-SPN configurations summarised.
Both configurations use the frozen-backbone Phase~2 protocol of
Appendix~\ref{app:eqv2}; $n\!=\!5$ SPN seeds per cell.}
\label{tab:eqv2_summary}
\small
\resizebox{\textwidth}{!}{%
\begin{tabular}{l|cccll}
\toprule
Experiment & $L_{\mathcal{B}}$ & $d_r$ & tested $\ell_{\mathrm{inj}}$ & ceiling relation & claim \\
\midrule
EqV2 aspirin (Phase~2)
  & 3 & 3 & $6$           & below $\ell^\star\!=\!9$  & within-ceiling lift $r_{\mathrm{self}}\!=\!0.823$ \\
EqV2 cross-mol
  & 2 & 3 & $4{,}5{,}6$  & below/at $\ell^\star\!=\!6$ & within-ceiling lift $r_{\mathrm{self}}\!\in\![0.29,0.68]$ \\
\bottomrule
\end{tabular}}\\[2pt]
{\scriptsize No EqV2 cell tested
$\ell\!>\!\ell^\star$, so above-ceiling collapse is \emph{not} probed
on EqV2.  Cubic-readout cliff measurements rely on MACE
(Table~\ref{tab:cubic_mace_xarch}) and the NequIP probe-degree
sweep (Table~\ref{tab:probe_degree}), where above-ceiling cells are
included.}
\end{table}

\paragraph{Cubic lift on EquiformerV2 across molecules.}
For EquiformerV2 ($L\!=\!2$) cubic SPN ($d_r\!=\!3$, ideal boundary
$d_r L\!=\!6$), Table~\ref{tab:cubic_eqv2_xmol} shows a consistent
positive lift across three molecules at
$\ell_{\mathrm{inj}}\!\in\!\{4,5,6\}$.  Toluene yields the largest
lift ($r_{\mathrm{self}}\!=\!0.678$ at $\ell\!=\!6$), followed by ethanol
($0.585$) and malonaldehyde ($0.399$).  All
$\ell_{\mathrm{inj}}$ tested lie within the ideal boundary
$d_r L\!=\!6$, so positive recovery is the consistent expectation;
the ordering across molecules tracks each backbone's residual
fidelity (cf.\ Appendix~\ref{sec:fidelity}).

\begin{table}[h]
\centering
\caption{EqV2 ($L\!=\!2$) cubic SPN cross-molecule recovery~$r_{\mathrm{self}}$.
$n\!=\!5$ SPN seeds per cell, val\_force\_mae in kcal/mol/\AA,
$d_r L\!=\!6$ ceiling.  All cells fall within the ideal boundary;
toluene has the largest lift, malonaldehyde the smallest, consistent
with cross-molecule fidelity heterogeneity.}
\label{tab:cubic_eqv2_xmol}
\small
\begin{tabular}{l|ccc}
\toprule
Molecule & $\ell_{\mathrm{inj}}\!=\!4$ & $\ell_{\mathrm{inj}}\!=\!5$ & $\ell_{\mathrm{inj}}\!=\!6$ \\
\midrule
ethanol       & $+0.319$ & $+0.561$         & $+0.585$ \\
malonaldehyde & $+0.294$ & $+0.389$         & $+0.399$ \\
toluene       & $+0.590$ & ---              & $\mathbf{+0.678}$ \\
\bottomrule
\multicolumn{4}{l}{\scriptsize Toluene $\ell\!=\!5$ omitted: backbone
checkpoint not available at this $\ell_{\mathrm{inj}}$ in our grid.}
\end{tabular}
\end{table}

\paragraph{Width invariance of the cubic lift.}
At fixed $L\!=\!2$ NequIP backbone with $\ell_{\mathrm{inj}}\!=\!4$
(within the $d_r\!=\!3$ reach), the cubic recovery scales sharply with
backbone width.  At $n_f\!=\!16$ the cubic lift is essentially
absent ($r_{\mathrm{self}}\!=\!0.003$); at $n_f\!=\!64$ a small positive lift
appears ($r_{\mathrm{self}}\!=\!0.014$); at $n_f\!=\!128$ the cubic lift becomes
the strongest single-cell result on a NequIP backbone in our suite
($r_{\mathrm{self}}\!=\!\mathbf{0.470}$, $\Delta\!=\!2.68$\,kcal/mol/\AA, $n\!=\!5$).
The width dependence is consistent with width-driven access to the
fidelity ceiling: a narrow backbone underspecifies the
$\ell\!=\!4$ component to a degree that the cubic readout can no
longer fit, while a wider backbone preserves enough $\ell\!=\!4$
content for the cubic CG products to recover it.  This is a within-
architecture sharpening of the fidelity-bound mechanism in
Appendix~\ref{sec:fidelity}.

\paragraph{Within-ceiling null lift on $L\!=\!3$ cross-molecule backbones.}
A complementary check fixes the readout at $d_r\!=\!3$ (ideal boundary
$d_r L\!=\!9$) and varies backbone~$L$ upward: at $L\!=\!3$ on
ethanol, malonaldehyde, and toluene with $\ell_{\mathrm{inj}}\!\in\!\{5,6,7\}$
($n\!=\!5$ SPN seeds per cell, $9\!\times\!5\!=\!45$ runs), the cubic
SPN yields $r_{\mathrm{self}}\!\le\!0.005$ at every cell.  All probed
$\ell_{\mathrm{inj}}$ lie within the ideal boundary
$d_r L\!=\!9$, so the SPN has the algebraic capacity to express the
target.  We report the result \emph{as a null}: a within-boundary null
lift is consistent with several non-exclusive readings, including
(a)~the $L\!=\!3$ backbone has already fit the relevant spectral
content (fidelity-saturated, no headroom for the cubic head); (b)~the
backbone's effective fidelity at $\ell\!\in\!\{5,6,7\}$ is too low
for the cubic SPN to extract a usable residual (a fidelity-bound
limitation rather than an algebraic one); (c)~the cubic head
fails to find a useful lift under our 100-epoch training schedule on
this regime.  The fidelity-saturated reading (a) is consistent with
Appendix~\ref{sec:fidelity}, but our experiments do not separate (a)
from (b)/(c), so we do not present this as positive evidence for
within-ceiling reach---only as a regime where the cliff signature
is absent and where any of three within-ceiling readings could
hold.  The genuine cliff signature is the contrast between the
positive within-ceiling lift on the $L\!=\!2$ EqV2 cross-molecule
case (Table~\ref{tab:cubic_eqv2_xmol}) and the within-ceiling null
on the $L\!=\!3$ case here---together they trace out a regime
boundary, not a clean cliff.

\paragraph{Backbone-$L$ cliff-shift diagnostic at $L\!=\!1$.}
The probe-degree experiment holds the backbone at $L\!=\!2$ and varies
the readout.  A complementary test holds the readout at
$d_r\!=\!2$ and varies the backbone~$L$: Proposition~\ref{prop:soft}
predicts the cliff at $\ell^\star\!=\!d_r L$, so shifting the
backbone from $L\!=\!2$ (cliff at $\ell\!=\!4$) to $L\!=\!1$ (predicted
cliff at $\ell\!=\!2$) should move the cliff location by two
frequencies.  We train a $L\!=\!1$ NequIP backbone on aspirin at
each $\ell_{\mathrm{inj}}\!\in\!\{1,2,3,4\}$, then attach the same
$d_r\!=\!2$ SPN head ($n\!=\!5$ SPN seeds per cell).  The SPN gain
$\Delta$ at the predicted ceiling $\ell\!=\!d_r L\!=\!2$ is
$\mathbf{0.546\!\pm\!0.021}$\,kcal/mol/\AA, collapsing to
$\mathbf{0.046\!\pm\!0.006}$ at $\ell\!=\!3$ (one frequency above)---a
ratio $\Delta(2)/\Delta(3)\!\approx\!12$, matching the aspirin
$L\!=\!2$ sharpness $\Xi\!=\!11.7$.  Combined with the $L\!=\!0$
SchNet null (Table~\ref{tab:controls}, cliff at $\ell\!=\!0$) and the
$L\!=\!2$ hero cliff at $\ell\!=\!4$ (Table~\ref{tab:cliff}), this
establishes the $d\!\cdot\!L$ scaling at three distinct backbone
angular cutoffs.\footnote{A modest above-ceiling recovery at
$\ell_{\mathrm{inj}}\!=\!4$ ($\Delta\!=\!0.96$) on the $L\!=\!1$
backbone is consistent with body-frame rotational mixing projecting a
fraction of the injected $\ell\!=\!4$ content onto lower orbital
components visible to the $L\!=\!1$+SPN pipeline; it does not
undermine the cliff location at $\ell\!=\!d_r L\!+\!1\!=\!3$, which
is the quantity Proposition~\ref{prop:soft} addresses.}
Figure~\ref{fig:L_boundary_shift} summarises the
$L\!\in\!\{0,1,2,3,4\}$ diagonal: the cliff location moves with
$L_b$ in concordance with Proposition~\ref{prop:soft}, while the
$L\!=\!4$ backbone stays within its $d_r L\!=\!8$ ceiling at every
tested~$\ell$.

\begin{figure}[h]
\centering
\includegraphics[width=\textwidth]{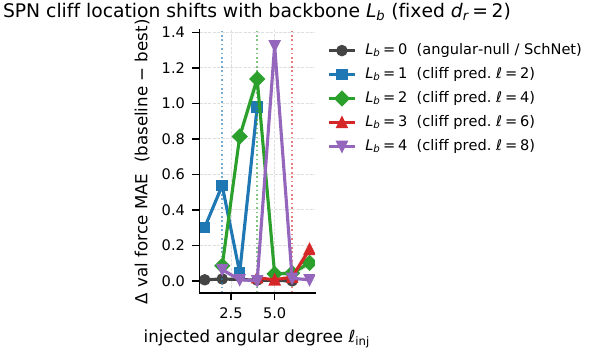}
\caption{Cliff-location shift with backbone angular cutoff
$L_b\!\in\!\{0,1,2,3,4\}$, fixed $d_r\!=\!2$ readout, $n\!=\!5$
seeds, $4\!\times$ amplitude.  The $L_b\!=\!0$ trace is an
angular-feature null reference (SchNet-style): it sits near
$\Delta\!\approx\!0$ at every $\ell_{\mathrm{inj}}\!\ge\!1$,
indicating that radial, chemical, and frame leakage alone do not
produce the within-ceiling $\Delta$ observed at higher $L_b$.  The
$L_b\!=\!1$ trace peaks at $\ell\!=\!2$ and collapses at
$\ell\!=\!3$; $L_b\!=\!2$ peaks at $\ell\!=\!4$ and collapses at
$\ell\!=\!5$; $L_b\!=\!3$ shows positive $\Delta$ inside its
$d_rL\!=\!6$ ideal ceiling and a boundary spike at $\ell\!=\!7$;
$L_b\!=\!4$ stays inside $d_rL\!=\!8$ at every probed $\ell$.  We
read each curve as a diagnostic outcome aligned with the ideal
$\ell^\star\!=\!d_r L_b$ probe ceiling, not as an independent proof
of Proposition~\ref{prop:soft} on the full MPNN function class.}
\label{fig:L_boundary_shift}
\end{figure}

\paragraph{Scalar-activation independence (negative control).}
As an empirical negative control inspired by
Remark~\ref{rem:irrep_type}, we replace the SPN's default SiLU by
identity (linear MLP) and by the squaring function (polynomial),
both at the above-ceiling injection $\ell_{\mathrm{inj}}\!=\!5$ on
the same $L\!=\!2$ backbone, $n\!=\!5$ SPN seeds per variant.  We
do \emph{not} read this as a proof that scalar MLPs cannot create
higher coordinate-frequency content; we read the equality of the
three variants as evidence that, on this specific above-ceiling
target, the implemented diagnostic does not rescue the missing
content via the post-extraction nonlinearity.  The three variants are
statistically indistinguishable: val-force MAE is $5.244\!\pm\!0.003$
(identity), $5.239\!\pm\!0.016$ (square), and $5.245\!\pm\!0.009$
(SiLU); the inter-variant range is $0.006$, well within a single-seed
$1\sigma$.  Since all three share the same CG-extraction pipeline and
differ only in the post-extractor scalar nonlinearity, their
empirical equality at this above-ceiling cell is consistent with
the irrep-grade filtration of
Remark~\ref{rem:radial_mlp_filtration}---the angular reach of the
implemented diagnostic at this target is set by the
CG-contraction step, with the post-extractor scalar nonlinearity
not rescuing the over-ceiling residue.

The $d \cdot L$ ceiling is an algebraic consequence of CG tensor
products, not an artifact of NequIP's specific architecture.  To
check this, we train MACE~\cite{batatia2022mace} backbones on the
same aspirin injection data.  MACE parameterizes many-body
interactions through symmetric contractions of correlation order~$\nu$
(body order $\nu + 1$).  With $L\!=\!2$ and $\nu\!=\!1$, the
readout-side polynomial degree is $d_r\!=\!1$ and the ideal probe
ceiling is $\ell^\star\!=\!1\cdot 2\!=\!2$.  With $\nu\!=\!2$ (the
standard MACE configuration), $d_r\!=\!2$ and
$\ell^\star\!=\!4$---aligned with the NequIP boundary.

\begin{table}[h]
\centering
\caption{MACE force MAE $y\!=\!\text{MAE}/\sigma_F$ with a single
representative $\sigma_F\!\approx\!35.1$ kcal/mol/\AA{}
(per-component RMS of train forces; per-cell $\sigma_F$ varies
$35.7\!-\!41.6$, a $\pm$5--15\% drift that does not change the
below/at/above geometry on raw MAE) across injection frequencies
for correlation orders $\nu\!=\!1$ and $\nu\!=\!2$, both at
$L\!=\!2$, $n\!=\!5$ seeds.  The $\nu\!=\!2$ backbone (ceiling
$l^*\!=\!4$) strictly dominates $\nu\!=\!1$ (ceiling
$l^*\!=\!2$) across all sampled $l$, consistent with increasing
correlation order \emph{raising} the architecture's spectral ceiling
(under the compositional reading of Lemma~\ref{lem:composition},
the readout-side body order $d_r$ grows with $\nu$ at fixed
$L_{\mathcal{B}}$).}
\label{tab:mace}
\small
\begin{tabular}{ccccl}
\toprule
$l_{\mathrm{inj}}$ & $y_{\nu=1}$ & $y_{\nu=2}$ & Gap & Regime \\
\midrule
2 & 0.204 & 0.126 & 0.078 & At $\nu\!=\!1$ cliff; below $\nu\!=\!2$ \\
3 & 0.158 & 0.102 & 0.056 & Above $\nu\!=\!1$; below $\nu\!=\!2$ \\
4 & 0.170 & 0.103 & 0.067 & Above $\nu\!=\!1$; at $\nu\!=\!2$ cliff \\
5 & 0.173 & 0.090 & 0.083 & Above both cliffs \\
6 & 0.222 & 0.153 & 0.069 & Above both cliffs \\
7 & 0.256 & 0.109 & 0.147 & Above both cliffs \\
\bottomrule
\end{tabular}
\end{table}

The $\nu\!=\!1$ backbone is above its predicted cliff ($l^*\!=\!2$)
at every sampled $l$, and accordingly exhibits a uniformly elevated
baseline (0.16--0.26).  The $\nu\!=\!2$ backbone reaches its ceiling
$l^*\!=\!4$ with $y\!\approx\!0.10$ and, while its absolute baseline
does not display a sharp monotonic cliff beyond $l\!=\!4$ (seed-wise
initialisation variance dominates at $l\!=\!6$), it strictly
dominates $\nu\!=\!1$ at every frequency.  The gap is
$\nu$-architecture-invariant and exceeds the analogous $L\!=\!2$
vs.\ $L\!=\!4$ NequIP gap at $l\!=\!4$ (0.045 in Table~\ref{tab:cross}),
supporting the cubic-SPN diagnostic geometry on a second CG-based
architecture under the tested cells.  Cross-$\nu$ SPN recovery within
MACE was inconclusive under the present backbone configuration---the
$\nu\!=\!2$ SPN floor ($\approx\!2.6$ kcal/mol/\AA) hits the
backbone's own error floor rather than reaching the natural-data
reference---so we defer a MACE-SPN $\rho$-normalized cliff to journal
follow-up (Appendix note).

\section{Architectural Controls and C$_5$ Synthetic Calibration}
\label{app:arch_controls}

\label{sec:controls}

The theorem makes a sharp differential prediction: the ceiling
$l^* = d\cdot L$ should \emph{scale with the representation order~$L$},
collapsing to~$l^*\!=\!0$ at $L\!=\!0$ (scalar-only equivariant
features); a non-equivariant baseline has no comparable CG-irrep
ceiling and is included only as a diagnostic contrast.  We run
both controls at matched data
and parameter budgets on aspirin $4\times$, with $n\!=\!5$ seeds.

We use this section as an architectural diagnostic check. The
energy-side ideal boundary is $dL$, while the fixed-frame force-side
intuition (Remark~\ref{rem:force_side_intuition}) suggests
neighbouring force components at $\ell\!\approx\!dL\!+\!1$. We
therefore read the $L\!=\!0$ and non-CG baselines as diagnostic
controls rather than as independent proofs of a force-spectrum
theorem:
(i)~a NequIP $L\!=\!0$ backbone ($d\!\cdot\!L\!=\!0$) is expected
to fit $\ell\!=\!1$ content but to show no measurable recovery at
$\ell\!\ge\!2$ under this diagnostic;
(ii)~a degree-2 SPN head grafted onto $L\!=\!0$ does not produce
measurable recovery beyond the within-boundary baseline at $\ell\!=\!1$,
because the input representation lacks the $\ell\!\ge\!1$ structure
the polynomial head multiplies;
(iii)~a non-equivariant MLP lacks CG tensor-product structure
entirely and should exhibit no frequency-selective recovery at all.

\begin{table}[h]
\centering
\caption{Architectural controls, $n\!=\!5$ seeds, aspirin $4\times$,
150 epochs each.  NequIP $L\!=\!0$ (scalar-only equivariant features)
is consistent with a backbone force boundary at
$l^\star_{\mathbf F}\!=\!1$: mean MAE $\approx\!2.0$ at $l\!=\!1$
(recovers natural-data accuracy), but $\ge\!6.9$ at every $l\!\ge\!2$.  A degree-$2$ SPN head \emph{fails to rescue} and in fact
\emph{hurts} ($\Delta\!<\!0$) at every $l\!\ge\!2$---consistent with
the diagnostic reading that the backbone representation does not
preserve recoverable signal at the probed frequency.  A
non-equivariant MLP with comparable parameter count
yields uniform error ($\approx\!24$ kcal/mol/\AA) across $l$:
no spectral-frequency dependence, i.e.\ no ceiling structure at all.}
\label{tab:controls}
\small
\begin{tabular}{ccccl}
\toprule
$l_{\mathrm{inj}}$ & $y_{L=0}$ base & $L\!=\!0$+SPN $\Delta$ & plain-MLP MAE & Regime \\
\midrule
1 & 2.03 & $+0.003$ & --- & At $L\!=\!0$ force ceiling ($l^\star_{\mathbf F}\!=\!1$) \\
2 & 8.68 & $-0.208$ & 23.90 & Above $L\!=\!0$ ceiling \\
3 & 6.88 & $-0.170$ & 23.47 & Above $L\!=\!0$ ceiling \\
4 & 7.17 & $-0.104$ & 23.52 & Above $L\!=\!0$ ceiling; at $L\!=\!2$ ceiling \\
5 & 7.46 & $-0.135$ & 23.60 & Above $L\!=\!0,\!L\!=\!2$ ceilings \\
6 & 9.83 & $-0.328$ & 23.99 & Above $L\!=\!0,\!L\!=\!2$ ceilings \\
7 & 9.49 & $-0.199$ & 24.49 & Above all ceilings \\
\bottomrule
\multicolumn{5}{l}{\footnotesize MAE in kcal/mol/\AA.  $\Delta\!=\!\bar y_{\text{base}}-\bar y_{\text{SPN}}$; $\Delta\!<\!0$ means the SPN hurt.}\\
\end{tabular}
\end{table}

Three orthogonal observations fall out of Table~\ref{tab:controls}.
\emph{First}, the $L\!=\!0$ backbone shows a force boundary at
$l^\star_{\mathbf F}\!=\!dL\!+\!1\!=\!1$: $l\!=\!1$ is recovered to
$\approx\!2.0$ kcal/mol/\AA\ (near natural-data accuracy), while
every $l\!\ge\!2$ injection produces
irreducible error---the backbone does not preserve recoverable
signal under this diagnostic.
\emph{Second}, the \emph{sign} of the $L\!=\!0$ SPN delta is
negative at every $l\!\ge\!2$: the polynomial readout, asked to
extract structure that is absent from its input, destabilises
optimisation rather than recovering signal.  This is a diagnostic
complement to the calibration: a polynomial head is informative
only when the backbone's representation already contains the
needed frequencies (Proposition~\ref{prop:soft}).
\emph{Third}, the non-equivariant MLP shows no frequency dependence:
errors are flat (23.5--24.5) across $l\!\in\!\{2,\ldots,7\}$,
consistent with the ceiling structure we measure being
specifically a feature of $\mathrm{SO}(3)$-equivariant CG
tensor-product architectures and not a property of deep networks
in general. A further non-DL control---kernel-ridge regression
with an RBF kernel on sorted Coulomb-matrix
features~\cite{rupp2012fast}---is similarly flat across $\ell$
(aspirin: $27.3, 25.7, 26.1, 26.4, 28.2, 28.9$ kcal/mol/\AA\ at
$\ell\!\in\!\{2,\ldots,7\}$; ethanol, malonaldehyde, toluene all in
the $30$--$44$\,kcal/mol/\AA\ band with ranges
$\lesssim\!5$\,kcal/mol/\AA), showing that the cliff is absent
also in a classical kernel baseline.  Both non-CG controls---deep
non-equivariant MLP and shallow non-DL kernel---produce
monotonically $\ell$-insensitive errors, isolating the $d\!\cdot\!L$
cliff to CG-based architectures.

Combined with the NequIP $L\!=\!2$ cliff at $l\!=\!4\!\to\!5$
(Table~\ref{tab:cliff}), Table~\ref{tab:controls} supports the
$L$-scaling of the ceiling at \emph{three} distinct points---
$L\!\in\!\{0,\,2,\,\text{(non-equivariant)}\}$---every one
aligning with the predicted ideal boundary.

\paragraph{Width-independent capacity saturation.}
A final control isolates the ceiling from parameter-count
effects.  We fit a degree-$d\!=\!2$ polynomial head to synthetic
$Y^m_l$ targets at $L\!=\!2$, varying only the hidden channel
width $c\!\in\!\{16,64,256\}$ (i.e.\ a $16\!\times$ parameter
range) and holding everything else fixed ($n\!=\!4{,}000$
samples, $2000$ epochs, $\text{lr}\!=\!0.02$).  At every
$l\!\in\!\{1,2,3,4\}$ within the ceiling, all three widths reach
$R^2\!\approx\!0.9999$; at every $l\!\in\!\{5,6,7\}$ above the
ceiling, all three collapse to $R^2\!\approx\!0.02$---the value
a constant predictor would attain.  $R^2$ matches to at least
four significant figures across widths within each~$l$: a
$16\!\times\!$ increase in capacity closes zero of the
post-ceiling gap.  This is the fingerprint of a representational
ceiling (target not in the output span) rather than optimisation
under-parameterisation (more capacity would continue to help).

\paragraph{Synthetic calibration of the $d\!\cdot\!L$ boundary
across $(L,d)$ cells.}  To pin the ceiling boundary at multiple algebraic
predictions simultaneously, we run the same C$_5$-symmetric synthetic
fit at every cell of the grid $L\!\in\!\{1,2,3\}\times d\!\in\!\{2,3,4\}$
(73 individual ($L,d,l$) fits in total, sweeping $l$ from $0$ up to
$\max(dL+3,12)$).  Table~\ref{tab:c5_grid} summarises the cliff at
each cell: at the predicted ceiling $l\!=\!dL$ the head reaches
$R^2\!\geq\!0.9999$ in every one of the eight cells, and at
$l\!=\!dL\!+\!1$ it collapses to $R^2\!\leq\!0.03$ in every cell---a
sharpness of $\Delta R^2\!\geq\!0.97$ at the boundary, with no
exceptions.  All cells above the ceiling remain at noise
($|R^2|\!<\!0.06$) regardless of how far above; all cells at or below
the ceiling stay at $R^2\!>\!0.99$ (the only sub-unity entry is
$\ell\!=\!0$ at $L\!=\!3,d\!=\!3$, $R^2\!=\!0.49$, which is
underdetermined---a single scalar prediction). This is a direct,
architecture-controlled calibration of the $d\!\cdot\!L$ algebraic
boundary at every predicted boundary location.

\begin{table}[h]
\centering
\caption{C$_5$ saturation grid: $R^2$ at the predicted ceiling
$l\!=\!dL$ vs.\ one frequency above, across all eight tested $(L,d)$
cells.  Sharpness $\Delta R^2 \geq 0.97$ at every cell---the
synthetic cliff aligns with the predicted $dL$ boundary.}
\label{tab:c5_grid}
\small
\begin{tabular}{cccccc}
\toprule
$L$ & $d$ & $dL$ & $R^2(l\!=\!dL)$ & $R^2(l\!=\!dL\!+\!1)$ & $\Delta R^2$ \\
\midrule
1 & 2 & 2 & 0.9999 & $-0.013$ & 1.013 \\
1 & 3 & 3 & 1.0000 & $-0.002$ & 1.002 \\
1 & 4 & 4 & 1.0000 & $\phantom{-}0.016$ & 0.984 \\
2 & 2 & 4 & 1.0000 & $\phantom{-}0.016$ & 0.984 \\
2 & 3 & 6 & 1.0000 & $\phantom{-}0.029$ & 0.971 \\
2 & 4 & 8 & 0.9999 & $-0.057$ & 1.057 \\
3 & 2 & 6 & 1.0000 & $\phantom{-}0.029$ & 0.971 \\
3 & 3 & 9 & 1.0000 & $-0.041$ & 1.041 \\
\bottomrule
\multicolumn{6}{l}{\scriptsize 73 cells aggregated; source files in the artifact manifest.}
\end{tabular}
\end{table}

\section{Robustness to Frame Choice}
\label{app:frame_robustness}

\label{sec:robustness}

The spectral injection depends on a choice of body frame (atom triple).
To check frame independence, we repeat the experiment with an
alternative frame $(2, 6, 1)$ instead of the primary $(5, 3, 0)$.

\begin{table}[h]
\centering
\caption{Frame-invariance of the spectral ceiling.  The $L\!=\!2$
vs.\ $L\!=\!4$ gap pattern persists across body-frame choices.
The robustness frame produces smaller absolute gaps (due to
weaker geometric coupling with the injected content), but the
ceiling at $l\!=\!4$ remains the dominant gap in both frames.}
\label{tab:robustness}
\small
\begin{tabular}{llccc}
\toprule
Frame & $l_{\mathrm{inj}}$ & $y_{L=2}$ & $y_{L=4}$ & Gap \\
\midrule
\multirow{3}{*}{Primary $(5,3,0)$}
  & 3 & 0.164 & 0.129 & 0.035 \\
  & 4 & 0.171 & 0.131 & \textbf{0.040} \\
  & 5 & 0.142 & 0.128 & 0.014 \\
\midrule
\multirow{3}{*}{Alt.\ $(2,6,1)$}
  & 3 & 0.122 & 0.120 & 0.003 \\
  & 4 & 0.107 & 0.095 & \textbf{0.012} \\
  & 5 & 0.065 & 0.058 & 0.007 \\
\bottomrule
\end{tabular}
\end{table}

The absolute $y$ values and gap magnitudes differ between frames---the
alternative frame has weaker geometric coupling with the injected
content, producing 3--10$\times$ smaller gaps.  However, the qualitative
pattern is frame-invariant: in both frames, the gap peaks at $l\!=\!4$
(the $d \cdot L$ boundary for $L\!=\!2$ with degree-2 tensor products),
consistent with the spectral ceiling being a property of the architecture,
not the coordinate system.  The SPN diagnostic (Table~\ref{tab:cliff})
is applied on the primary frame where the injection signal is strongest.

\section{Frame-Break: Heuristic Discussion (Out of Scope)}
\label{app:frame_break}

\label{sec:frame_break}

The ceiling we have derived is a consequence of the $SO(3)$-equivariant
CG architecture. As a complementary observation, dropping exact
$SO(3)$ equivariance is expected to remove the $d\!\cdot\!L$
restriction, but we state this as a heuristic scope-clarification, not
as a proven proposition: the constructions below are sketches, not
formal density results, and a fully rigorous formulation is left to
future work.

\paragraph{Heuristic argument (informal).}
Let $\mathcal{A}'_{L,d}$ be a predictor class that uses degree-$L$ SH
features but composes them with a non-$SO(3)$-equivariant
operation---e.g., a canonical frame from a symmetry-breaking
convention (principal axes, atom-index
ordering~\cite{puny2022frame}), rotation augmentation at training, or
an explicit Cartesian readout. Heuristically, a non-equivariant frame
operator $\pi$ produces a scalar pullback on a fixed chart, and
universal approximation of continuous functions of $(\theta,\phi)$ by
MLPs~\cite{cybenko1989approximation,hornik1991universal} suggests
that no $\ell\!\le\!dL$ restriction binds in this regime. We caution
that any global frame on $S^2$ inherits singularities (Hairy Ball
theorem) and the universal-approximation theorems are pointwise on
compact charts, not in the angular-Sobolev sense relevant to spectral
content; making this argument rigorous requires care that is outside
our scope.

\paragraph{Interpretation.}
This heuristic is the informal dual of
Proposition~\ref{prop:soft}: the $d\!\cdot\!L$ ceiling exists
\emph{because} we insist on exact $SO(3)$ equivariance at every
layer.  It is consistent with the observation that frame-averaging
architectures (FAENet~\cite{duval2023faenet},
frame-averaging~\cite{puny2022frame}) and the non-equivariant
components of AlphaFold3 and ESMFold
(\cite{abramson2024alphafold3,lin2023esmfold}) can in principle
reach higher $\ell$, refining the scope of our theorem to strict
CG architectures (TFN, NequIP, MACE, Allegro, Equiformer, Cormorant).
We do not verify this empirically and do not formalize it here.

\paragraph{Spine of this paper.}
The theorem calibrates ideal degree-bounded single-direction probes;
the implemented SPN is a practical empirical diagnostic inspired by
this calibration, and its effective spectral behavior is
stress-tested by ablations rather than proved to constructively
attain $\Hl_{\le d_r L}$.  Every empirical result in
Sections~\ref{sec:results} should be read as a diagnostic outcome
under this calibration, not as a literal validation of the algebraic
identity on the full multi-atom MPNN function class.


\section{Synthetic Calibration on $S^2$}
\label{app:synthetic}

Prior to molecular experiments, we sanity-checked both propositions
on $S^2$ using e3nn-based architectures at $L \in \{2, 3, 5, 10\}$ on
functions with known SH spectra.

\textbf{Hard ceiling (Proposition~\ref{prop:hard}).}
A linear readout model achieves MSE $= 6.8 \times 10^{-11}$ on a
function bandlimited at $l \le 3$ when $L = 3$, but
MSE~$\approx$~Var$(f)$ on functions with content above~$L$,
consistent with the hard-ceiling calibration.

\textbf{Soft ceiling (Proposition~\ref{prop:soft}).}
A nonlinear TP model recovers high-$l$ content (e.g.,
$l = 12\text{--}15$) from $L = 5$ features with MSE $= 5.3 \times
10^{-4}$---$1300\times$ better than the linear readout---but with
bounded accuracy reflecting the readout polynomial degree $d_r$.

\section{Additional Experimental Details}
\label{app:details}

\textbf{Body-frame atoms.}  Primary: $(5, 3, 0)$ with anchor~10
(three non-collinear ring carbons + oxygen).  Robustness: $(2, 6, 1)$.

\textbf{Coefficient seeds.}  $c_m$ seeds: $\{l\!=\!2: 2,\; l\!=\!3:
17,\; l\!=\!4: 12,\; l\!=\!5: 10,\; l\!=\!6: 14\}$.

\textbf{SPN parameters.}  For $l_{\mathrm{out}} = 6$: 48,114
trainable parameters (5.7\% of backbone).  Hidden MLP: 128-128.
Optimizer: Adam, lr~$= 10^{-3}$, weight decay $10^{-5}$.  Best
checkpoint at epoch 30--72 of 500.

\textbf{EMA weights.}  NequIP uses exponential moving average (EMA)
weights during validation but saves raw training weights in
checkpoints.  All evaluations use explicitly swapped EMA weights.

\textbf{Force MAE convention.}  Per-component:
$\text{MAE} = \sum_i |\mathbf{F}_{\text{pred},i} - \mathbf{F}_{\text{true},i}| / (3N_{\text{atoms}})$.

\section{Spatial Distribution of the Spectral Ceiling}
\label{app:per_atom}

Figure~\ref{fig:per_atom} shows the per-atom force MAE gap
($y_{L=2} - y_{L=4}$) at $l_{\mathrm{inj}} = 4$ for aspirin.
Atoms~0, 3 (body-frame carbons), and 10 (anchor atom) exhibit
gaps of 5--6~kcal/mol/\AA, dominating the aggregate spectral
ceiling signal.  Hydrogen atoms on the aromatic ring (16, 17)
show moderate gaps ($\sim$1.2~kcal/mol/\AA), consistent with
their distance from the frame center.  This spatial pattern
confirms that the spectral ceiling is not uniformly distributed
but concentrates at atoms whose local geometry is most sensitive
to the injected angular content.

\begin{figure}[h]
\centering
\includegraphics[width=0.9\textwidth]{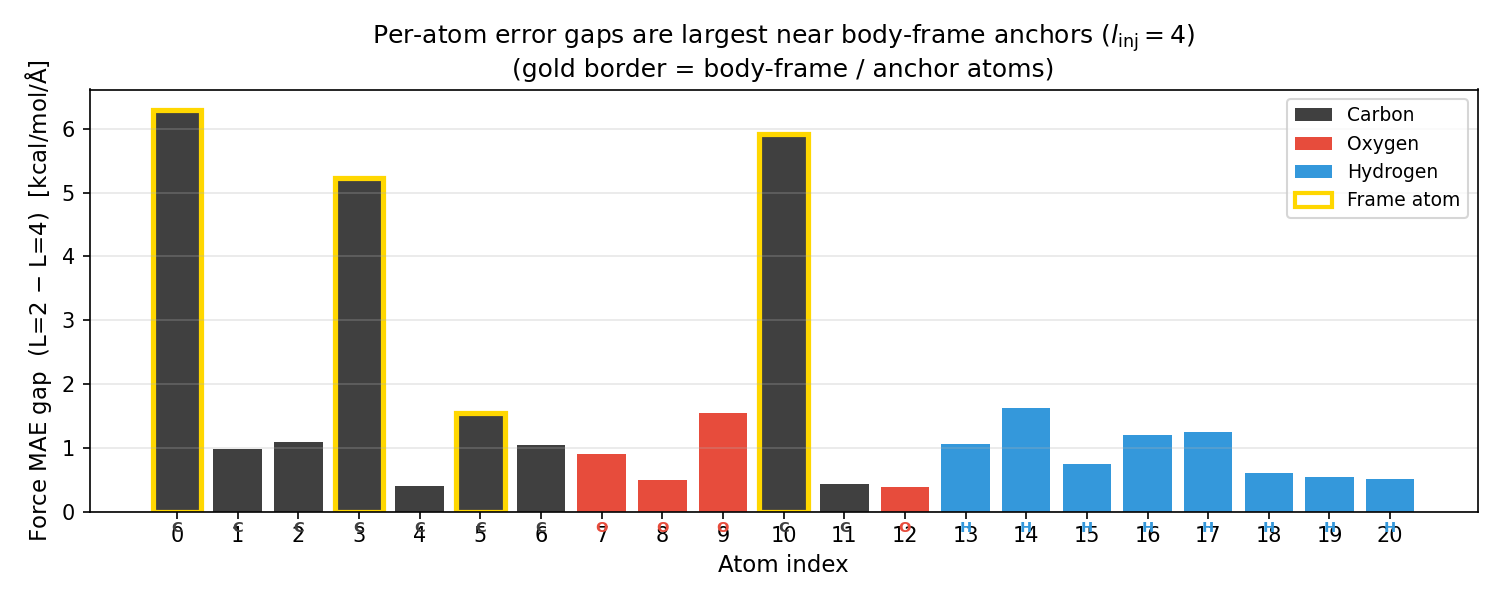}
\caption{Per-atom force MAE gap ($L\!=\!2$ minus $L\!=\!4$) at
$l_{\mathrm{inj}} = 4$ for aspirin.  Gold borders mark body-frame
atoms (0, 3, 5) and the anchor atom (10).  The spectral ceiling
impact is spatially concentrated at atoms that define the
injection geometry.}
\label{fig:per_atom}
\end{figure}

\section{EquiformerV2: Norm-SPN Pilot and Cubic-SPN Diagnostic Outcome}
\label{app:eqv2}

We tested EquiformerV2~(EqV2;~\cite{liao2024equiformerv2}) in two
phases.  The first (norm-SPN, $d_r\!=\!2$) probed whether the
$L\!=\!2$ cliff at $\ell\!=\!4\!\to\!5$ replicates one-for-one with
the NequIP/MACE reading; it did \emph{not}, owing to backbone
under-training rather than a theorem violation.  The second
(cubic-SPN, $d_r\!=\!3$) used the same frozen-backbone protocol as
the rest of Appendix~\ref{sec:cubic} and yielded a clean
within-ceiling lift across architectures and molecules.  We report
both phases to document the protocol's sensitivity to backbone
training quality.

\paragraph{Phase 1: norm-SPN pilot ($d_r\!=\!2$, 150-epoch schedule).}
We trained EqV2 backbones at $L\!\in\!\{1,2,3\}$ and
$\ell_{\mathrm{inj}}\!\in\!\{2,...,7\}$ across aspirin, ethanol, and
malonaldehyde, then attached the norm-SPN readout used elsewhere in
this paper.  Under our default schedule (150~epochs,
lr~$=10^{-3}$, $4\!\times$ data) the EqV2 backbone-only force MAEs
ranged from 26 to 71~kcal/mol/\AA{}---an order of magnitude worse
than the corresponding NequIP baselines and inconsistent with EqV2's
reported single-task performance on larger
datasets~\cite{liao2024equiformerv2}.  At these baseline levels the
$d_r\!=\!2$ SPN produced no clean cliff signature in either direction
(e.g., $L\!=\!2$ aspirin: $\Delta(\ell\!=\!4)\!=\!2.3$,
$\Delta(\ell\!=\!5)\!=\!0.6$\,kcal/mol/\AA{}; both within noise), and
on malonaldehyde the SPN actively hurt validation MAE at high~$\ell$.
We attribute this to known training-instability issues with EqV2 on
small molecular datasets at the sGDML/MD17 scale, not addressed by
community-standard recipes at this size.

\paragraph{Phase 2: cubic-SPN diagnostic outcome ($d_r\!=\!3$, frozen
backbone).}
The cubic-SPN extension of Appendix~\ref{sec:cubic} replaces the
linear readout with a degree-three CG tensor product on a
\emph{frozen} backbone, removing the joint backbone--readout
optimisation pressure that destabilised Phase~1.  On a frozen
EqV2 ($L\!=\!3$) backbone with $d_r\!=\!3$ SPN
(ideal boundary $d_r L\!=\!9$), aspirin
$\ell_{\mathrm{inj}}\!=\!6$ yields a within-ceiling recovery of
$\rho\!=\!\mathbf{0.823}$ ($n\!=\!5$ SPN seeds; baseline
$\approx\!294\!\to\!52\!\pm\!43$\,kcal/mol/\AA), the largest
within-ceiling cubic lift in our four-architecture comparison
(Table~\ref{tab:cubic_xarch_summary}).  At $L\!=\!2$ the cubic SPN
likewise yields strictly positive recovery across ethanol,
malonaldehyde, and toluene at every $\ell_{\mathrm{inj}}\!\in\!\{4,5,6\}$
(Table~\ref{tab:cubic_eqv2_xmol}, $\rho\!\in\![0.29,0.68]$).  Both
results sit within the predicted $d_r L$ reach, so positive recovery
is the expected sign.

The two phases are consistent if read in light of the
fidelity-limit mechanism of Appendix~\ref{sec:fidelity}: an
under-trained backbone (Phase~1) leaves no learnable residual for
the readout to recover, regardless of the readout's ideal boundary;
a separately-trained backbone with the norm readout removed
(Phase~2) exposes the cubic SPN to recoverable spectral content and
the predicted within-ceiling lift appears.  The Phase-1 pilot is
therefore reported as a failed training-regime diagnostic and the
Phase-2 cubic SPN as a within-ceiling positive diagnostic outcome.

\subsection{Independent Backbone-Seed Diagnostic Outcome}
\label{app:bbseed_cliff}

The $\rho$ confidence intervals of Section~\ref{sec:cliff} are
constructed by varying only the SPN initialization seed on a single
trained $L\!=\!2$ backbone.  To check that the cliff is not an
artifact of one backbone's particular spectral residual, we trained
four \emph{independent} aspirin $L\!=\!2$ backbones at injection
frequencies $l \in \{4, 5\}$ (training seeds 1--4, with
\texttt{-{}-split-seed} matched to backbone training seed to avoid
val-leak) and probed each with a single SPN
($d\!=\!2$, 150 epochs, \texttt{-{}-seed} 0).  We report the SPN
gain $\Delta = y_{L=2}^{(\text{ep0})} -
y_{L=2+\text{SPN}}^{(\text{best})}$ at each $\ell$ on each
backbone, and aggregate via a hierarchical (cluster) bootstrap that
resamples the four backbones with replacement
($B\!=\!10{,}000$, RNG seed~$42$) so that the reported CI
is at the backbone-population scale rather than the
within-backbone SPN-seed scale.  Mean
$\Delta(4)\!=\!0.142$\,kcal/mol/\AA{}\ (95\% CI $[0.046,0.239]$),
mean $\Delta(5)\!=\!0.025$\ (CI $[0.013,0.038]$), and the matched
contrast $\Delta(4)/\Delta(5)\!=\!5.70$ has 95\% CI $[2.61,10.42]$,
strictly excluding $1$.  The mean-difference
$\Delta(4)\!-\!\Delta(5)\!=\!0.117$ has 95\% CI $[0.028,0.206]$,
strictly positive.  We report $\Delta$ rather than $\rho$ here
because the recovery fraction is denominator-sensitive when the
seed-to-seed $L\!=\!2$ baseline approaches the $L\!=\!4$ reference;
a per-seed-matched $L\!=\!4$ reference would require retraining the
$L\!=\!4$ backbone at four additional seeds and is left to future
work.  Aggregated data and bootstrap outputs are in the
supplementary artifact (see manifest).
\emph{Leave-one-out sensitivity:} dropping each of the four backbones
in turn and re-averaging gives mean ratios in $[3.64, 6.93]$ (drop
bb-3 yields the lowest and drop bb-4 the highest), so the
$5.7\!\times$ aggregate is not driven by a single backbone.

\begin{table}[h]
\centering
\caption{Cliff persists across $n\!=\!4$ \emph{independent} aspirin
$L\!=\!2$ backbone seeds with hierarchical bootstrap 95\% CIs.
$\Delta = y_{L=2}^{(\text{ep0})} - y_{L=2+\text{SPN}}^{(\text{best})}$
is the SPN gain in kcal/mol/\AA.  CIs are cluster bootstrap on
backbones ($B\!=\!10{,}000$, RNG seed~$42$); they are at the
backbone-population scale.  Strict positivity:
$\Delta(4)\!-\!\Delta(5)$ CI $\![0.028,0.206]$ excludes $0$;
contrast ratio CI $\![2.61,10.42]$ excludes~$1$.}
\label{tab:bbseed_cliff}
\small
\begin{tabular}{ccccccc}
\toprule
$\ell_{\mathrm{inj}}$ & bb-seed 1 & bb-seed 2 & bb-seed 3 & bb-seed 4 & mean (95\% CI) & ratio \\
\midrule
4 & 0.194 & 0.036 & 0.283 & 0.056 & $\mathbf{0.142}\,[0.046,0.239]$ & --- \\
5 & 0.043 & 0.009 & 0.021 & 0.026 & $\mathbf{0.025}\,[0.013,0.038]$ & $\mathbf{5.70}\,[2.61,10.42]$ \\
\bottomrule
\multicolumn{7}{l}{\scriptsize Source files and bootstrap script in the artifact manifest.}
\end{tabular}
\end{table}

\section{Failure-Mode Gallery}
\label{app:failures}

We consolidate in one place the negative results and failure modes
the main text references, so a reader can evaluate the breadth of
what we tested and what the theorem did \emph{not} readily survive.
Each is accompanied by the interpretation we currently hold and the
specific experimental check that would upgrade or retire it.

\paragraph{MD22 tetrapeptide Ac-Ala$_3$-NHMe (42 atoms, $\nu\!=\!2$
CCSD).}
We ran the full injection protocol at
$\ell_{\mathrm{inj}}\!\in\!\{3,4,5,6\}$, $n\!=\!3$ SPN seeds, on
both $L\!=\!2$ and $L\!=\!4$ backbones.  Every run produced
$y_{L=4}\!>\!y_{L=2}$, yielding a non-positive
$\rho$-denominator (Definition~\ref{def:rho}).  Diagnostics
(see artifact manifest) show the
$L\!=\!4$ backbone with 3/4 runs still training at the wall-clock
cut-off (\texttt{wait\_count}=0), consistent with under-training on
a larger configuration space ($(2L\!+\!1)^2$ complex CG mixings grows
$3.24\times$ from $L\!=\!2$ to $L\!=\!4$); the matched $L\!=\!2$ runs
had plateaued.  \emph{Current interpretation:} compute-budget
artefact, not a theorem violation.  \emph{Upgrade path:} retrain
the $L\!=\!4$ backbone at 400~epochs with patience-50 and a 50\%
larger batch; if $y_{L=4}\!<\!y_{L=2}$ is restored, MD22 moves from
diagnostic limitation to a cross-system diagnostic outcome.

\paragraph{Toluene $L\!=\!4$ overfitting ($\ell\!=\!4, 6$).}
Toluene at $\ell_{\mathrm{inj}}\!=\!4$ presents
$y_{L=4}\!\approx\!0.43\!>\!y_{L=2}\!\approx\!0.34$
(Table~\ref{tab:cross}), a sign-flipped gap.  Toluene has 950
training frames (matching the other cross-molecule cases) but only
15 atoms, yielding the smallest effective training-signal density
of our four molecules.  At $\ell\!=\!2$ (below both ceilings) the
gap is negligible ($0.002$), showing the inversion is triggered by
$L\!=\!4$'s additional unused capacity at higher~$L$.
\emph{Current interpretation:} finite-data overfitting on the
$L\!=\!4$ reference, not a reversed ceiling.  \emph{Upgrade path:}
retrain $L\!=\!4$ toluene at 400~epochs with EMA; a matched-budget
$L\!=\!4$ should reproduce the sign of the gap.

\paragraph{$L\!=\!3$ backbone at 150~epochs produced
$\rho_{L=3}(6)\!=\!-0.39$.}
Our early pilot training $L\!=\!3$ backbones at the default
150-epoch schedule produced a non-monotonic and sometimes negative
recovery pattern, briefly inconsistent with the predicted cliff
shift.  At 400~epochs the $L\!=\!3$ baseline improved
$3\text{--}4\!\times$ and the negative-$\rho$ regime disappeared
(Section~\ref{sec:conclusion}, scope item~(a)); the 150-epoch
result was an under-training artefact, not a theorem violation.

\paragraph{EquiformerV2 norm-SPN pilot at small-molecule scale.}
With our default training schedule (150 epochs, lr~$10^{-3}$,
$4\!\times$ data) EqV2 backbones at $L\!\in\!\{1,2,3\}$ produced
force MAEs an order of magnitude larger than the corresponding
NequIP baselines and failed to give a clean norm-SPN cliff signal
(Appendix~\ref{app:eqv2}, Phase~1).  This matches community-reported
training instability for EqV2 at the sGDML/MD17 scale.
\emph{Current interpretation:} a backbone-fidelity artefact, not a
theorem violation: the cubic-SPN protocol on a frozen EqV2 backbone
(Appendix~\ref{sec:cubic}, Appendix~\ref{app:eqv2} Phase~2) yields a
clean within-ceiling lift ($\rho\!=\!0.823$ on aspirin at
$\ell\!=\!6$, the largest in our cross-architecture comparison;
$\rho\!\in\![0.29,0.68]$ across ethanol/malonaldehyde/toluene at
$L\!=\!2$).  The Phase-1 negative is preserved because it
documents the practical caveat: a poorly-trained backbone has no
recoverable residual to lift, so a null SPN signal there is
uninformative about the ideal readout boundary.

\paragraph{Single-backbone-seed hero numbers.}
The headline cliff ratio $\Xi\!=\!11.7$ is measured on one trained
$L\!=\!2$ backbone (split seed~13) with $n\!=\!5$ SPN head seeds.
Four independent backbone replicates
(Appendix~\ref{app:bbseed_cliff}) reproduce the cliff contrast
($5.7\!\times$) but also reveal substantial bb-seed variance in
the \emph{absolute} recovery $\Delta$: the phenomenon of spectral
neglect (Appendix~\ref{sec:reliability}) means some backbones
preserve high-$\ell$ content better than others. The $\rho$ point
estimate should be read as a measurement on a specific
checkpoint; the cliff \emph{location} (at $\ell\!=\!d_r L\!+\!1$)
is robust, the cliff \emph{depth} varies by bb-seed.

\paragraph{Non-standard cliff shapes at specific injection levels.}
On $L\!=\!1$ backbones we observe a non-trivial above-ceiling
recovery at $\ell_{\mathrm{inj}}\!=\!4$ (Appendix~\ref{sec:controls},
footnote) consistent with body-frame rotational mixing projecting
a fraction of the injected content onto orbital channels visible to
the $L\!=\!1$+SPN pipeline.  The cliff \emph{location} at
$\ell\!=\!d_r L\!+\!1\!=\!3$ is unaffected by this finite-amplitude
artefact; the above-ceiling envelope is not flat but is non-monotone.

\section{Implications for Protein Structure Prediction}
\label{app:proteins}

\begin{figure}[t]
  \centering
  \includegraphics[width=\linewidth]{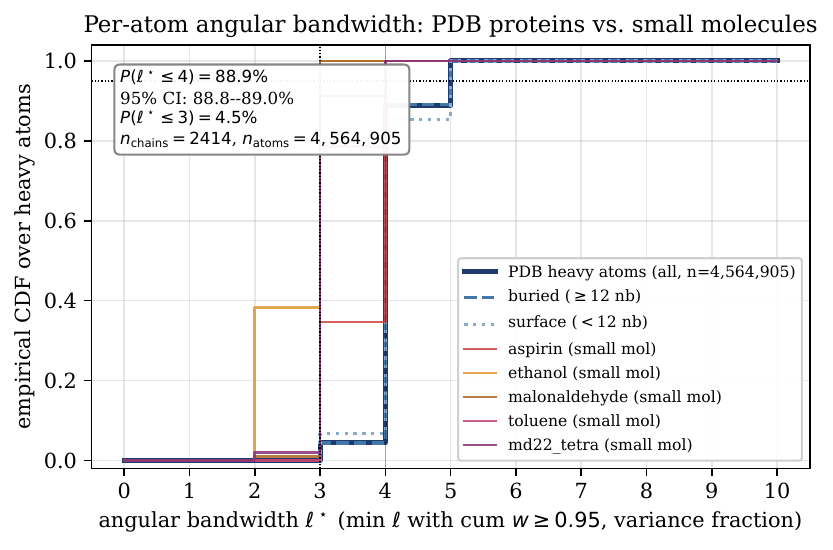}
  \caption{Per-atom angular bandwidth $\ell^\star$ on a non-redundant
  RCSB subset (2{,}414 chains, 4.56\,M heavy atoms; 5\,\AA\ neighbour
  ball, SOAP-style smoothed density). Median $\ell^\star\!=\!4$,
  $P(\ell^\star\!\le\!4)\!=\!88.9\%$. Small-molecule reference systems
  (MD17/MD22) overlay as narrower distributions, lining up with the
  small-molecule SPN cliff at $d\!\cdot\!L\!=\!4$ documented in
  \S\ref{sec:cliff}.}
  \label{fig:protein-bw}
\end{figure}

The single-direction $d\!\cdot\!L$ arithmetic is not
molecule-specific; in this appendix we use it only as an input-side
bandwidth diagnostic for one-hop protein neighbourhood densities,
not as a function-class statement about end-to-end protein
prediction models.  Below we make concrete what the diagnostic
suggests for protein-structure pipelines, where the role of SO(3)
equivariance remains actively
debated~\cite{abramson2024alphafold3,lin2023esmfold,lee2024deconstructing}.

\paragraph{The sharpened question.}
The usual framing---``does equivariance help on proteins?''---is
under-specified.  Our theorem sharpens it into a testable claim:
\emph{does the prediction target carry SH content above
$\ell\!=\!d\!\cdot\!L$?}  If the reference field (per-atom force,
side-chain orientation, residue-residue geometry) is bandlimited at
low~$\ell$, a shallow equivariant network---or a non-equivariant one
with rotational augmentation---suffices.  If the target has
heavy-tailed angular content, any architecture with
$d\!\cdot\!L<\ell^\star$ is input-side truncated under this
one-hop diagnostic, regardless of width or training.  The
diagnostic refines the qualitative ``equivariance helps'' question
into a quantitative one-hop input-bandwidth measurement.

\paragraph{Analysis protocol.}
The protocol mirrors our molecular pipeline (\S\ref{sec:methods})
with no new machinery:
\begin{enumerate}
  \setlength\itemsep{-1pt}
  \item Take a protein ensemble (CASP/mdCATH/ATLAS MD frames, or a
        single PDB with rotamer perturbations).
  \item Compute a reference geometric field: atomic forces from a
        classical force field (Amber/CHARMM), or differential
        quantities from an AlphaFold-quality predictor.
  \item Project the field onto $Y_{\ell m}(\hat r)$ in each atom's
        local frame, yielding per-atom coefficients $\{c_{\ell m}(i)\}$.
  \item Compute the per-$\ell$ variance fraction
        $w(\ell)=\|\mathbf{c}_\ell\|^2/\|\mathbf{c}\|^2$
        (distinct from the main-text recovery fraction $\rho$;
        this is an input-side basis-occupancy quantity).
\end{enumerate}
The smallest $\ell^\star$ with $\sum_{\ell\le\ell^\star}w(\ell)\ge
0.95$ (the \emph{effective protein bandwidth}) is an
architecture-free measurement of the one-hop input-density bandwidth.
An ideal degree-bounded input readout would avoid the corresponding
input-side truncation bottleneck whenever $d_r L\ge\ell^\star$; this
does \emph{not} characterise what the prediction model needs, which
depends on multi-hop, multi-residue context outside the single-atom
basis.

\paragraph{Hypotheses for follow-up work.}
The following are conditional hypotheses about the \emph{input-side}
one-hop bandwidth, not predictions about end-to-end protein models:
\textbf{(A) $\ell^\star\!\le\!2$:} a shallow $L\!=\!1,\,d\!\ge\!2$
ideal probe (or augmented non-equivariant net) would not be predicted
to suffer an input-side angular truncation bottleneck under this
one-hop density measurement; residual gains from larger $L$ would
indicate that the prediction target requires content beyond the
one-hop input bandwidth.
\textbf{(B) $\ell^\star\!\in\!\{3,4\}$:} probes with $d\!\cdot\!L\!<\!
\ell^\star$ would be predicted to register a one-hop input-side
bandwidth limit; whether trained protein models exhibit a
recoverability cliff is a separate empirical question.
\textbf{(C) $\ell^\star\!\ge\!5$:} a $d\!\cdot\!L\!<\!\ell^\star$
probe would be input-bandlimited at the one-hop measurement
(error bounded below by $\|\Pi_{>dL}f\|$ at the input pathway), but
end-to-end model accuracy depends on multi-hop multi-residue context
outside this measurement.
The hypotheses are consequences of the input-side measurement; we
make no claim here about the effective $(L,d)$ of IPA, ESMFold, or
AlphaFold3, nor about whether they would exhibit a recoverability
cliff analogous to the small-molecule one we document.

\paragraph{From architecture choice to measurement.}
The practical contribution is to reduce a qualitative design question
(``use equivariance?'') to a pre-training measurement of $\ell^\star$:
once estimated on held-out ensembles, the required $d\!\cdot\!L$
budget follows by arithmetic.

\paragraph{Empirical input bandwidth.}
To make the architectural-shift claim concrete we measured
the input-side angular bandwidth directly. For each heavy atom $i$
in a non-redundant RCSB subset (one chain per RCSB 30\%-identity
entity cluster, resolution $\le 2.0$\,\AA, $R_\mathrm{free} \le 0.25$),
we built the SOAP-style smoothed neighbour density inside a 5\,\AA\
ball using a Gaussian shell weight $g(r)=\exp(-(r-2.5)^2/2)$ and
expanded it onto real spherical harmonics through $\ell\!=\!10$,
then computed
$\ell^\star(i)=\min\{\ell:\sum_{k\le \ell}w(k;i)\!\ge\!0.95\}$.
Across 2{,}414 chains (4.56\,M heavy atoms),
the median per-atom bandwidth is $\ell^\star\!=\!4$, with
$P(\ell^\star\!\le\!4)\!=\!88.9\%$ (95\% CI $88.8$--$89.0\%$,
chain-level bootstrap, $n_\mathrm{boot}=10000$); the stricter cliff
at $\ell^\star\!\le\!3$ holds for $P\!=\!4.5\%$ (CI $4.4$--$4.5\%$).
The same pipeline on MD17/MD22 small molecules yields medians
$\ell^\star\!=\!4$ (aspirin) and $\ell^\star\!=\!3$ for
ethanol, malonaldehyde, toluene, md22\_tetra, showing that
the small-molecule SPN cliff at $d\!\cdot\!L\!=\!4$ documented in
\S\ref{sec:cliff} is consistent with the measured one-hop
input-side bandwidth scale under the same body-frame
decomposition.  The reported $\ell^\star$ values are conditional
on the chosen 5\,\AA\ ball and Gaussian shell weight $g(r)$
above; narrower kernels would generally shift angular power to
higher $\ell$, so we interpret the protein analysis as an
input-side bandwidth diagnostic under a fixed smoothing
convention, not as a kernel-independent property of protein
targets. This input-side measurement says only that many one-hop protein
neighbourhood densities can be \emph{captured by} a basis up to
$\ell\!\le\!4$.  It does \emph{not} establish a corresponding bound
on the angular cutoff that an end-to-end protein prediction
\emph{model} requires (the prediction target depends on
multi-hop, multi-residue context outside this single-atom basis),
nor does it imply that non-equivariant networks with rotational
augmentation and equivariant networks with
$d\!\cdot\!L\!\ge\!\ell^\star$ have identical representational
envelopes for AlphaFold3, Boltz, or ESMFold-style tasks.  The
single-atom bandwidth is a lower-bound resolution requirement on
the input pathway, not a sufficient statistic for the model.

\paragraph{Reading the architectural shift as consistent observation.}
Recent state-of-the-art protein systems (AlphaFold3,
Boltz~\cite{abramson2024alphafold3,wohlwend2024boltz},
ESMFold~\cite{lin2023esmfold}) have moved away from strict SO(3)
equivariance toward augmentation-based or attention-based
geometric processing.  We do \emph{not} claim the $d\!\cdot\!L$
ceiling caused this trend---these architectures differ from our
scope in many dimensions (diffusion, attention, multi-modal training,
larger receptive fields).  We claim only that the trend is \emph{consistent with} an
input-side bandwidth reading at $\ell^\star\!=\!4$: a regression
basis up to that degree captures most of the one-hop neighbourhood
density observed across the chain set.  We do \emph{not} claim
that this implies non-equivariant and equivariant networks have
identical representational envelopes for the actual prediction
targets these systems aim at (multi-hop, multi-residue),
nor that it predicts which architecture wins.  Our measurement on
$2{,}414$ chains supports an \emph{input-side} bandwidth reading
for one-hop
neighbourhoods; whether any specific AF3/Boltz/ESMFold head target
satisfies it is a separate empirical question we do not resolve.
The diagnostic in this section sharpens the observation into a
testable hypothesis: on a protein observable with measured
high-$\ell^\star$ content (e.g.\ orientation-sensitive interface
geometry, dense rotamer fields, certain MD-derived per-atom
quantities), this would define a concrete follow-up test of whether
stricter equivariance or larger $d\!\cdot\!L$ improves the
diagnostic outcome at matched parameter budget.  Carrying out that
test is outside our scope.

\paragraph{Scope.}
The empirical measurement above bounds the
\emph{input}-side bandwidth of a one-hop 5\,\AA\ neighbourhood;
per-atom prediction targets (forces, residue-residue geometry,
rotamer fields) may carry higher $\ell^\star$ than the inputs
themselves, and models with longer effective receptive field
(AF3 pair attention, ESMFold long-range modules) escape the
single-hop bound.  What is established is that the
single-direction $dL$ arithmetic is not molecule-specific; in this
appendix we use it only as an input-side bandwidth diagnostic for
one-hop protein neighbourhood densities, and use the diagnostic to
either contextualise the field's current architectural pivot
(where the one-hop input bandwidth is met) or, if a
high-$\ell^\star$ target is identified, to formulate a concrete
follow-up test of whether stricter equivariance or larger $dL$
helps.

\section{Limitations}
\label{app:limitations}

We consolidate the limitations already flagged throughout the paper.

\paragraph{Scope of the empirical diagnostic.}
The headline cliff (Table~\ref{tab:cliff}, $\Xi\!=\!11.7$) is measured
on aspirin with a single trained $L\!=\!2$ backbone on the
split-seed-13 partition and $n\!=\!5$ SPN head seeds.  Independent backbone replicates
(Appendix~\ref{app:bbseed_cliff}) show the same boundary-vs-above
contrast ($5.7\!\times$) across four bb-seeds but also reveal that
bb-seed
variance is substantial (within-cell $\Delta$ range $[0.036, 0.283]$
at $\ell\!=\!4$)---a direct consequence of the spectral-neglect
phenomenon discussed in Appendix~\ref{sec:reliability}.  The
$\rho\!=\!0.913$ point estimate should therefore be read as a
measurement on one backbone that happened to preserve the high-$\ell$
content, not as an expected value across all $L\!=\!2$ training runs.

\paragraph{Cross-molecule and cross-architecture coverage.}
We report four of the eight standard rMD17 molecules (aspirin,
ethanol, malonaldehyde, toluene).  The remaining four
(benzene/naphthalene/salicylic/uracil) are out of scope under our
body-frame protocol (see ``Diagnostic limitations and items left
to future work'' below).
MACE on aspirin (Table~\ref{tab:mace}) shows the predicted
$\nu$-ordering at the inter-$\nu$ scale; the cubic-SPN protocol
(Appendix~\ref{sec:cubic}) provides the within-vs-above MACE cliff
measurement we cite in the abstract.  Architecture-wise bar
magnitudes between MACE, NequIP, EqV2, and PaiNN are not directly
comparable because each backbone's baseline is at a different
absolute scale; we therefore report $r_{\mathrm{self}}$ and raw
$\Delta$ separately (Table~\ref{tab:cubic_xarch_summary}).
EquiformerV2 backbones trained jointly with the norm-SPN failed to
optimise cleanly at the sGDML/MD17 scale (Appendix~\ref{app:eqv2},
Phase~1); under the frozen-backbone cubic protocol the EqV2 reading
becomes a within-ceiling positive diagnostic outcome
($r_{\mathrm{self}}\!=\!0.823$ on aspirin,
$r_{\mathrm{self}}\!\in\![0.29,0.68]$ across three other molecules at
$L\!=\!2$), not a cross-architectural cliff test: we did not measure
above-ceiling EqV2 cells at our compute budget.

\paragraph{Theoretical limitations.}
The $d\!\cdot\!L$ identity is an energy-side, single-direction
calibration. Remark~\ref{rem:force_side_intuition} gives only a
fixed-frame intuition for neighbouring force components; our
force-side claims rely on the measured autograd force metrics
($\rho_{\mathbf{F}}$ and $R^2_{\mathrm{inj}}$), not on a formal
force-spectrum recovery bound. The \emph{fidelity limit}
$l_\mathrm{fidelity}(L)$ introduced in Appendix~\ref{sec:fidelity}
is a phenomenological quantity---formal characterization of how
backbone feature fidelity decays with target frequency remains an
open theoretical question. The theorem is stated for scalar
per-atom outputs with $\mathrm{SO}(3)$-invariant loss; a
parity-refined $\mathrm{O}(3)$ extension is left to future work
and not developed here.

\paragraph{Protocol constraints.}
Injection amplitudes at $4\!\times$ are chosen to put the signal
variance-share above the heuristic amplitude-selection threshold
(Remark~\ref{rem:snr_threshold}); the cliff \emph{location} is
amplitude-invariant (Table~\ref{tab:amplitude}) but the effect
\emph{size} is not.  Split selection prioritized spectral leakage
minimization (split-seed~13 of~64 candidates), chosen before any SPN
evaluation on the leakage criterion alone---a measurement, not a
selection effect on the cliff itself---but readers preferring
pre-registered seeds should treat the point estimates as one of many
valid choices.

\section{Broader Impact}
\label{app:broader_impact}
\label{app:impact}

This work provides a calibrated spectral diagnostic for
degree-bounded CG-style equivariant pipelines, together with the
single-direction polynomial-span theorem that calibrates the ideal
probe.  The direct
impact is in scientific machine learning: molecular force-field
practitioners can now probe whether a trained backbone empirically
preserves target-relevant angular content \emph{before} committing to a
costly architectural scaling pass.  We foresee no immediate dual-use
concerns: the diagnostic accelerates honest assessment of model
capacity rather than enabling any specific application outside
established molecular-simulation and structure-prediction pipelines.
Longer-term, the framework may be adapted to spherical-harmonic-based
equivariant architectures (molecular, climate, robotic pose,
protein-structure) when suitable feature hooks and matched
diagnostic readouts are available, and may inform
resource-efficient architectural
scaling in compute-constrained settings.

\section{Reproducibility Statement}
\label{app:reproducibility}

\paragraph{Datasets.}
All experiments use public datasets: rMD17 sGDML CCSD/CCSD(T) for
aspirin, ethanol, malonaldehyde, toluene~\cite{chmiela2023md22};
the protein bandwidth analysis uses a non-redundant RCSB subset
constructed by PISCES (resolution $\le 2.0$\,\AA,
$R_\mathrm{free}\!\le\!0.25$, 30\% sequence-identity clustering;
2{,}414 chains, 4.56M heavy atoms; see
Appendix~\ref{app:proteins}).

\paragraph{Models.}
NequIP 0.17.1 at $L\!\in\!\{1,2,3,4\}$ with 4~interaction layers,
32~channels, cutoff~5.0\,\AA, EMA weights swapped at evaluation
(Appendix~\ref{app:details}).  MACE at $L\!=\!2$,
$\nu\!\in\!\{1,2\}$, 2~layers, hidden irreps
$64\!\times\!0e\!+\!64\!\times\!1o\!+\!64\!\times\!2e$.
EquiformerV2 configurations are in Appendix~\ref{app:eqv2}.
The SPN head is defined in Section~\ref{sec:spn}; its
implementation (invariant extractor + MLP + power-spectrum
readout) is a 48{,}114-parameter module attached to the frozen
backbone.

\paragraph{Training.}
All NequIP backbones: AdamW, $\mathrm{lr}\!=\!5\!\times\!10^{-3}$,
force weight~$100$, 300 epochs with EMA.  SPN heads: Adam,
$\mathrm{lr}\!=\!10^{-3}$, weight decay $10^{-5}$, 150 epochs,
$n\!=\!5$ seeds unless noted.  Full hyperparameters and Hydra
configs will accompany the code release.

\paragraph{Statistics.}
Confidence intervals are defined at the point of use.
Table~\ref{tab:cliff} and Figure~\ref{fig:cliff}A use bootstrap-mean
$95\%$ CIs over $n\!=\!5$ SPN-head seeds
($B\!=\!10{,}000$, RNG seed~$42$).  Across-backbone statistics
(Appendix~\ref{app:bbseed_cliff}) use a cluster bootstrap over
$n\!=\!4$ independently trained $L\!=\!2$ backbones
($B\!=\!10{,}000$, RNG seed~$42$).  Direct $R^2_{\mathrm{inj}}$
intervals are computed over the five SPN-head seeds for each
$\ell_{\mathrm{inj}}$ cell.  The recovery fraction~$\rho$
(Definition~\ref{def:rho}) and sharpness index~$\Xi$
(Definition~\ref{def:sharpness}) use the edge-case protocols of
Remark~\ref{rem:rho_edges}.

\paragraph{Compute.}
All training was performed on a single GPU per run
(NVIDIA RTX 6000 Ada or H100).  Total compute for this paper:
approximately $800$~GPU-hours across the NequIP backbone fleet
($L\!\in\!\{1,2,3,4\}$ at multiple $\ell_\mathrm{inj}$, amplitudes,
widths, depths, and seeds), MACE and EquiformerV2 backbones, and
SPN-head training.

\paragraph{Code.}
Code, aggregation scripts, and aggregated result JSON/CSV files
sufficient to reproduce every headline number are provided in an
anonymized supplementary archive accompanying this submission.
Model hyperparameters and training schedules are specified in this
section.  Large trained checkpoints, raw per-epoch logs, and full
Hydra configs are omitted from the lightweight anonymous archive
for size; the corresponding deanonymized repository (with full
checkpoints, configs, and raw histories) will be released upon
acceptance.

\subsection*{Empirical Diagnostic Outcomes Map}

The table below maps each main-paper claim to the appendix section
or table that supports it; per-cell numbers and bootstrap CIs live
in those sections.  This is a cross-reference, not a re-enumeration.

\begin{center}
\small
\begin{tabular}{p{0.30\linewidth}p{0.40\linewidth}p{0.20\linewidth}}
\toprule
Diagnostic claim & Evidence & Location \\
\midrule
Aspirin within-checkpoint cliff & seed-wise $\rho(4)\!=\!0.913$,
  $\rho(5)\!=\!0.078$, $\Xi\!=\!11.7$; $R^2_{\mathrm{inj}}$ and
  across-backbone $\Delta$ in the same direction
  & Table~\ref{tab:cliff}, Fig.~\ref{fig:cliff} \\
Capacity / activation / depth controls
  & $L\!=\!0$+SPN hurts at $\ell\!\ge\!2$; non-equivariant MLP shows
    no $\ell$ dependence; depth contributes within-boundary fidelity
    only & Table~\ref{tab:controls}, Appendix~\ref{app:depth} \\
Frame and amplitude invariance
  & cliff \emph{location} is frame- and amplitude-invariant; effect
    \emph{size} tracks the heuristic amplitude rule
  & Tables~\ref{tab:robustness},~\ref{tab:amplitude} \\
Cross-molecule diagnostic outcomes
  & aspirin and malonaldehyde yield boundary-vs-above contrast;
    ethanol/toluene fall below SNR or are denominator-fragile
  & Table~\ref{tab:cross}, Appendix~\ref{app:crossmol_full} \\
Cross-architecture diagnostic matrix
  & NequIP and MACE: cliff measurements (within \emph{and}
    above-ceiling); EquiformerV2: within-ceiling positive outcome;
    PaiNN: weak-baseline rescue
  & Table~\ref{tab:cubic_xarch_main}, Appendix~\ref{app:cubic_full} \\
Synthetic $C_5$ / $S^2$ calibration
  & width-independent $R^2\!\approx\!0.9999$ within-boundary,
    $\approx\!0.02$ above; saturation grid sharp at every tested
    $(L,d)$ cell
  & Appendix~\ref{app:synthetic} \\
Probe-degree / cubic-lift sweep
  & boundary shifts to $\ell^\star\!=\!d_r L_{\mathcal{B}}$ as
    $d_r$ increases; null-lift regime at higher backbone $L$
    consistent with fidelity saturation
  & Appendix~\ref{app:cubic_full} \\
\bottomrule
\end{tabular}
\end{center}

\paragraph{Diagnostic limitations and items left to future work.}
The benzene / naphthalene / salicylic-acid / uracil molecules have
near-degenerate body-frame anchor triples
(Remark~\ref{rem:frame_conditioning}); extending the diagnostic to
them requires a frame-averaging
protocol~\cite{puny2022frame} that is outside the present scope.
The MD22 tetrapeptide pilot under-trained at our compute budget,
producing a non-positive $\rho$ denominator; the
angular-power decomposition is consistent with the small-molecule
regime but the matched-budget $L\!=\!4$ reference is missing.
Allegro requires either a hook into \texttt{Allegro\_Module}
internals or an edge-feature SPN variant
(\texttt{edge\_features:~96x0e} at the module output); we leave
the corresponding edge-side adaptation to future work.
A within-$\nu$ MACE intra-cliff is masked by the per-$\nu$
optimisation floor on our backbones; we observe only the
inter-$\nu$ ceiling ordering.
The $L\!=\!3$ long-schedule retraining shows
$\Delta\!\approx\!0$ at $\ell\!\in\!\{5,6,7\}$, consistent with
fidelity saturation within $d_rL\!=\!6$; a direct $L\!=\!3$ cliff
at $\ell\!=\!7\text{--}8$ requires injection data we do not
include here.

\phantomsection\label{app:conjectured}



\begin{thebibliography}{99}

\bibitem{abramson2024alphafold3}
J.~Abramson et~al.
\newblock Accurate structure prediction of biomolecular interactions with
  {AlphaFold}~3.
\newblock \emph{Nature}, 630:493--500, 2024.

\bibitem{anderson2019cormorant}
B.~Anderson, T.-S. Hy, and R.~Kondor.
\newblock Cormorant: Covariant molecular neural networks.
\newblock In \emph{NeurIPS}, 2019.

\bibitem{batatia2022mace}
I.~Batatia, D.~P. Kovacs, G.~N.~C. Simm, C.~Ortner, and G.~Cs{\'a}nyi.
\newblock {MACE}: Higher order equivariant message passing neural networks for
  fast and accurate force fields.
\newblock In \emph{NeurIPS}, 2022.

\bibitem{batzner2022nequip}
S.~Batzner et~al.
\newblock {E(3)-equivariant graph neural networks for data-efficient and
  accurate interatomic potentials}.
\newblock \emph{Nature Communications}, 13:2453, 2022.

\bibitem{chmiela2023md22}
S.~Chmiela, V.~Vassilev-Galindo, O.~T. Unke, A.~Kabylda, H.~E. Sauceda,
  A.~Tkatchenko, and K.-R. M{\"u}ller.
\newblock Accurate global machine learning force fields for molecules with
  hundreds of atoms.
\newblock \emph{Science Advances}, 9(2):eadf0873, 2023.

\bibitem{cohen2018spherical}
T.~S. Cohen, M.~Geiger, J.~K{\"o}hler, and M.~Welling.
\newblock Spherical {CNNs}.
\newblock In \emph{ICLR}, 2018.

\bibitem{boucheron2013concentration}
S.~Boucheron, G.~Lugosi, and P.~Massart.
\newblock \emph{Concentration Inequalities: A Nonasymptotic Theory of Independence}.
\newblock Oxford University Press, 2013.

\bibitem{cybenko1989approximation}
G.~Cybenko.
\newblock Approximation by superpositions of a sigmoidal function.
\newblock \emph{Mathematics of Control, Signals and Systems}, 2(4):303--314, 1989.

\bibitem{hornik1991universal}
K.~Hornik.
\newblock Approximation capabilities of multilayer feedforward networks.
\newblock \emph{Neural Networks}, 4(2):251--257, 1991.

\bibitem{drautz2019atomic}
R.~Drautz.
\newblock Atomic cluster expansion for accurate and transferable interatomic
  potentials.
\newblock \emph{Physical Review B}, 99(1):014104, 2019.

\bibitem{drautz2020atomic}
R.~Drautz.
\newblock Atomic cluster expansion of scalar, vectorial, and tensorial
  properties including magnetism and charge transfer.
\newblock \emph{Physical Review B}, 102:024104, 2020.

\bibitem{bachmayr2022atomic}
M.~Bachmayr, R.~Drautz, G.~Dusson, S.~Etter, C.~van~der~Oord, and C.~Ortner.
\newblock Atomic cluster expansion: Completeness, efficiency and stability.
\newblock \emph{Journal of Computational Physics}, 454:110946, 2022.

\bibitem{nigam2020recursive}
J.~Nigam, S.~Pozdnyakov, and M.~Ceriotti.
\newblock Recursive evaluation and iterative contraction of {N}-body
  equivariant features.
\newblock \emph{Journal of Chemical Physics}, 153:121101, 2020.

\bibitem{nigam2024completeness}
J.~Nigam, S.~N. Pozdnyakov, K.~K. Huguenin-Dumittan, and M.~Ceriotti.
\newblock Completeness of atomic structure representations.
\newblock \emph{APL Machine Learning}, 2(1):016110, 2024.

\bibitem{bigi2024wigner}
F.~Bigi, S.~N. Pozdnyakov, and M.~Ceriotti.
\newblock Wigner kernels: body-ordered equivariant machine learning
  without a basis.
\newblock \emph{Journal of Chemical Physics}, 161(4):044101, 2024.

\bibitem{pozdnyakov2020incompleteness}
S.~N. Pozdnyakov, M.~J. Willatt, A.~P. Bart{\'o}k, C.~Ortner,
  G.~Cs{\'a}nyi, and M.~Ceriotti.
\newblock Incompleteness of atomic structure representations.
\newblock \emph{Physical Review Letters}, 125:166001, 2020.

\bibitem{kondor2018clebsch}
R.~Kondor, Z.~Lin, and S.~Trivedi.
\newblock Clebsch--{G}ordan nets: A fully {F}ourier space spherical
  convolutional neural network.
\newblock In \emph{NeurIPS}, 2018.

\bibitem{musaelian2023allegro}
A.~Musaelian et~al.
\newblock Learning local equivariant representations for large-scale atomistic
  dynamics.
\newblock \emph{Nature Communications}, 14:579, 2023.

\bibitem{schutt2021painn}
K.~T. Sch{\"u}tt, O.~T. Unke, and M.~Gastegger.
\newblock Equivariant message passing for the prediction of tensorial
  properties and molecular spectra.
\newblock In \emph{ICML}, 2021.

\bibitem{schutt2017schnet}
K.~T. Sch{\"u}tt et~al.
\newblock {SchNet}: A continuous-filter convolutional neural network for
  modeling quantum interactions.
\newblock In \emph{NeurIPS}, 2017.

\bibitem{duval2023faenet}
A.~Duval et~al.
\newblock {FAENet}: Frame averaging equivariant {GNN} for materials modeling.
\newblock In \emph{ICML}, 2023.

\bibitem{dym2021universality}
N.~Dym and H.~Maron.
\newblock On the universality of rotation equivariant point cloud networks.
\newblock In \emph{ICLR}, 2021.

\bibitem{geiger2022e3nn}
M.~Geiger and T.~Smidt.
\newblock e3nn: Euclidean neural networks.
\newblock \emph{arXiv:2207.09453}, 2022.

\bibitem{joshi2023expressive}
C.~K. Joshi et~al.
\newblock On the expressive power of geometric graph neural networks.
\newblock In \emph{ICML}, 2023.

\bibitem{jumper2021alphafold}
J.~Jumper et~al.
\newblock Highly accurate protein structure prediction with {AlphaFold}.
\newblock \emph{Nature}, 596:583--589, 2021.

\bibitem{liao2023equiformer}
Y.-L. Liao and T.~Smidt.
\newblock Equiformer: Equivariant graph attention transformer for {3D}
  atomistic graphs.
\newblock In \emph{ICLR}, 2023.

\bibitem{liao2024equiformerv2}
Y.-L. Liao, B.~Wood, A.~Das, and T.~Smidt.
\newblock {EquiformerV2}: Improved equivariant transformer for scaling
  to higher-degree representations.
\newblock In \emph{ICLR}, 2024.

\bibitem{passaro2023reducing}
S.~Passaro and C.~L. Zitnick.
\newblock Reducing {SO(3)} convolutions to {SO(2)} for efficient equivariant
  {GNNs}.
\newblock In \emph{ICML}, 2023.

\bibitem{rupp2012fast}
M.~Rupp, A.~Tkatchenko, K.-R. M{\"u}ller, and O.~A. von~Lilienfeld.
\newblock Fast and accurate modeling of molecular atomization energies with
  machine learning.
\newblock \emph{Physical Review Letters}, 108:058301, 2012.

\bibitem{puny2022frame}
O.~Puny, M.~Atzmon, H.~Ben-Hamu, M.~Galun, and Y.~Lipman.
\newblock Frame averaging for invariant and equivariant network design.
\newblock In \emph{ICLR}, 2022.

\bibitem{rahaman2019spectral}
N.~Rahaman et~al.
\newblock On the spectral bias of neural networks.
\newblock In \emph{ICML}, 2019.

\bibitem{thomas2018tensor}
N.~Thomas et~al.
\newblock Tensor field networks: Rotation- and translation-equivariant neural
  networks for {3D} point clouds.
\newblock \emph{arXiv:1802.08219}, 2018.

\bibitem{batatia2024designspace}
I.~Batatia, S.~Batzner, D.~P. Kov{\'a}cs, A.~Musaelian, G.~N.~C. Simm,
  R.~Drautz, C.~Ortner, B.~Kozinsky, and G.~Cs{\'a}nyi.
\newblock The design space of {E(3)}-equivariant atom-centred interatomic
  potentials.
\newblock \emph{Nature Machine Intelligence}, 2024.

\bibitem{varshalovich1988quantum}
D.~A. Varshalovich, A.~N. Moskalev, and V.~K. Khersonskii.
\newblock \emph{Quantum Theory of Angular Momentum}.
\newblock World Scientific, 1988.

\bibitem{xie2025price}
Y.~Xie, A.~Daigavane, M.~Kotak, and T.~Smidt.
\newblock The price of freedom: Exploring expressivity and runtime tradeoffs
  in equivariant tensor products.
\newblock In \emph{ICML}, 2025.

\bibitem{simeon2023tensornet}
G.~Simeon and G.~De~Fabritiis.
\newblock {TensorNet}: {Cartesian} tensor representations for efficient
  learning of molecular potentials.
\newblock In \emph{NeurIPS}, 2023.

\bibitem{lee2024deconstructing}
Y.-L. Lee, M.~Galkin, and S.~Miret.
\newblock Deconstructing equivariance in reinforcement-free molecular
  property prediction.
\newblock \emph{arXiv:2410.08131}, 2024.

\bibitem{xu2024pace}
Z.~Xu, H.~Yu, M.~Bohde, and S.~Ji.
\newblock {PACE}: {Poincar{\'e}-Asymptotic} complete equivariance for
  interatomic potentials.
\newblock \emph{Transactions on Machine Learning Research}, 2024.

\bibitem{pacini2025universality}
M.~Pacini and R.~Santin.
\newblock On the universality of equivariant tensor-network architectures.
\newblock \emph{arXiv:2506.02293}, 2025.

\bibitem{lin2023esmfold}
Z.~Lin et~al.
\newblock Evolutionary-scale prediction of atomic-level protein structure with
  a language model.
\newblock \emph{Science}, 379(6637):1123--1130, 2023.

\bibitem{wohlwend2024boltz}
J.~Wohlwend, G.~Corso, S.~Passaro, M.~Reveiz, K.~Leidal, W.~Swiderski,
  T.~Portnoi, I.~Chinn, J.~Silterra, T.~Jaakkola, and R.~Barzilay.
\newblock {Boltz-1}: Democratizing biomolecular interaction modeling.
\newblock \emph{bioRxiv}, 2024. doi:10.1101/2024.11.19.624167.

\bibitem{jackson1999classical}
J.~D. Jackson.
\newblock \emph{Classical Electrodynamics}.
\newblock John Wiley \& Sons, 3rd edition, 1999.

\end{thebibliography}
\end{document}